\def\arxivmode{1}

\ifdefined\arxivmode
  \documentclass[aop,preprint]{imsart-arxiv}
\else
  \documentclass[aop]{imsart}
\fi
\usepackage{changepage}
\usepackage[font=scriptsize,labelfont=bf]{caption}
\usepackage[numbers]{natbib}
\usepackage{graphicx,caption,subcaption}
\usepackage{enumerate}
\usepackage{amsmath,amsbsy,amssymb,amsfonts,amsthm}
\usepackage{xspace}
\usepackage{algorithm,algorithmic} %
\usepackage{hyperref}
\usepackage{cleveref}
\crefformat{footnote}{#2\footnotemark[#1]#3}
\usepackage[inline]{enumitem}
\usepackage{nicefrac}

\def\nonewproofenvironments{1} %
\newcommand{\dt}{\,dt}
\newcommand{\ds}{\,ds}

\newcommand{\dy}{\,dy}
\newcommand{\dz}{\,dz}
\newcommand{\dW}{\,dW} %
\newcommand{\eps}{\epsilon}

\def\balign#1\ealign{\begin{align}#1\end{align}}
\def\baligns#1\ealigns{\begin{align*}#1\end{align*}}
\def\balignat#1\ealign{\begin{alignat}#1\end{alignat}}
\def\balignats#1\ealigns{\begin{alignat*}#1\end{alignat*}}
\def\bitemize#1\eitemize{\begin{itemize}#1\end{itemize}}
\def\benumerate#1\eenumerate{\begin{enumerate}#1\end{enumerate}}

\newenvironment{talign*}
 {\let\displaystyle\textstyle\csname align*\endcsname}
 {\endalign}
\newenvironment{talign}
 {\let\displaystyle\textstyle\csname align\endcsname}
 {\endalign}

\def\balignst#1\ealignst{\begin{talign*}#1\end{talign*}}
\def\balignt#1\ealignt{\begin{talign}#1\end{talign}}%
\newcommand{\qtext}[1]{\quad\text{#1}\quad} 

\let\originalleft\left
\let\originalright\right
\renewcommand{\left}{\mathopen{}\mathclose\bgroup\originalleft}
\renewcommand{\right}{\aftergroup\egroup\originalright}

\def\Gronwall{G\"otze\xspace}
\def\Gronwall{Gr\"onwall\xspace}
\def\Holder{H\"older\xspace}
\def\Ito{It\^o\xspace}

\def\tinycitep*#1{{\tiny\citep*{#1}}}
\def\tinycitealt*#1{{\tiny\citealt*{#1}}}
\def\tinycite*#1{{\tiny\cite*{#1}}}
\def\smallcitep*#1{{\scriptsize\citep*{#1}}}
\def\smallcitealt*#1{{\scriptsize\citealt*{#1}}}
\def\smallcite*#1{{\scriptsize\cite*{#1}}}

\def\mbb#1{\mathbb{#1}}

\def\tbf#1{\textbf{#1}}

\def\textsum{{\textstyle\sum}} %
\def\textint{{\textstyle\int}} %

\def\reals{\mathbb{R}} %
\def\R{\mathbb{R}}
\def\<{\left\langle} %
\def\>{\right\rangle}

\def\defeq{\triangleq} %
\def\half{\frac{1}{2}}
\def\texthalf{{\textstyle\frac{1}{2}}}
\newcommand{\textfrac}[2]{{\textstyle\frac{#1}{#2}}}

\newcommand{\psdge}{\succcurlyeq}

\newcommand{\psdgt}{\succ}

\newcommand{\norm}[1]{\left\|{#1}\right\|} %
\newcommand{\onenorm}[1]{\norm{#1}_1} %
\newcommand{\twonorm}[1]{\norm{#1}_2} %
\newcommand{\infnorm}[1]{\norm{#1}_{\infty}} %
\newcommand{\opnorm}[1]{\norm{#1}_{op}} %
\newcommand{\fronorm}[1]{\norm{#1}_{F}} %
\def\staticnorm#1{\|{#1}\|} %
\newcommand{\inner}[2]{\langle{#1},{#2}\rangle} %
\def\maxeig#1{\lambda_{\mathrm{max}}\left({#1}\right)}
\def\mineig#1{\lambda_{\mathrm{min}}\left({#1}\right)}

\def\indic#1{\mbb{I}\left[{#1}\right]} %

\def\maxarg#1{\max\left({#1}\right)} %
\def\minarg#1{\min\left({#1}\right)} %
\def\E{\mbb{E}} %
\def\Earg#1{\E\left[{#1}\right]}
\def\Esubarg#1#2{\E_{#1}\left[{#2}\right]}
\def\bigO#1{O(#1)} %
\renewcommand{\exp}[1]{\operatorname{exp}\left(#1\right)} %
\newcommand{\staticexp}[1]{\operatorname{exp}(#1)} %
\newcommand{\Gsn}{\mathcal{N}}

\newcommand{\grad}{\nabla} %
\newcommand{\Hess}{\nabla^2} %
\newcommand{\pderiv}[2]{\frac{\partial #1}{\partial #2}} %

\providecommand{\dom}{\mathop\mathrm{dom}}

\def\supp#1{\mathrm{supp}({#1})}

\newcommand{\algref}[1]{Algorithm~{\ref{alg:#1}}}

\newcommand{\cororef}[1]{Corollary~{\ref{cor:#1}}}

\newcommand{\eqnref}[1]{\eqref{eqn:#1}}
\newcommand{\exref}[1]{Example~{\ref{ex:#1}}}
\newcommand{\figref}[1]{Figure~{\ref{fig:#1}}}
\newcommand{\lemref}[1]{Lemma~{\ref{lem:#1}}}
\newcommand{\lemsref}[1]{Lemmas~{\ref{lem:#1}}}
\newcommand{\lemssref}[1]{{\ref{lem:#1}}}

\newcommand{\secref}[1]{Section~{\ref{sec:#1}}}
\newcommand{\secsref}[1]{Sections~{\ref{sec:#1}}}
\newcommand{\secssref}[1]{{\ref{sec:#1}}}
\newcommand{\propref}[1]{Proposition~{\ref{prop:#1}}}

\newcommand{\thmref}[1]{Theorem~{\ref{thm:#1}}}

\newcommand{\lemreflow}[1]{Lemma~\lowercase{\ref{lem:#1}}}

\newcommand{\propreflow}[1]{Proposition~\lowercase{\ref{prop:#1}}}

\newcommand{\thmreflow}[1]{Theorem~\lowercase{\ref{thm:#1}}}

\ifdefined\nonewproofenvironments\else
\ifdefined\ispres\else
\newtheorem{theorem}{Theorem}
\newtheorem{lemma}[theorem]{Lemma}
\newtheorem{corollary}[theorem]{Corollary}
\newtheorem{definition}[theorem]{Definition}

\renewenvironment{proof}{\noindent\textbf{Proof}\hspace*{1em}}{\qed\\}
\newenvironment{proof-sketch}{\noindent\textbf{Proof Sketch}
  \hspace*{1em}}{\qed\bigskip\\}
\newenvironment{proof-idea}{\noindent\textbf{Proof Idea}
  \hspace*{1em}}{\qed\bigskip\\}
\newenvironment{proof-of-lemma}[1][{}]{\noindent\textbf{Proof of Lemma {#1}}
  \hspace*{1em}}{\qed\\}
\newenvironment{proof-of-theorem}[1][{}]{\noindent\textbf{Proof of Theorem {#1}}
  \hspace*{1em}}{\qed\\}
\newenvironment{proof-attempt}{\noindent\textbf{Proof Attempt}
  \hspace*{1em}}{\qed\bigskip\\}

\newenvironment{remark}{\noindent\textbf{Remark}
  \hspace*{1em}}{\bigskip}
\fi

\newtheorem{proposition}[theorem]{Proposition}

\fi

\def\sr{s_r}
\renewcommand{\norm}[1]{\staticnorm{#1}} %

\newtheorem{theorem}{Theorem}
\newtheorem{lemma}[theorem]{Lemma}
\newtheorem{corollary}[theorem]{Corollary}
\newtheorem{proposition}[theorem]{Proposition}
\newtheorem{definition}[theorem]{Definition}
\theoremstyle{definition}
\newtheorem{example}{Example}
\theoremstyle{remark}
\newtheorem*{remark}{Remark}
\newtheorem*{remarks}{Remarks}

\graphicspath{{./figs/}}

\newcommand{\gset}[0]{\mathcal{G}} %
\newcommand{\steinset}[0]{\steinsetarg{}} %
\newcommand{\steinsetarg}[1]{\gset_{\norm{\cdot}_{#1}}} %
\newcommand{\gsteinset}[2]{\gset_{\norm{\cdot}_{#1},Q_n,{#2}}} %
\newcommand{\hset}[0]{\mathcal{H}} %
\newcommand{\pset}[0]{\mathcal{P}_1} %

\newcommand{\wasssetarg}[1]{\mathcal{W}_{#1}} %
\newcommand{\wassset}{\wasssetarg{\norm{\cdot}}} %
\newcommand{\onewassset}{\wasssetarg{\onenorm{\cdot}}} %
\newcommand{\twowassset}{\wasssetarg{\twonorm{\cdot}}} %

\newcommand{\ball}{\mathcal{B}} %

\newcommand{\generator}[1]{\mathcal{A}{#1}} %
\newcommand{\genarg}[2]{(\generator{#1})({#2})} %

\newcommand{\operator}[1]{\mathcal{T}{#1}} %
\newcommand{\oparg}[2]{(\operator{#1})({#2})} %

\newcommand{\diffusion}[1]{\mathcal{T}{#1}} %
\newcommand{\diffarg}[2]{(\diffusion{#1})({#2})} %

\newcommand{\ipm}{d_\hset} %
\newcommand{\stein}[3]{\mathcal{S}({#1},{#2},{#3})} %
\newcommand{\opstein}[2]{\stein{#1}{\operator{}}{#2}} %
\newcommand{\diffstein}[2]{\stein{#1}{\diffusion{}}{#2}} %
\newcommand{\wass}{d_{\wassset}} %
\newcommand{\onewass}{d_{\onewassset}} %
\newcommand{\twowass}{d_{\twowassset}} %
\newcommand{\wassarg}[1]{d_{\wasssetarg{#1}}} %
\newcommand{\lswass}[1]{W_{#1,\norm{\cdot}}} %
\newcommand{\twolswass}[1]{W_{#1,\twonorm{\cdot}}} %

\newcommand{\trans}[2]{P_{#1}{#2}} %
\newcommand{\fulltrans}[0]{(\trans{t}{})_{t\geq0}} %
\newcommand{\transarg}[3]{(P_{#1}{#2})({#3})} %

\newcommand{\qvar}[0]{x} %
\newcommand{\pvar}[0]{z}%
\newcommand{\QVAR}[0]{\MakeUppercase{\qvar}} %
\newcommand{\PVAR}[0]{\MakeUppercase{\pvar}} %

\newcommand{\process}[2]{\PVAR_{#1,#2}} %
\newcommand{\fullprocess}[1]{(\process{t}{#1})_{t\geq0}} %

\newcommand{\diffref}[1]{\ref{diff:#1}} %

\usepackage[colorinlistoftodos,prependcaption,textsize=tiny]{todonotes}

\begin{document}

\captionsetup{belowskip=0pt,aboveskip=4pt} %
\floatsep=\baselineskip %
\textfloatsep=\baselineskip %

\begin{frontmatter}

\title{Measuring Sample Quality with Diffusions}

\begin{aug}
  \author{\fnms{Jackson}  \snm{Gorham}\ead[label=e1]{jacksongorham@gmail.com}},
  \author{\fnms{Andrew B.} \snm{Duncan}\ead[label=e2]{a.duncan@imperial.ac.uk}},
    \author{\fnms{Sebastian J.} \snm{Vollmer}\ead[label=e3]{svollmer@turing.ac.uk}},
  \and
  \author{\fnms{Lester} \snm{Mackey}\ead[label=e4]{lmackey@microsoft.com}}

  \runauthor{J. Gorham, A.B. Duncan, S.J. Vollmer, and L. Mackey}

  \affiliation{Stanford University, Imperial College London, University of Warwick, and Microsoft Research New England}

\ifdefined\arxivmode
\else
  \address{J.\ Gorham\\
  Stanford University\\
  Palo Alto, CA\\
  USA\\
  \phantom{E-mail:\ }\printead*{e1}}
  \address{A.B.\ Duncan\\
  Department of Mathematics\\ 
  Imperial College London\\
  London SW7 2AZ\\
  UK\\
  \phantom{E-mail:\ }\printead*{e2}}
  \address{S.J. Vollmer \\
  Mathematics Institute\\
  Zeeman Building\\
  University of Warwick \\
  Coventry CV4 7AL\\
  UK\\
  \phantom{E-mail:\ }\printead*{e3}}
  \address{L.\ Mackey\\
  Microsoft Research New England\\
  Cambridge, MA\\
  USA\\
  \phantom{E-mail:\ }\printead*{e4}}
\fi
\end{aug}

\begin{abstract}
Stein's method for measuring convergence to a continuous target distribution relies on an operator characterizing the target and \emph{Stein factor} bounds on the solutions of an associated differential equation.  While such operators and bounds are readily available for a diversity of univariate targets, few multivariate targets have been analyzed.  We introduce a new class of characterizing operators based on \Ito diffusions and develop explicit multivariate Stein factor bounds for any target with a fast-coupling \Ito diffusion.  As example applications, we develop computable and convergence-determining \emph{diffusion Stein discrepancies} for log-concave, heavy-tailed, and multimodal targets and use these quality measures to select the hyperparameters of biased Markov chain Monte Carlo (MCMC) samplers, 
compare random and deterministic quadrature rules, and quantify bias-variance tradeoffs in approximate MCMC.  Our results establish a near-linear relationship between diffusion Stein discrepancies and Wasserstein distances, improving upon past work even for strongly log-concave targets.  The exposed relationship between Stein factors and Markov process coupling may be of independent interest.
\end{abstract}

\begin{keyword}[class=MSC]
\kwd[Primary ]{%
	60J60; %
	62-04; %
	62E17; %
    60E15; %
    	65C60%
}
\kwd[; secondary ]{%
	62-07; %
	65C05; %
    68T05%
}
\end{keyword}

\begin{keyword}
\kwd{multivariate Stein factors}
\kwd{\Ito diffusion}
\kwd{Stein's method}
\kwd{Stein discrepancy}
\kwd{sample quality}
\kwd{Wasserstein decay}
\kwd{Markov chain Monte Carlo}
\end{keyword}

\end{frontmatter}
\section{Introduction}
\label{sec:intro}

Consider a target probability distribution $P$ 
with finite mean, continuously differentiable density $p$,
and support on all of $\reals^d$.
We will name the set of all such distributions $\pset$.
We assume that $p$ can be evaluated up to its normalizing constant
but that exact expectations under $P$ are unattainable for most functions of interest.
We will therefore use a \emph{weighted sample}, represented as a discrete distribution
$Q_n = \sum_{i=1}^n q(x_i)\delta_{x_i}$, to approximate intractable expectations
$\Esubarg{P}{h(\PVAR)}$ with tractable sample estimates $\Esubarg{Q_n}{h(X)} = \sum_{i=1}^n q(x_i)h(x_i)$.
Here, the support of $Q_n$ is a collection of distinct sample points $x_1, \dots, x_n \in \reals^d$,
and the weight $q(x_i)$ associated with each point is governed by a probability mass function $q$.
We assume nothing about the process generating the sample points, so they may be the product of any random or deterministic mechanism.

Our ultimate goal is to develop a computable quality measure suitable for comparing any 
two samples approximating the same target distribution.
More precisely, we seek to quantify how well $\E_{Q_n}$ approximates $\E_P$
in a manner that, at the very least,
\begin{enumerate*}[label=\emph{(\roman*)}]
\item\label{desiderata:det-conv} indicates when a sample sequence is converging to $P$,
\item\label{desiderata:det-nonconv} identifies when a sample sequence is not
converging to $P$, and
\item\label{desiderata:comp} is computationally tractable.
\end{enumerate*}
A natural starting point is to consider the maximum error incurred by the sample approximation over a class of scalar test functions $\hset$,
\balign\label{eqn:ipm-definition}
d_{\hset}(Q_n, P) \defeq \sup_{h\in\hset} |\Esubarg{P}{h(\PVAR)} -
  \Esubarg{Q_n}{h(\QVAR)}|.
\ealign
When $\hset$ is convergence determining, the measure~\eqnref{ipm-definition} is an \emph{integral probability metric} (IPM)~\cite{Muller97}, and $\ipm(Q_n, P)$ converges to zero only if the sample sequence $(Q_n)_{n\geq 1}$ converges in distribution to $P$.

While a variety of standard probability metrics are representable as IPMs \cite{Muller97},
the intractability of integration under $P$ precludes us from computing most of these candidate quality measures.
Recently, \citet{GorhamMa15} sidestepped this issue by constructing a class of test functions $h$ known a priori to have zero mean under $P$.
Their resulting quality measure -- the \emph{Langevin graph Stein
discrepancy} -- satisfied our computability and convergence detection
requirements
(Desiderata \ref{desiderata:det-conv} and \ref{desiderata:comp})
and detected sample sequence non-convergence
(Desideratum \ref{desiderata:det-nonconv}) for strongly log concave targets
with bounded third and fourth derivatives~\citep{MackeyGo16}.
Our first contribution is to show that the Langevin Stein discrepancy in fact determines convergence for all smooth, \emph{distantly dissipative} target distributions by explicitly lower and upper bounding the Langevin Stein discrepancy by standard Wasserstein distances.  
Distant dissipativity is a substantial relaxation of log concavity that covers a variety of common non-log concave targets like Gaussian mixtures and robust Student's t regression posteriors. 
This contribution greatly extends the range of applicability of the Langevin Stein discrepancy.

Because heavy-tailed distributions are never distantly dissipative, as a second contribution, we extend the computable Stein discrepancy framework of \citep{GorhamMa15} to accommodate heavy-tailed target distributions by introducing a new class of multivariate Stein operators based on general \Ito diffusions.  These operators can be used as drop-in replacements for the commonly used Langevin operator in applications.

As a third contribution, we establish a near linear relationship between the
introduced \emph{diffusion Stein discrepancies} $\diffstein{Q_n}{\steinset}$
and standard $L^s$ Wasserstein distances 
$
{\lswass{s}}\left(Q_n,P\right)  \defeq  \inf_{X\sim Q_n, Z \sim P}{\E[{\norm{X-Z}^s}]}^{1/s}
$. Namely,
\begin{align*}
\lswass{1}\left(Q_n,P\right) &\le
C_1\,\diffstein{Q_n}{\steinset}\max(1,
\log(\nicefrac{1}{ \diffstein{Q_n}{\steinset}})) \qtext{and} \\
\diffstein{Q_n}{\steinset} &\le C_2\,\lswass{2}\left(Q_n,P\right) \end{align*}
for constants $C_1,C_2 > 0$ determined by \thmref{nonconstant-lower-bound} and \propref{discrepancy-upper-bound}.
This improves upon prior analyses even in the case of
strongly log concave targets.

Our primary contribution underlies these three advances.
By relating Stein's method to Markov process coupling rates in~\secref{steins_method}, 
we prove that every sufficiently fast coupling \Ito diffusion gives rise to explicit, uniform multivariate \emph{Stein factor} bounds on the derivatives of Stein equation solutions.  
Stein factor bounds are central to Stein's method of measuring distributional convergence,
and while a wealth of bounds are available for univariate targets (see, e.g., \citep{SteinDiHoRe04,ChatterjeeSh11,ChenGoSh11} for explicit bounds or \cite{LeyReSw2017} for a recent review), Stein factors for continuous multivariate distributions have largely been relegated to Gaussian~\citep{Barbour90,Gotze91,ReinertRo09,ChatterjeeMe08,Meckes09,NourdinPeRe10,Gaunt2016}, Dirichlet~\citep{GanRoRo2016}, and strongly log-concave~\citep{MackeyGo16} target distributions.
Our approach, which exposes a general relationship between Stein factors and Markov process coupling times, 
extends the reach of Stein's method to the stationary distributions of all fast coupling \Ito diffusions.

In \secref{examples}, we provide examples of practically checkable sufficient conditions
for fast coupling and illustrate the process of verifying these conditions
for canonical log-concave, heavy-tailed, and multimodal targets.
\secref{programs} describes a practical algorithm for computing diffusion Stein discrepancies using a geometric spanner and linear programming.
In \secref{experiments}, we complement the principal theoretical contributions of this work with several simple numerical examples illustrating how diffusion Stein discrepancies can be deployed in practice.
In particular, we use our discrepancies to
select the hyperparameters of biased samplers, 
compare random and deterministic quadrature rules, and quantify bias-variance tradeoffs in approximate Markov chain Monte Carlo.
A discussion of related and future work follows in \secref{conclusion}, and all proofs are deferred to the appendices.
\textbf{Notation}
For $r\in[1,\infty]$, let $\norm{\cdot}_r$ denote the $\ell^r$ norm on
$\reals^d$. We will use $\norm{\cdot}$ as a generic norm on $\reals^d$
satisfying $\norm{\cdot} \geq \twonorm{\cdot}$ and define the associated dual norms,
$\norm{v}^* \defeq \sup_{u\in \reals^d: \norm{u}=1}{\inner{u}{v}}$ for
vectors $v\in\reals^d$ and
$\norm{W}^* \defeq \sup_{u\in \reals^d: \norm{u}=1}{\norm{Wu}^*}$ for
matrices $W\in \reals^{d \times d}$.
Let $e_j$ be the $j$-th standard
basis vector, $\grad_j$ be the partial derivative $\pderiv{}{x_j}$, and
$\mineig{\cdot}$ and $\maxeig{\cdot}$ be the smallest and largest eigenvalues of a symmetric matrix.
For any real vector $v$ and tensor $T$, let $\opnorm{v} \defeq \twonorm{v}$
and $\opnorm{T} \defeq \sup_{\twonorm{u}=1}\opnorm{T[u]}$.
For each sufficiently differentiable vector- or matrix-valued function $g$, we 
define the bound $M_0(g) \defeq \sup_{x\in\reals^d} \opnorm{g(x)}$ and 
the $k$-th order \Holder coefficients
\balignst
M_{k}(g)\defeq\sup_{x,y\in\reals^{d},x\neq y}\textfrac{\opnorm{\grad^{\left\lceil k\right\rceil -1}g(x)-\grad^{\left\lceil k\right\rceil -1}g(y)}}{\twonorm{x-y}^{\left\{ k\right\} }} \qtext{for} k > 0,
\ealignst
where $\left\{ k\right\} \defeq k - \lceil k-1\rceil$ and $\grad^0$ is the identity operator.
For each differentiable matrix-valued function $a$,
we let $\inner{\grad}{a(x)} = \sum_j e_j \sum_{k}\grad_k a_{jk}(x)$ represent the divergence operator applied to each row of $a$ and define the Lipschitz coefficients
$F_k(a) \defeq \sup_{x\in\reals^d,\twonorm{v_1}=1,\dots,\twonorm{v_k}=1} \fronorm{\grad^k a(x)[v_1,\dots,v_k]}$ for $\fronorm{\cdot}$ the Frobenius norm.
Finally, when the domain and range of a function $f$ can be inferred from context, we write $f\in C^k$ to indicate that $f$ has $k$ continuous derivatives.
\section{Stein's method}
\label{sec:steins_method}

In the early 1970s,  Charles Stein~\citep{Stein72} introduced a powerful three-step approach to upper-bounding a reference IPM $\ipm$:
\benumerate
\item First, identify an operator $\operator{}$ that maps input functions\footnote{Real-valued $g$ are also common, but $\reals^d$-valued $g$ are more convenient for our purposes.} $g : \reals^d \to \reals^d$ in a domain $\gset$ into mean-zero functions under $P$, i.e.,
\[
\Esubarg{P}{\oparg{g}{\PVAR}} = 0 \qtext{for all} g \in \gset.
\]
The operator $\operator{}$ and its domain $\gset$ define the \emph{Stein discrepancy}~\citep{GorhamMa15},
\balign\notag
	\opstein{Q_n}{\gset}
		&\defeq \sup_{g\in\gset} |\Esubarg{Q_n}{\oparg{g}{X}}| \\
		&= \sup_{g\in\gset} |\Esubarg{Q_n}{\oparg{g}{X}} - \Esubarg{P}{\oparg{g}{\PVAR}}| 
		= d_{\operator{}\gset}(Q_n,P), \label{eqn:stein-discrepancy}
\ealign
a quality measure which takes the form of an integral probability metric while avoiding explicit integration under $P$.
\item Next, prove that, for each test function $h$ in the reference class $\hset$, the \emph{Stein equation}
\balign\label{eqn:stein-equation}
h(x) - \Esubarg{P}{h(\PVAR)} 
	= \oparg{g_h}{x}
\ealign
admits a solution $g_h \in \gset$.
This step ensures that the reference metric $\ipm$ lower bounds the Stein discrepancy (Desideratum \ref{desiderata:det-nonconv}) 
and, in practice, can be carried out simultaneously for large classes of target distributions.
\item 
Finally, use whatever means necessary to upper bound the Stein discrepancy and thereby establish convergence to zero under appropriate conditions (Desideratum \ref{desiderata:det-conv}).
Our general result, \propref{discrepancy-upper-bound}, suffices for this purpose.
\eenumerate
While Stein's method is traditionally used as analytical tool to establish rates of distributional convergence, 
we aim, following \cite{GorhamMa15}, to develop the method into a practical computational tool for measuring the quality of a sample.
We begin by assessing the convergence properties of a broad class of Stein operators derived from \Ito diffusions.
Our efforts will culminate in \secref{programs}, where we show how to explicitly compute the Stein discrepancy~\eqnref{stein-discrepancy} given any sample measure $Q_n$ and appropriate choices of $\operator{}$ and $\gset$.

\subsection{Identifying a Stein operator}
\label{sec:stein_op_id}
To identify an operator $\operator{}$ that generates mean-zero functions under $P$, we will appeal to the elegant and widely applicable \emph{generator method} construction
of \citet{Barbour88,Barbour90} and \citet{Gotze91}.
These authors note that if $(Z_t)_{t\geq0}$ is a Feller process with invariant measure $P$, then the \emph{infinitesimal generator} $\generator{}$ of the process, defined pointwise by
\balign\label{eqn:inf-generator}
	\genarg{u}{x} = \lim_{t\to0} {(\Earg{u(\PVAR_t) \mid \PVAR_0 = x} - u(x))/}{t} %
\ealign
satisfies $\Esubarg{P}{\genarg{u}{\PVAR}} = 0$ under very mild restrictions on $u$ and $\generator{}$.
\citet{GorhamMa15} developed a \emph{Langevin Stein operator} based on the generator a specific Markov process -- the Langevin diffusion described in \diffref{langevin}.  
Here, we will consider a broader class of continuous Markov processes known as \emph{\Ito diffusions}. 
\begin{definition}[\Ito diffusion {\cite[Def.~7.1.1]{Oksendal2013}}]
\label{def:diffusion}
A (time-homogeneous) \emph{\Ito diffusion} 
with starting point $x\in\reals^d$,
Lipschitz \emph{drift coefficient} $b : \reals^d \to \reals^d$, and 
Lipschitz \emph{diffusion coefficient} $\sigma : \reals^d \to \reals^{d\times m}$
is a stochastic process $(\process{t}{x})_{t\geq0}$ solving the \Ito stochastic differential equation
\balign\label{eqn:diffusion}
d\process{t}{x} = b(\process{t}{x})\dt + \sigma(\process{t}{x})\,dW_t\qtext{with} \process{0}{x}=x\in\reals^d, 
\ealign
where $(W_t)_{t\geq 0}$ is an $m$-dimensional Wiener process.
\end{definition}
As the next theorem, distilled from \cite[Thm. 2]{ma2015complete} and \cite[Sec. 4.6]{Pavliotis14}, shows, it is straightforward to construct \Ito diffusions with a given invariant measure $P$ (see also \cite{kent1978time, hwang1993accelerating}).

\begin{theorem}[{\cite[Thm. 2]{ma2015complete} and \cite[Sec. 4.6]{Pavliotis14}}]
\label{thm:invariance}
Fix an \Ito diffusion with $C^1$ drift and diffusion coefficients
$b$ and $\sigma$, and define its \emph{covariance coefficient} $a(x) \defeq \sigma(x)\sigma(x)^\top$.
$P\in\pset$ is an invariant measure of this diffusion if and only if
$%
  b(x)=\texthalf\textfrac{1}{p(x)}\inner{\grad}{p(x)a(x)} +f(x) 
$ %
for a \emph{non-reversible component} $f \in C^1$ satisfying
$
\inner{\grad}{p(x) f(x)} = 0$
for all $x \in \reals^d.$
If $f$ is $P$-integrable, then 
\balign
\label{eqn:drift}
  b(x)=\texthalf\textfrac{1}{p(x)}\inner{\grad}{p(x)(a(x)+c(x))}
\ealign
for $c$ a differentiable $P$-integrable skew-symmetric $d\times d$ matrix--valued function termed the \emph{stream coefficient}~\citep{conca2007periodic,landim1998convection}.
In this case, for all $u\in C^2\cap\dom(\generator{})$, the infinitesimal generator \eqnref{inf-generator} of the diffusion takes the form
\balign
\label{eqn:diff-generator}
\genarg{u}{x} 
	= \texthalf \textfrac{1}{p(x)}\inner{\grad}{p(x)(a(x) + c(x))\grad u(x)}.\footnotemark
\ealign
\end{theorem}\footnotetext{We have chosen an atypical form for the infinitesimal generator in \eqnref{diff-generator}, as it will give rise to a first-order differential operator \eqnref{diffusion-operator} with more desirable properties.  One can check, for instance, that the first order operator 
$\diffarg{g}{x} = 2\inner{b(x)}{g(x)} + \inner{a(x)}{\grad g(x)}$
derived from the standard form of the generator,
$\genarg{u}{x} = \inner{b(x)}{\grad u(x)} + \half\inner{a(x)}{\Hess u(x)}$,
fails to satisfy \propref{diffusion-zero} whenever the non-reversible component $f(x) \not\equiv 0$.}

\begin{remarks}
\thmref{invariance} does not require Lipschitz assumptions on $b$ or $\sigma$.
An example of a non-reversible component which is not $P$-integrable is $f(x) = v/p(x)$ for any constant vector $v \in \reals^d$. 
Prominent examples of $P$-targeted diffusions include
\begin{enumerate}[label=({D\arabic*})]
\item \label{diff:langevin} the \emph{(overdamped) Langevin diffusion} (also known as the \emph{Brownian} or \emph{Smoluchowski dynamics})~\citep[Secs. 6.5 and 4.5]{Pavliotis14}, where $a \equiv I$ and $c \equiv 0$; %
\item \label{diff:preconditioned} the \emph{preconditioned Langevin diffusion} \cite{stuart2004conditional},
where $c\equiv 0$ and $a \equiv \sigma\sigma^\top$
for a constant diffusion coefficient $\sigma\in\reals^{d\times m}$ ;
\item  \label{diff:riemannian} the \emph{Riemannian Langevin diffusion}~\cite{kent1978time,roberts2002langevin,girolami2011riemann}, 
where $c\equiv 0$ and $a$ is not constant;
\item  \label{diff:nonreversible} the \emph{non-reversible preconditioned Langevin diffusion} \cite[see, e.g.,][]{ma2015complete,duncan2016variance,rey2014irreversible}, 
where $a \equiv \sigma\sigma^\top$ for $\sigma \in\reals^{d\times m}$ constant and $c$ not identically $0$;

\item  \label{diff:underdamped} and the \emph{second-order} or \emph{underdamped Langevin diffusion} \cite{horowitz1987second}, where we target the joint distribution $P \otimes \mathcal{N}(0, I)$ on $\reals^{2d}$ with 
\baligns
a \equiv 2\begin{pmatrix} 0 & 0 \\ 0 & I\end{pmatrix}
\text{ and }
c \equiv 2\begin{pmatrix} 0 & -I \\ I & 0\end{pmatrix}.%
\ealigns
\end{enumerate}
We will present detailed
examples making use of these diffusion classes in \secsref{examples}
and \secssref{experiments}.
\end{remarks}

\thmref{invariance} forms the basis for our Stein operator of choice, the \emph{diffusion Stein operator} $\diffusion{}$,
defined by substituting $g$ for $\frac{1}{2} \grad u$ in the generator \eqnref{diff-generator}:
\balign\label{eqn:diffusion-operator}
\diffarg{g}{x} 
	&= \textfrac{1}{p(x)}\inner{\grad}{p(x)(a(x)+c(x))g(x)}.
\ealign
$\diffusion{}$ is an appropriate choice for our setting as it depends on $P$ only through $\grad \log p$ and is therefore computable even when the normalizing constant of $p$ is unavailable.
One suitable domain for $\diffusion{}$ is the \emph{classical Stein set}~\citep{GorhamMa15} 
of 1-bounded functions with 1-bounded, 1-Lipschitz derivatives:
\baligns
	\steinset 
		\defeq \bigg\{ g :\reals^d\to\reals^d \bigg|
		&\sup_{x\neq y\in \reals^d} \maxarg{\norm{g(x)}^*,\norm{\grad g(x)}^*, \textfrac{\norm{\grad g(x) - \grad g(y)}^*}{\norm{x-y}}} \leq 1\bigg\}.
\ealigns
Indeed, our next proposition, proved in \secref{diffusion-zero-proof}, shows that, on this domain, the diffusion Stein operator generates
mean-zero functions under $P$.
\begin{proposition} \label{prop:diffusion-zero}
	If $\diffusion{}$ is the diffusion Stein operator \eqnref{diffusion-operator} for $P \in \pset$ with $a, c \in C^1$ and $a$, $c,$ $b$ \eqnref{drift} $P$-integrable, then $\Esubarg{P}{\diffarg{g}{\PVAR}} = 0$ for all $g \in \steinset$.
\end{proposition}
\noindent Together, $\diffusion{}$ and $\steinset$ give rise to the \emph{classical diffusion Stein discrepancy} $\diffstein{Q_n}{\steinset}$, our primary object of study in \secsref{discrepancy-lower-bound} and \secssref{discrepancy-upper-bound}.

\subsection{Lower bounding the diffusion Stein discrepancy}\label{sec:discrepancy-lower-bound}
To establish that the classical diffusion Stein discrepancy detects non-convergence (Desideratum \ref{desiderata:det-nonconv}), we
will lower bound the discrepancy in terms of the \emph{$L^1$ Wasserstein distance}, $\twowass = \twolswass{1}$, a standard reference IPM generated by 
\[\textstyle\hset = \twowassset\defeq\{h : \reals^d\to\reals \mid \sup_{x\neq y\in\reals^d} |h(x) - h(y)| \leq \twonorm{x-y} \}.\]
The first step is to show that, for each $h\in\twowassset$, the solution $g_h$ to the Stein equation \eqnref{stein-equation} with diffusion Stein operator \eqnref{diffusion-operator} has low-order derivatives uniformly bounded by target-specific constants called \emph{Stein factors}.

Explicit Langevin diffusion \diffref{langevin} Stein factor bounds are readily available for a wide variety of univariate targets\footnote{The Langevin operator recovers \emph{Stein's density method operator~\cite{SteinDiHoRe04}} when $d=1$.} (see, e.g., \citep{SteinDiHoRe04,ChatterjeeSh11,ChenGoSh11} for explicit bounds or \cite{LeyReSw2017} for a recent review). 
In contrast, in the multivariate setting, efforts to establish Stein factors have focused on Gaussian~\citep{Barbour90,Gotze91,ReinertRo09,ChatterjeeMe08,Meckes09,NourdinPeRe10,Gaunt2016}, Dirichlet~\citep{GanRoRo2016}, and strongly log-concave~\citep{MackeyGo16} targets with preconditioned Langevin \diffref{preconditioned} operators. 
To extend the reach of the literature, we will derive multivariate Stein factors for targets with fast-coupling \Ito diffusions.
Our measure of coupling speed is the \emph{Wasserstein decay rate}.
\begin{definition}[Wasserstein decay rate]
\label{def:wass-decay-rate}
Let $\fulltrans{}$ be the \emph{transition semigroup} of an \Ito diffusion $\fullprocess{x}$ defined via 
\baligns
\transarg{t}{f}{x} \defeq \Earg{f(\process{t}{x})} \qtext{for all measurable $f$,} x\in\reals^d, \qtext{and} t\geq 0.
\ealigns
For any non-increasing integrable function $r : \reals_{\geq 0} \to \reals$,
we say that %
$\fulltrans{}$
has \emph{Wasserstein decay rate} $r$ if
\balign
\label{eqn:wass-decay-rate}
\twowass(\delta_x \trans{t}{},\,\delta_y \trans{t}{}) 
	\leq r(t)\, \twowass(\delta_x,\delta_y)
	\quad\text{for all } x,y \in \reals^d \text{ and } t \geq 0,
\ealign
where $\delta_x \trans{t}{}$ denotes the distribution of $\process{t}{x}$.
\end{definition}

Our next result, proved in \secref{stein-factors-proof}, shows that the smoothness of a solution $g_h$ to a Stein equation is controlled by the rate of Wasserstein decay and hence by how quickly two diffusions with distinct starting points couple.
The Stein factor bounds on the derivatives of $u_h$ and $g_h$ may be of independent interest for establishing rates of distributional convergence.
\begin{theorem}[Stein factors from Wasserstein decay] 
\label{thm:stein-factors}
Fix any Lipschitz $h$. %
If an \Ito diffusion has invariant measure $P\in \pset$,
transition semigroup $(\trans{t}{})_{t\geq 0}$,
Wasserstein decay rate $r$,
and infinitesimal generator $\generator{}$ \eqnref{inf-generator}, then
\balignt\label{eqn:stein-function}
	u_h 
		\defeq \int_0^\infty \Esubarg{P}{h(\PVAR)} 
			- \trans{t}{h} \dt
\ealignt
is twice continuously differentiable and satisfies
\balignst
	M_1(u_h) 
		\leq  M_1(h) \int_0^\infty r(t)\dt
	\quad\qtext{and}\quad
	h - \Esubarg{P}{h(\PVAR)}  = \generator{u_h}.
\ealignst
Hence,
$g_h \defeq \half \grad u_h$ solves the Stein equation \eqnref{stein-equation}
with diffusion Stein operator~\eqnref{diffusion-operator}
whenever $\generator{}$ has the form \eqnref{diff-generator}.
If the drift and diffusion coefficients $b$ and $\sigma$ have locally Lipschitz second derivatives and a right inverse $\sigma^{-1}(x)$ for each $x\in\reals^d$ and $h \in C^2$ with bounded second derivatives, then
\begin{equation}
\label{eqn:stein-factor2}
	M_2(u_h) \leq M_1(h)(\beta_1 + \beta_2),
\end{equation}
where
\balignst
\beta_1 &= r(0)(2M_0(\sigma^{-1}) + r(0)M_1(\sigma)M_0(\sigma^{-1}) + r(0)\sqrt{\alpha}), \qtext{and}\\
\beta_2 &= r(0)(e^{\gamma_2}M_0(\sigma^{-1})+e^{\gamma_2}M_1(\sigma)M_0(\sigma^{-1}) + \textfrac{2}{3}e^{\gamma_4}\sqrt{\alpha}) \int_0^\infty r(t)\dt
\ealignst
for $\gamma_\rho \defeq \rho M_1(b)+ \frac{\rho^2-2\rho}{2}M_1(\sigma)^2+\frac{\rho}{2} F_1(\sigma)^2$,	
$\alpha \defeq \frac{M_2(b)^2}{2M_1(b)+4M_1(\sigma)^2}+2F_2(\sigma)^2$.
If, additionally, $\grad^3b$ and $\grad^3\sigma$ are locally Lipschitz and $h\in C^3$
with bounded third derivatives, then, for all $\iota \in (0,1)$,
\balignt
\label{eqn:stein-factor3-holder}
M_{3-\iota}(u_{h}) %
 & \leq M_{1}(h)\frac{1}{K}\left(\frac{1}{\iota}+\int_0^\infty r(t)\dt\right)
\ealignt
for $K > 0$ a constant depending only on $M_{1:3}(\sigma),M_{1:3}(b),$
$M_0(\sigma^{-1}),$ and $r$.
\end{theorem}
\begin{remark}
Thms. 1 and 2 of \citet{PardouxVe01} also bound the solutions of the Stein equation \eqnref{stein-equation}. %
However, for generic Lipschitz $h$, \citep[Thms. 1 and 2]{PardouxVe01} provide inexplicit constants; only guarantee the polynomial growth of $g_h$ and its derivatives, not uniform boundedness; and require bounded $\sigma$, a strong assumption which rules out the heavy-tailed examples of \secref{examples}.
\end{remark}

A first consequence of \thmref{stein-factors}, proved in \secref{constant-lower-bound-proof}, concerns Stein operators \eqnref{diffusion-operator} with constant covariance and stream matrices $a$ and $c$.  In this setting, fast Wasserstein decay implies that the diffusion Stein discrepancy converges to zero only if the Wasserstein distance does (Desideratum \ref{desiderata:det-nonconv}).
\begin{theorem}[Stein discrepancy lower bound: constant $a$ and $c$]
\label{thm:constant-lower-bound}
Consider an \Ito diffusion with diffusion Stein operator $\diffusion{}$ \eqnref{diffusion-operator} for $P \in \pset$, 
Wasserstein decay rate $r$,
constant covariance and stream matrices $a$ and  $c$, 
and Lipschitz drift $b(x) = \half(a+c)\grad \log p(x)$.
If $\sr \defeq \int_0^\infty r(t)\dt$, then
\balign\label{eqn:stein-discrepancy-constant}
&\twowass(Q_n,P) \\ 
	\leq\  3\sr&\textstyle\maxarg{\diffstein{Q_n}{\steinset} , \sqrt[3]{\diffstein{Q_n}{\steinset}\sqrt{2}\,\Earg{\twonorm{G}}^{2}(2M_1(b)+\frac{1}{\sr})^2}}, \notag
\ealign
where $G \in \reals^d$ is a standard normal vector and $M_1(b) \leq \half\opnorm{a+c}M_2(\log p)$.
\end{theorem}
\thmref{constant-lower-bound} in fact provides an explicit upper bound on the Wasserstein distance in terms of the Stein discrepancy and the Wasserstein decay rate.
Under additional smoothness assumptions on the coefficients, the explicit relationship between Stein discrepancy and Wasserstein distance can be improved and extended to diffusions with non-constant diffusion coefficient, as our next result, proved in \secref{nonconstant-lower-bound-proof}, shows.
\begin{theorem}[Stein discrepancy lower bound: non-constant $a$ and $c$]
\label{thm:nonconstant-lower-bound}
Consider an \Ito diffusion for $P \in \pset$ with diffusion Stein operator $\diffusion{}$ \eqnref{diffusion-operator}, Wasserstein decay rate $r$, and Lipschitz drift and diffusion coefficients $b$ \eqnref{drift} and $\sigma$ with locally Lipschitz second derivatives.  If $\sr \defeq \int_0^\infty r(t)\dt$, then 
\baligns%
&\twowass(Q_n,P) \\
	\le2&\max\left(\diffstein{Q_n}{\steinset}  \max(\sr,\beta_1+\beta_2), \sqrt{\diffstein{Q_n}{\steinset}\sqrt{{2}{/\pi}}(\beta_1+\beta_2)\zeta}\right), \notag
\ealigns
for $\beta_1$, $\beta_2$ defined in \thmref{stein-factors}
and  
\baligns
\zeta \defeq \Earg{\twonorm{G}}(1 + 2M_1(b) s_r + M_1^*(m) (\beta_1 + \beta_2))
\ealigns
where $G \in \reals^d$ is a standard normal vector,
$m \defeq a + c$, and
$M_1^*(m) \defeq \sup_{x\neq y} \opnorm{m(x) - m(y)}^*/\twonorm{x-y}$.

If, additionally, $\grad^3b$ and $\grad^3\sigma$ are locally Lipschitz, then
\balignt
\nonumber\twowass(Q_n,P)\leq\, &2 \diffstein{Q_n}{\steinset} \max\bigg (\max(s_{r},\beta_1+\beta_2),\\\label{eqn:refined-nonconstant-lower-bound}
&
e \max\left(\textfrac{d^{1/4}\sqrt{\zeta}}{\sqrt{K}}, \textfrac{\sqrt{d}}{K}\right)
(s_r + \max(\log(\nicefrac{1}{\diffstein{Q_n}{\steinset}}), 1))\bigg),
\ealignt
for a constant $K > 0$ depending only on $M_{1:3}(\sigma),M_{1:3}(b),$
$M_0(\sigma^{-1}),$ and $r$.
\end{theorem}
\begin{remark}
\label{rem:nonconstant-lower-bound}
The $\log(\nicefrac{1}{\diffstein{Q_n}{\steinset}})$ term in \eqnref{refined-nonconstant-lower-bound} reflects the potential non-smoothness of the Stein equation solution $g_h$ studied in \thmref{stein-factors}. Indeed, for $d\geq 2$ and standard multivariate Gaussian $P$, there exist Lipschitz $h$ with infinite $M_2(g_h)$ \cite[Remark 2]{raivc2004multivariate}.
\end{remark}
In \secref{examples}, we will present practically checkable conditions
implying fast Wasserstein decay and discuss both broad families
and specific diffusion-target pairings covered by this theory.

\subsection{Upper bounding the diffusion Stein discrepancy} \label{sec:discrepancy-upper-bound}
In upper bounding the Stein discrepancy, one classically aims to establish rates of convergence to $P$ for specific sequences $(Q_n)_{n=1}^\infty$.
Since our interest is in explicitly computing Stein discrepancies for arbitrary sample sequences, 
our general upper bound in \propref{discrepancy-upper-bound} serves principally to provide sufficient conditions under which the classical diffusion Stein discrepancy converges to zero.
\begin{proposition}[Stein discrepancy upper bound] \label{prop:discrepancy-upper-bound}
Let $\diffusion{}$ be the diffusion Stein operator \eqnref{diffusion-operator} for $P \in \pset$.
If $m \defeq a + c$ and $b$ \eqnref{drift} are $P$-integrable, 
\baligns
	&\diffstein{Q_n}{\steinset}
		\leq \inf_{X\sim Q_n, Z \sim P}\ (\Earg{2\norm{b(X)-b(\PVAR)} +\norm{m(X)-m(\PVAR)}} \\
		&\qquad\qquad\qquad\qquad\qquad+\Earg{(2\norm{b(\PVAR)}+\norm{m(\PVAR)})
		\min(\norm{X-\PVAR},2)})\\
		&\leq \lswass{s}(Q_n,P) (2M_1^{\norm{\cdot}}(b)+M_1^{\norm{\cdot}}(m)) \\
		&+ \lswass{s}(Q_n,P)^t\, 2^{1-t}\, \E[{(2\norm{b(\PVAR)} + \norm{m(\PVAR)})^{s/(s-t)}}]^{(s-t)/s}
\ealigns
for any $s\geq 1$ and $t \in (0,1]$.
Moreover, for $\mu_0 \defeq \Earg{e^{2\norm{b(\PVAR)}+\norm{m(\PVAR)}}}$,
\baligns
\diffstein{Q_n}{\steinset} &\le
\lswass{1}(Q_n,P) (2M_1^{\norm{\cdot}}(b)+M_1^{\norm{\cdot}}(m)) \\
&+ \min(\lswass{1}(Q_n,P), 2)
\log({(e\mu_0)/}{\min(\lswass{1}(Q_n,P), 2)}).
\ealigns
\end{proposition}
This result, proved in \secref{discrepancy-upper-bound-proof}, 
complements the Wasserstein distance lower bounds of \secref{discrepancy-lower-bound} and implies that, for Lipschitz and sufficiently integrable $m$ and $b$, 
the diffusion Stein discrepancy converges to zero whenever $Q_n$ converges to $P$ in Wasserstein distance. 
\subsection{Extension to non-uniform Stein sets}
For any $c_1, c_2, c_3 > 0$, our analyses and algorithms readily accommodate the non-uniform Stein set
\baligns%
\steinset^{c_{1:3}}
  \defeq \bigg\{ g :\reals^d\to\reals^d \bigg|
  \sup_{x\neq y\in \reals^d} \maxarg{\textfrac{\norm{g(x)}^*}{c_1},\textfrac{\norm{\grad g(x)}^*}{c_2}, \textfrac{\norm{\grad g(x) - \grad g(y)}^*}{c_3\norm{x-y}}} \leq 1
  \bigg\}.
\ealigns
This added flexibility can be valuable when tight upper bounds on a reference IPM, like the Wasserstein distance, are 
available for a particular choice of Stein factors $(c_1,c_2,c_3)$.
When such Stein factors are unknown or difficult to compute, 
we recommend the parameter-free classical Stein set and graph Stein set of the sequel as practical defaults,
since the classical Stein discrepancy is strongly equivalent to any non-uniform Stein discrepancy:

\begin{proposition}[Equivalence of non-uniform Stein discrepancies] \label{prop:stein-equiv}
For any $c_1,c_2, c_3 > 0$,
\baligns
\min(c_1,c_2,c_3)\diffstein{Q_n}{\steinset} \leq \diffstein{Q_n}{\steinset^{c_{1:3}}} \leq \max(c_1,c_2,c_3)\diffstein{Q_n}{\steinset}.
\ealigns
\end{proposition}
\begin{remark}
The short proof follows exactly as in \cite[Prop. 4]{GorhamMa15}.
\end{remark}

\section{Sufficient conditions for Wasserstein decay}
\label{sec:examples}

Since the Stein discrepancy lower bounds of \secref{steins_method}
depend on the Wasserstein decay \eqnref{wass-decay-rate} of the chosen diffusion, 
we next provide examples of practically checkable sufficient
conditions for Wasserstein decay and illustrate the process of verifying these conditions for 
a selection of diffusion-target pairings. 
These pedagogical examples 
serve to succinctly illustrate the process of verifying our assumptions 
and do not represent the full scope of applicability.
\subsection{Uniform dissipativity}
It is well known~\cite[see, e.g.,][Eq. 7]{CattiauxGu14} that the Langevin diffusion \diffref{langevin} enjoys exponential Wasserstein decay whenever $\log p$ is $k$-strongly log concave, i.e., when the drift $b = \half \grad \log p$ satisfies
$\inner{ b(x) - b(y)}{x-y}   \leq - \textfrac{k}{2} \twonorm{ x-y }^{2}$
for $k > 0$.
An analogous 
\emph{uniform dissipativity}
condition gives explicit exponential decay for a generic \Ito diffusion:
\begin{theorem}[Wasserstein decay: uniform dissipativity]%
\label{thm:concave-decay}
Fix $k>0$ and $G \psdgt 0$, and 
let $\norm{w}_G^2 \defeq \inner{w}{G w},$ for any vector or matrix $w\in\reals^{d\times d'}, d'\geq 1$.
An \Ito diffusion with drift and diffusion coefficients $b$ and $\sigma$ satisfying
\[
2\inner{b(x) - b(y)}{G(x-y)} + \norm{\sigma(x) - \sigma(y)}_G^2 \leq -k\norm{x-y}_G^2 
\text{ for all } x,y\in\reals^d
\]
has Wasserstein decay rate \eqnref{wass-decay-rate} $r(t) = e^{-kt/2}\sqrt{\nicefrac{\maxeig{G}}{\mineig{G}}}$.
\end{theorem}
\begin{remark}
The proof of \thmref{concave-decay} in \secref{concave-decay-proof} holds even when the drift $b$ is not Lipschitz, yields the same decay rate for $\wasssetarg{2,\twonorm{\cdot}}$, and relies on a synchronous coupling of \Ito diffusions, mimicking \cite[Sec. 1]{CattiauxGu14}.
\end{remark}
Hence, if the drift $b$ of an \Ito diffusion is \emph{$-k/2$-one-sided Lipschitz}, i.e.,
\balign\label{eqn:one-sided-lipschitz}
2\inner{b(x) - b(y)}{G(x-y)} \leq -k\norm{x-y}_G^2 \qtext{for all} x,y\in\reals^d
\ealign
and some $G \psdgt 0$, 
and the diffusion coefficient $\sigma$ is $\sqrt{k'}$-Lipschitz, that is,
\[
\norm{\sigma(x) - \sigma(y)}_G^2 \leq k'\norm{x-y}_G^2  \qtext{ for all } x,y\in\reals^d,
\]
then, whenever $k' < k$, the diffusion 
exhibits exponential Wasserstein decay.
with rate $e^{-(k-k')t/2} \sqrt{\nicefrac{\maxeig{G}}{\mineig{G}}}$.
\begin{example}[Bayesian logistic regression with Gaussian prior]
A one-sided Lipschitz drift arises naturally in the setting of Bayesian logistic regression~\cite{GelmanCaStDuVeRu2014},
a canonical model of binary outcomes $y \in \{-1,1\}$ given measured covariates $v\in \R^d$.
Consider the log density of a Bayesian logistic regression posterior
based on a dataset of $L$ observations $(v_l, y_l)$ and a $\Gsn(\mu, \Sigma)$ prior:
\baligns
\log p(\beta) 
	= \underbrace{-\ \textstyle\frac{1}{2}\twonorm{\Sigma^{-1/2}(\beta-\mu)}^2}_{\text{multivariate Gaussian prior}} 
	\underbrace{-\ \textsum_{l=1}^L
  \log (1 + \exp{-y_l\inner{v_l}{\beta}})}_{\text{logistic regression likelihood}}\ +\ \text{const.}
\ealigns
Here, our inferential target is the unobserved parameter vector $\beta\in \R^d$.
Since 
\[\textstyle
-\Sigma^{-1}
	\psdge \Hess \log p(\beta) 
	= -\Sigma^{-1} -\sum_{l=1}^L\frac{e^{y_l\inner{v_l}{\beta}}}{(1 + e^{y_l\inner{v_l}{\beta}})^2}v_lv_l^{\top}
	\psdge -\Sigma^{-1} -\frac{1}{4}\sum_{l=1}^Lv_lv_l^{\top},
\]
the $P$-targeted preconditioned Langevin diffusion \diffref{preconditioned} drift $b(\beta)=\texthalf \Sigma\grad \log p(\beta)$ 
satisfies \eqnref{one-sided-lipschitz} with $k = 1$ and $G = \Sigma^{-1}$
and
$M_1(b) \leq \texthalf\opnorm{I + \frac{1}{4}\Sigma\sum_{l=1}^Lv_lv_l^{\top}}$.
Hence, the diffusion enjoys geometric Wasserstein decay (\thmref{concave-decay}) and a Wasserstein lower bound on the Stein discrepancy (\thmref{constant-lower-bound}).

\end{example}
\begin{example}[Bayesian Huber regression with Gaussian prior]
\label{ex:huber-regression}
Huber's least favorable distribution provides a robust error model for the regression of a continuous response $y\in\reals$
onto a vector of measured covariates  $v\in\reals^d$ \cite{HuberRo2009}.
Given $L$ observations $(v_l, y_l)$ and a $\Gsn(\mu, \Sigma)$ prior on 
an unknown parameter vector $\beta\in \R^d$, the Bayesian Huber regression log posterior takes the form
\baligns
\log p(\beta) 
	= \underbrace{-\ \textstyle\frac{1}{2}\twonorm{\Sigma^{-1/2}(\beta-\mu)}^2}_{\text{multivariate Gaussian prior}} 
	\ \underbrace{-\ \textsum_{l=1}^L
  \ \rho_c(y_l - \inner{v_l}{\beta})}_{\text{Huber's least favorable likelihood}}\ +\ \text{const.}
\ealigns
where
$\rho_c(r) \defeq \frac{1}{2}r^2 \indic{|r| \le c} + c
(|r| -\frac{1}{2}c) \indic{|r| > c}$
for fixed $c > 0$.
Since $\rho_c'(r) =  \minarg{\maxarg{r, -c}, c}$ is 1-Lipschitz and convex,
and the Hessian of the log prior is $-\Sigma^{-1}$,
the $P$-targeted preconditioned Langevin diffusion \diffref{preconditioned} drift $b(\beta)=\texthalf \Sigma\grad \log p(\beta)$ 
satisfies \eqnref{one-sided-lipschitz} with $k = 1$ and $G = \Sigma^{-1}$
and
$M_1(b) \leq \texthalf\opnorm{I + \Sigma\sum_{l=1}^Lv_lv_l^{\top}}$.
This is again sufficient for exponential Wasserstein decay and a Wasserstein lower bound on the Stein discrepancy.
\end{example}
\subsection{Distant dissipativity, constant $\sigma$}\label{sec:curvature-infinity}
When the diffusion coefficient $\sigma$ is constant with $a\defeq \half\sigma\sigma^\top$ invertible, 
\citet{Eberle2015} showed that 
a \emph{distant dissipativity} condition is sufficient for
exponential Wasserstein decay.
Specifically, if we define a one-sided Lipschitz constant conditioned on a distance $r > 0$,
\[
-\kappa(r) 
	= \sup \{2(b(x) - b(y))^\top a^{-1}(x-y)/r^2 : (x-y)^\top a^{-1}(x-y) = r^2\},
\]
then \citep[Cor. 2]{Eberle2015} establishes exponential Wasserstein decay whenever 
$\kappa$ is continuous with $\liminf_{r\to\infty} \kappa(r) > 0$ and $\int_0^1 r\kappa(r)^- dr < \infty$.  
For a Lipschitz drift, this holds whenever $b$ is dissipative at large distances,
that is, whenever, for some $k > 0$, we have $\kappa(r) \geq k$ for all $r$ sufficiently large \citep[Lem. 1]{Eberle2015}.
\begin{example}[Gaussian mixture with common covariance]
\label{ex:gmm}
Consider an $m$-component mixture density $p(x) = \sum_{j=1}^m w_j \phi_j(x)$,
where the component weights $w_j \ge 0$ sum to one
and 
$\phi_j$ is the density of a $\Gsn(\mu_j, \Sigma)$ distribution on $\reals^d$.
Fix any $x,y\in\reals^d$.
If $\twonorm{\Sigma^{-1/2}(x-y)} = r$, 
the $P$-targeted preconditioned Langevin diffusion \diffref{preconditioned} with drift $b(z) = \half a\grad \log p(z)$ and $a = \Sigma$ satisfies
\baligns
&2(b(x) - b(y))^\top a^{-1}(x-y)
	= (\grad \log p(x) - \grad \log p(y))^\top (x-y) \\
	&= - r^2 + \inner{\Sigma^{-1/2}(\mu(x) - \mu(y))}{\Sigma^{-1/2}(x-y)} 
	\leq -r^2 + r\Delta,
\ealigns
by Cauchy-Schwarz and Jensen's inequality,
for $\Delta \defeq \sup_{j,k} \twonorm{\Sigma^{-1/2}(\mu_j - \mu_k)}$, $\mu(x) \defeq \sum_{j=1}^m \pi_j(x) \mu_j$, and
$\pi_j(x) \defeq \frac{w_j \phi_j(x)}{ p(x)}$.
Moreover, by the mean value theorem, Cauchy-Schwarz, and Jensen's inequality, we have, for each $v\in\reals^d$,
\baligns
&2\inner{\Sigma^{-1/2}(b(x)-b(y))}{v}
	= \inner{\Sigma^{-1/2}(\grad \mu(z)-I)(x-y)}{v} \\
	&\quad= \inner{(\Sigma^{-1/2}S(z)\Sigma^{-1/2}-I)\Sigma^{-1/2}(x-y)}{v} 
	\leq \twonorm{v}\twonorm{\Sigma^{-1/2}(x-y)} \,L, 
\ealigns
for some $z\in\reals^d$,
$S(x) \defeq %
	\half\sum_{j=1}^m\sum_{k=1}^m \pi_j(x) \pi_k(x) (\mu_j-\mu_k)(\mu_j-\mu_k)^\top$,
and $L \defeq \sup_{j,k} |1-\twonorm{\Sigma^{-1/2}(\mu_j - \mu_k)}^2/2|$.
Hence, $b$ is Lipschitz, and $\kappa(r) \geq \half$ when $r > 2\Delta$, so our diffusion enjoys exponential Wasserstein decay \citep[Lem. 1]{Eberle2015} and a Stein discrepancy upper bound on the Wasserstein distance.
\end{example}
\subsection{Distant dissipativity, non-constant $\sigma$}
Using a combination of synchronous and reflection couplings, 
\citet[Thm. 2.6]{Wang2016} showed that diffusions satisfying a 
distant dissipativity condition exhibit exponential Wasserstein decay, even
when the diffusion coefficient $\sigma$ is non-constant.
In \secref{nonconstant-decay-proof}, we combine the coupling strategy of \citep[Thm. 2.6]{Wang2016} with the approach of \citep{Eberle2015} for diffusions with constant $\sigma$ to obtain the following explicit Wasserstein decay rate for distantly dissipative diffusions with bounded $\sigma^{-1}$.
\begin{theorem}[Wasserstein decay: distant dissipativity]
\label{thm:nonconstant-decay}
Let $(\trans{t}{})_{t\geq 0}$ be the transition semigroup  of an \Ito diffusion 
with drift and diffusion coefficients $b$ and $\sigma$. 
Define the truncated diffusion coefficient
\[
\sigma_0(x) = (\sigma(x) \sigma(x)^\top - \lambda_0^2 I)^{1/2}
\qtext{for some}
\lambda_0 \in [0, 1/M_0(\sigma^{-1})]
\]
and the distance-conditional dissipativity function
\balign
\label{eqn:distant-dissipativity}
\kappa(r) = \inf \{ &-2 \alpha (\inner{b(x)-b(y)}{x-y} + \texthalf \norm{\sigma_0(x) - \sigma_0(y)}_F^2 \\
&- \texthalf{\twonorm{(\sigma_0(x) - \sigma_0(y))^\top (x-y)}^2}{/r^2} )/r^2 : \twonorm{x-y} = r \} \nonumber
\ealign
for any
$
m_0 \leq \inf_{x \ne y} \textfrac{\twonorm{ (\sigma_0(x) - \sigma_0(y))^\top (x-y) }}{\twonorm{x-y}}
\qtext{and} 
\alpha \defeq 1/(\lambda_0^2 + m_0^2/4).
$

If the constants $R_0 = \inf\{R\geq 0\::\:\kappa(r)\geq 0,\ \forall\,r\geq R\}$
and $R_1 = \inf\{R\ge R_0\::\:\kappa(r)R(R-R_0)\geq 8,\ \forall\, r\geq R\}$
satisfy $R_0\leq R_1 <\infty$, then %
\balign
\label{eqn:distant-diss-wass-decay-rate}
\twowass(\delta_x \trans{t},\,\delta_y \trans{t}) 
	\leq {2}{\varphi(R_0)^{-1}}e^{-ct}\, \twowass(\delta_x,\delta_y) %
\ealign
for all $x,y \in \reals^d$ and $t \geq 0$, where 
$
\frac{1}{c}\ =\ \alpha\int_0^{R_1}\int_0^s\staticexp{\frac {1}{4}\int_t^su\kappa^{-}(u)\:du}\dt\ds
$,
$\varphi(r)=e^{-\frac {1}{4}\int_0^rs\kappa^{-}(s)\ds}$,
and
$\kappa^{-}(s) = \max(-\kappa(s), 0)$.
\end{theorem}
\begin{remark}
\thmref{nonconstant-decay} holds even when the drift $b$ is not Lipschitz.
\end{remark}

The Wasserstein decay rate \eqnref{distant-diss-wass-decay-rate} in \thmref{nonconstant-decay} has a simple form 
when the diffusion is dissipative at large distances and $\kappa$ is bounded below.
These rates follow exactly as in \citep[Lem. 1]{Eberle2015}.

\begin{corollary}\label{cor:nonconstant-decay}
Under the conditions of \thmref{nonconstant-decay}, suppose that, 
for $R,L \geq 0$ and $K>0$,
$%
\kappa (r)\geq -L \text{ for } r\leq R
\text{ and } \kappa(r)\ge K \text{ for } r>R
$%
. 
Then
\begin{equation*}
\textstyle
\alpha^{-1}c^{-1}\ \leq
\begin{cases}
\frac{e-1}{2}R^2\,+\,e\sqrt{8K^{-1}}\,R\,+\, 4K^{-1}  & \text{if } LR_0^2 \leq 8 \\ %
\frac{8\sqrt{2\pi}}{RL^{1/2}}(L^{-1}+K^{-1})\exp{\frac{LR^2}{8}}+32R^{-2}K^{-2} &  \text{if } LR_0^2 > 8.
\end{cases}
\end{equation*}
\end{corollary}
\begin{example}[Multivariate Student's t regression with pseudo-Huber prior]
\label{ex:studentt-regression}

The multivariate Student's t distribution is also commonly employed as a robust error model for the linear regression of continuous responses $y \in \reals^L$ onto measured covariates $V\in\reals^{L\times d}$~\cite{Zellner1976,Liu1996}.
Under a pseudo-Huber prior \cite{HartleyZi2004}, a Bayesian multivariate Student's t regression posterior takes the form
\baligns
p(\beta) 
	\propto \underbrace{ \textstyle \staticexp{\delta^2 (1 - \sqrt{1 + \twonorm{\beta / \delta}^2})}}_{\text{pseudo-Huber prior}} 
	\ \underbrace{(1
+ \textfrac{1}{\nu}(y - V\beta)^{\top}\Sigma^{-1}(y - V\beta))^{-(\nu + L)/2}}_{\text{multivariate Student's t likelihood}}
\ealigns
for fixed $\delta, \nu > 0$ and $\Sigma\psdgt 0$. 
Introduce the shorthand $\psi_{\lambda}(r) \defeq 2\sqrt{1+r^2/\delta^2}-\lambda^2$ for each $\lambda \in [0,\sqrt{2})$
and $\xi(\beta) \defeq 1 + \frac{1}{\nu}(y - V\beta)^{\top}\Sigma^{-1}(y - V\beta)$.
Since 
\baligns
\grad\log p(\beta) = -2\beta/\psi_0(\twonorm{\beta})
  + (1 + \textfrac{\nu}{L} ){V^{\top} \Sigma^{-1} (y - V\beta)}/\xi(\beta)
\ealigns
is bounded, no $P$-targeted preconditioned Langevin diffusion \diffref{preconditioned} will satisfy the distant dissipativity conditions of \secref{curvature-infinity}.
However, we will show that whenever $V^\top V \psdgt 0$, the Riemannian Langevin diffusion \diffref{riemannian} with 
$\sigma(\beta) = \sqrt{\psi_0(\twonorm{\beta})} I \in\reals^{d\times d}$,
$a(\beta) = \half\psi_0(\twonorm{\beta}) I$, and $b(\beta) = a(\beta) \grad \log p(\beta) +\inner{\grad}{a(\beta)}$
satisfies the Wasserstein decay preconditions of \cororef{nonconstant-decay}.

Indeed, fix any $\lambda_0 \in (0,1/M_0(\sigma^{-1}))=(0,\sqrt{2})$.
Since 
$M_1(\sqrt{\psi_{\lambda}}) \leq \frac{1}{\delta\sqrt{2-\lambda^2}}$,
$M_1(\psi_{\lambda}) \leq \frac{2}{\delta}$,
and $M_2(\psi_{\lambda}) \leq \frac{2}{\delta^2}$, $\sigma_0$, $\sigma$, $a$, and $\grad a$ are all Lipschitz.
The drift $b$ is also Lipschitz, since $\grad \log p$ and the product of $a(\beta)$ and
\baligns
&\textstyle\Hess \log p(\beta) =  
	-2I/\psi_0(\twonorm{\beta})
	+{8\beta\beta^\top}/(\delta^2\psi_0^3(\twonorm{\beta})) \\
	&\textstyle+ \left (1 + \frac{\nu}{L}\right
) (2{V^{\top} \Sigma^{-1}(y-V\beta)(y-V\beta)^{\top}\Sigma^{-1} V}{/\xi^2(\beta)}
  -{V^{\top} \Sigma^{-1} V}{/\xi(\beta)}).
\ealigns
are bounded.  Hence, $\kappa$ \eqnref{distant-dissipativity} is bounded below.
Moreover, the the \Holder continuity of $x\mapsto\sqrt{x}$, Cauchy-Schwarz,
and the triangle inequality imply
\baligns
&\kappa(r) 
	\geq \inf_{\twonorm{\beta-\beta'} = r} \textfrac{2 \alpha}{r^2} (\inner{b(\beta')-b(\beta)}{\beta-\beta'} - \textfrac{d-1}{2}|\textstyle\sqrt{\psi_{\lambda_0}(\twonorm{\beta})} - \sqrt{\psi_{\lambda_0}(\twonorm{\beta'})}|^2) \\
&\geq 
	2\alpha-\textfrac{2\alpha}{r}(\textfrac{d-1}{\delta}
	+M_1(\psi_0) + \textstyle\sup_{\beta}\,(1 + \textfrac{\nu}{L} )\psi_{0}(\twonorm{\beta})\twonorm{V^{\top} \Sigma^{-1} (y - V\beta)}/\xi(\beta) ) \\
&\geq 
	2\alpha-\textfrac{2\alpha}{r}\left(\textfrac{d+1}{\delta}
	+\textstyle\sup_{s}\,(1 + \textfrac{\nu}{L} )\textfrac{2(1+s/\delta)(\twonorm{V^{\top} \Sigma^{-1}y}+ s\opnorm{V^{\top} \Sigma^{-1}V})}{1 + \frac{1}{\nu}\max(0,s/\opnorm{(V^{\top} \Sigma^{-1}V)^{-1}} - \twonorm{\Sigma^{-1}y})^2 }\right).
\ealigns
Letting $\zeta$ represent the supremum in the final inequality, we see that
$\kappa(r) \geq \alpha = 1/\lambda_0^2$ whenever 
$
r \geq 2(\textfrac{d+1}{\delta} + \zeta).
$
Hence, \cororef{nonconstant-decay} delivers exponential Wasserstein decay.
A Wasserstein lower bound on the Stein discrepancy now follows from \thmref{nonconstant-lower-bound}, since $M_2(\sqrt{\psi_0}) \leq \frac{1}{\sqrt{2}\delta^2}$, $M_3(\psi_0) \leq \frac{96}{25\sqrt{5}\delta^3}$,
and $a(\beta)\Hess \log p(\beta)$ is Lipschitz, and hence $M_2(\sigma)$ and $M_2(b)$ are bounded.\end{example}
\section{Computing Stein discrepancies}
\label{sec:programs}
In this section, we introduce a computationally tractable Stein discrepancy that inherits the favorable convergence properties established in \secsref{steins_method} and \secssref{examples}.
We will directly port the spanner discrepancy methodology developed and detailed in \cite{GorhamMa15} and use our new diffusion operators as drop-in replacements for the overdamped Langevin operators advocated in \cite{GorhamMa15}.
While we only explicitly discuss target distributions supported on all of $\reals^d$,
constrained domains of the form $(\alpha_1, \beta_1)\times\cdots\times
(\alpha_d, \beta_d)$ where $-\infty
\le \alpha_i < \beta_i \le \infty$ for all $1\le i \le d$ can be handled
by introducing boundary constraints as in \cite[Section 4.4]{GorhamMa15}.

\subsection{Spanner Stein discrepancies}
\label{sec:spanner-graph-stein}
For any sample $Q_n$, Stein operator $\diffusion{}$, and Stein set $\gset$, the Stein discrepancy $\diffstein{Q_n}{\gset}$ is recovered by solving an optimization problem over functions $g \in \gset$.
For example, if we write $m\defeq a+c$ and
$b(x)\defeq \half\textfrac{1}{p(x)} \inner{\grad}{p(x)m(x)}$,
the classical diffusion Stein discrepancy is the value
\baligns
\diffstein{Q_n}{\steinset}
=
&\textstyle\sup_{g} \textsum_{i=1}^n
q(x_i)(2\inner{b(x_i)}{g(x_i)} + \inner{m(x_i)}{\grad g(x_i)}) \\
&\text{s.t.}
\max(\norm{g(x)}^*,\norm{\grad g(x)}^*,\textfrac{\norm{\grad g(x) - \grad g(y)}^*}{\norm{x - y}}) \leq 1, \forall x, y \in\reals^d.
\ealigns
For all Stein sets, the diffusion Stein discrepancy objective is linear in $g$ and only queries $g$ and $\grad g$ at the $n$ sample points underlying $Q_n$.
However, the classical Stein set $\steinset$ 
constrains $g$ at all points in its domain, resulting in an infinite-dimensional optimization problem.\footnote{When $d=1$, the problem reduces to a finite-dimensional convex quadratically constrained quadratic program with linear objective as in \citep[Thm. 9]{GorhamMa15}.}

To obtain a finite-dimensional problem that is convergence-determining and straightforward to optimize,
we will make use of the \emph{graph Stein sets} of \citep{GorhamMa15}.
For a given graph $G = (V,E)$ with $V = \supp{Q_n}$, the graph Stein set,
\baligns
&\gsteinset{}{G} = \left\{ g : 
	\max(\norm{g(v)}^*, \norm{\grad g(v)}^*,
   	\textfrac{\norm{g(x) - g(y)}^*}{\norm{x - y}},
    \textfrac{\norm{\grad g(x) - \grad g(y)}^*}{\norm{x - y}}) \le 1, \right. \\
    &\textstyle\frac{\norm{g(x) - g(y) - {\grad g(x)}{(x - y)}}^*}{\frac{1}{2}\norm{x - y}^2} \leq 1,
    \textstyle\frac{\norm{g(x) - g(y) -{\grad g(y)}{(x -
    y)}}^*}{\frac{1}{2}\norm{x - y}^2} \leq 1, \; \forall (x,y)\in E, v \in V \bigg\},
\ealigns
imposes boundedness constraints only at sample points and smoothness constraints only
at pairs of sample points enumerated in the edge set $E$.
The graph is termed a \emph{$t$-spanner} \citep{Chew86,PelegSc89} if each edge $(x,y)\in E$ is assigned the weight $\norm{x-y}$, and, for all $x'\neq y' \in V$, there exists a path between $x'$ and $y'$ in the graph with total path weight no greater than $t\norm{x'-y'}$.
Remarkably, for any linear Stein operator $\diffusion{}$, a \emph{spanner Stein discrepancy} $\diffstein{Q_n}{\gsteinset{}{G_t}}$ based on a $t$-spanner $G_t$ is equivalent to the classical Stein discrepancy in the following strong sense, implying Desiderata \ref{desiderata:det-conv} and \ref{desiderata:det-nonconv}.
\begin{proposition}[Equivalence of classical and spanner Stein discrepancies] \label{prop:spanner-equivalence}
If $G_t = (\supp{Q_n}, E)$ is a $t$-spanner for $t \geq 1$, then
\baligns
\diffstein{Q_n}{\steinset} \leq \diffstein{Q_n}{\gsteinset{}{G_t}} \leq \kappa_d t^2\, \diffstein{Q_n}{\steinset}
\ealigns
where $\kappa_d$ is independent of $(Q_n,P,\diffusion{},G_t)$ and depends only
on $d$ and $\norm{\cdot}$.
\end{proposition}
\begin{remark}
The proof  relies on the Whitney-Glaeser extension theorem \cite[Thm. 1.4]{Shvartsman08} of \citet{Glaeser58} and follows exactly as in \citep[Prop. 5 and 6]{GorhamMa15}.
\end{remark}

When $d=1$, a $t$-spanner with exactly $n-1$ edges is obtained in $\bigO{n\log n}$ time for all $t\geq 1$ by introducing edges just between sample points that are adjacent in sorted order.
More generally, if $\norm{\cdot}$ is an $\ell^p$ norm, one can construct a $2$-spanner with $\bigO{\kappa_d' n}$ edges in $\bigO{\kappa_d' n\log(n)}$ expected time where
$\kappa_d'$ is a constant that depends only on the norm
$\norm{\cdot}$ and the dimension $d$~\cite{Har-PeledMe06}.
Hence, a spanner Stein discrepancy can be computed by solving a finite-dimensional convex optimization problem with a linear objective, $\bigO{n}$ variables, and $\bigO{\kappa_d' n}$ convex constraints, making it an appealing choice for a computable quality measure (Desideratum \ref{desiderata:comp}).

\subsection{Decoupled linear programs}
\label{sec:spanner-stein-parallelism}
Moreover, if we choose the norm $\norm{\cdot} = \norm{\cdot}_1$, the graph Stein discrepancy optimization problem decouples into $d$ independent linear programs (LPs) that can be solved in parallel using off-the-shelf solvers.
Indeed, for any $G=(\supp{Q_n},E)$, $\diffstein{Q_n}{\gsteinset{1}{G}}$ equals
\balign\label{eqn:l1-program}
&\textsum_{j=1}^d 
\sup_{\psi_j\in \reals^{n}, \Psi_j\in\reals^{d\times n}}\; 
\textsum_{i=1}^{n} q(x_i) (2b_j(x_i) \psi_{ji} + \textsum_{k=1}^d m_{jk}(x_i)\Psi_{jki}) \\ \notag
\;\text{  s.t. }&
\infnorm{\psi_j} \leq 1, \;
\infnorm{\Psi_j} \leq 1, \text{ and for all } i\neq l, (x_i,x_l) \in E \\ \notag
\;\max\big(
  &\textstyle\frac{|\psi_{ji} - \psi_{jl}|}{\onenorm{x_i - x_l}},
  \textstyle\frac{\infnorm{\Psi_j(e_i - e_k)}}{\onenorm{x_i - x_l}},
  \textstyle\frac{|\psi_{ji} - \psi_{jl} - \inner{\Psi_j e_i}{x_i - x_l}|}{\frac{1}{2}\onenorm{x_i - x_l}^2},
  \textstyle\frac{|\psi_{ji} - \psi_{jl} - \inner{\Psi_j e_i}{x_l - x_i}|}{\frac{1}{2}\onenorm{x_i - x_l}^2}
\big) \le 1, \notag
\ealign
where $\psi_{ji}$ and $\Psi_{jki}$ represent the values $g_j(x_i)$ and $\grad_k g_j(x_i)$ respectively.
Therefore, our recommended quality measure is the $2$-spanner diffusion Stein discrepancy with $\norm{\cdot} = \norm{\cdot}_1$. Its computation is summarized in \algref{spannerstein}.
An efficient implementation of \algref{spannerstein}, integrated with 11 linear program solver options, is publicly available via our Julia package.\footnote{\label{fn:julia}\url{https://jgorham.github.io/SteinDiscrepancy.jl/}}

\begin{algorithm}[t]
   \caption{Spanner diffusion Stein discrepancy, $\diffstein{Q_n}{\gsteinset{1}{G_2}}$ }
   \label{alg:spannerstein}
   \begin{algorithmic}
	\STATE {\bfseries input:} sample $Q_n$, target score $\grad \log p$, covariance coefficient $a$, stream coefficient $c$
	\STATE $G_2 \gets$ 2-spanner of $V = \supp{Q_n}$
	\STATE \textbf{for} $j=1$ \TO $d$ \textbf{do (in parallel)}
	\STATE\hspace{5mm}  $\tau_j \gets$ Optimal value of $j$-th
   coordinate linear program \eqref{eqn:l1-program} with graph $G_2$
	\RETURN $\sum_{j=1}^d \tau_j$
\end{algorithmic}
\end{algorithm}

\section{Numerical illustrations}
\label{sec:experiments}

In this section, we complement the principal theoretical contributions of this work 
with several simple numerical illustrations demonstrating how diffusion Stein discrepancies can be deployed in practice.
We will use our proposed quality measures
to select hyperparameters for biased samplers,
to quantify a bias-variance trade-off for approximate MCMC,
and to compare deterministic and random quadrature rules. %
In each case, we choose experimental settings in which a notion of surrogate ground truth is available for external validation.
We solve all linear programs using Julia for Mathematical
Programming~\citep{LubinDu15} with the Gurobi 6.0.4 solver \cite{Gurobi15} and use
the C++ greedy spanner implementation of \citet{BoutsteBu14} to
compute our $2$-spanners.
Our timings were obtained on a single core of an
Intel Xeon CPU E5-2650 v2 @ 2.60GHz.
Code reconstructing all experiments is available on the 
Julia package site.\cref{fn:julia}
\subsection{A simple example} \label{sec:simple-example}
We first present a simple example to illustrate several Stein discrepancy properties.
For a Gaussian mixture target $P$ (\exref{gmm}) with
$p(x) \propto e^{-\half(x-\frac{\Delta}{2})^2}+e^{-\half(x+\frac{\Delta}{2})^2}$ and $\Delta > 0$, we simulate one i.i.d.\ sequence of sample points from $P$ and a second i.i.d.\ sequence from $\Gsn(-\frac{\Delta}{2}, 1)$, which represents only one component of $P$.
For various mode separations $\Delta$, \figref{simple-example} shows that the Langevin spanner Stein discrepancy \diffref{langevin} applied to the first $n$ Gaussian mixture sample points decreases to zero at a $n^{-1/2}$ rate, while the discrepancy applied to the single mode sequence stays bounded away from zero.  
However, \figref{simple-example} also indicates that larger sample sizes are needed to distinguish between the mixture and single mode sample sequences when $\Delta$ is large.
This accords with our theory (see \exref{gmm}, \cororef{nonconstant-decay}, and \thmref{constant-lower-bound}), which implies that both the Langevin diffusion Wasserstein decay rate and the bound relating Stein to Wasserstein degrade as the mixture mode separation $\Delta$ increases.

\begin{figure}
  \centering
  \includegraphics[width=0.95\textwidth]{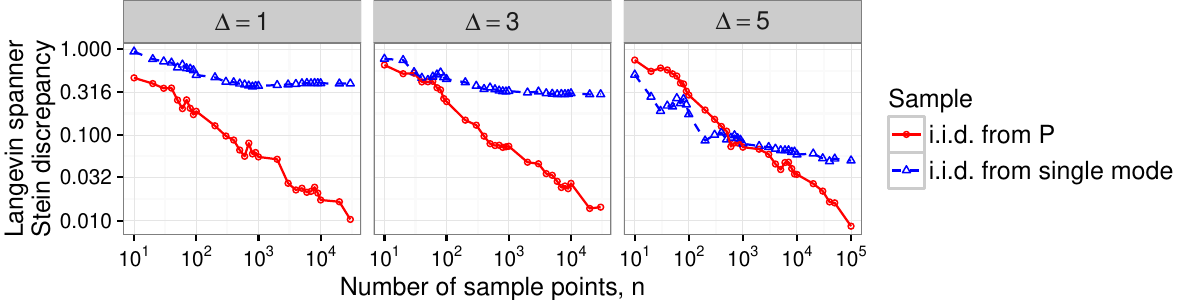}
  \caption{Stein discrepancy for normal mixture target $P$ with $\Delta$ mode separation (\secref{simple-example}).}
  \label{fig:simple-example}
\end{figure}

\subsection{Selecting sampler hyperparameters} \label{sec:tuning-hyperparameters}
Stochastic Gradient Riemannian Langevin Dynamics
(SGRLD) \cite{PattersonTe13} with a constant step size $\eps$ is an approximate MCMC procedure designed to accelerate posterior inference.
Unlike asymptotically correct MCMC algorithms, SGRLD has a stationary distribution that deviates increasingly from its target $P$ as its step size $\eps$ grows. 
On the other hand, if $\eps$ is too small, SGRLD fails to explore the sample space sufficiently quickly.
Hence, an appropriate setting of $\eps$ is paramount for accurate inference.

\begin{figure}[!ht]
  \centering
  \begin{subfigure}[b]{0.95\textwidth}
    \includegraphics[width=\textwidth]{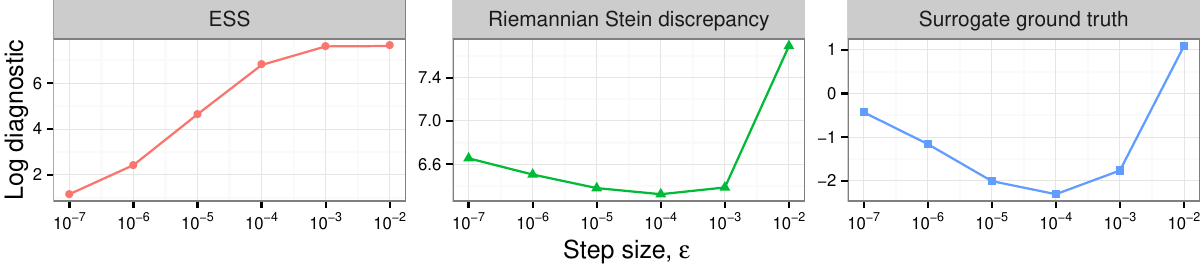}
    \caption{Step size selection criteria and surrogate ground truth (median marginal Wasserstein).
    ESS maximized at $\eps = 10^{-2}$. Stein discrepancy and ground truth minimized at $\eps = 10^{-4}$.}
    \label{fig:julia_compare-hyperparameters-multivariatetpseudohuber-approxwasserstein_diagnostics}
  \end{subfigure}
  \begin{subfigure}[b]{0.95\textwidth}
    \includegraphics[width=\textwidth]{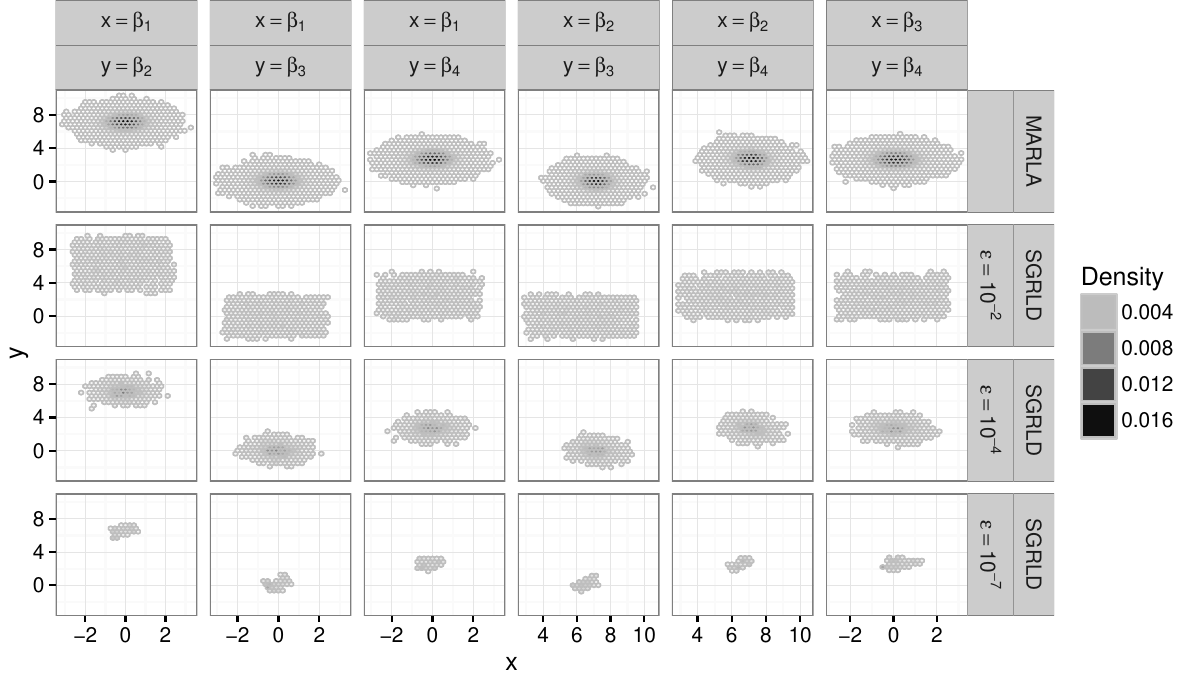}
    \caption{Bivariate hexbin plots. \tbf{Top row:} surrogate ground truth sample ($2\times 10^8$ MARLA points). 
    \tbf{Bottom 3 rows:} $2,000$ SGRLD sample points for various step sizes $\eps$.}
    \label{fig:julia_compare-hyperparameters-multivariatetpseudohuber-approxwasserstein_hexbin_marginal2d_scatter}
  \end{subfigure}
  \caption{Step size selection, stochastic gradient Riemannian Langevin dynamics (\secref{tuning-hyperparameters}).} %
  \label{fig:hyper}
\end{figure}

To demonstrate the value of diffusion Stein discrepancies for hyperparameter selection,  
we analyzed a biometric dataset of $L=202$ athletes from the Australian Institute of Sport
that was previously the focus of a heavy-tailed regression analysis~\cite{Soffritti2011multivariate}. 
In the notation of \exref{studentt-regression}, we used SGRLD to conduct a Bayesian multivariate Student's $t$ regression ($\nu = 10$, $\Sigma = I$) of athlete lean body mass onto red blood count, white blood count, plasma ferritin concentration, and a constant regressor of value ${1}{/\sqrt{L}}$ with a pseudo-Huber prior ($\delta=0.1$) on the unknown parameter vector $\beta \in \reals^4$.

After standardizing the output variable and non-constant regressors and initializing each chain with an approximate posterior mode found by L-BFGS started at the origin,
we ran SGRLD with minibatch size $30$, metric $G(\beta) =1/(2\sqrt{1 + \twonorm{\beta/ \delta}^2})I$, and a variety of step sizes $\eps$ to produce sample sequences of length $200,000$ thinned to length $2,000$. 
We then selected the step size that delivered the highest quality sample 
-- either the maximum effective sample size (ESS, a popular MCMC mixing diagnostic based on asymptotic variance \citep{BrooksGeJoMe11})
or the minimum Riemannian Langevin spanner Stein discrepancy with $a(\beta) = G^{-1}(\beta)$.
The longest discrepancy computation consumed $6$s for spanner construction and $65$s to solve a coordinate optimization problem.
As a surrogate measure of ground truth, we also generated a sample $Q^*$ of size $2\times 10^8$ from the Metropolis-adjusted Riemannian Langevin Algorithm (MARLA) \cite{girolami2011riemann} with metric $G$
and compute the median bivariate marginal Wasserstein distance $\onewass$ between each SGRLD sample and $Q^*$ thinned to $5,000$ points~\citep{gudmundsson2007small}.

\figref{julia_compare-hyperparameters-multivariatetpseudohuber-approxwasserstein_diagnostics}
shows that ESS, which does not account for stationary distribution bias, selects the largest step size available, $\eps = 10^{-2}$. 
As seen in \figref{julia_compare-hyperparameters-multivariatetpseudohuber-approxwasserstein_hexbin_marginal2d_scatter}, this choice results in samples that are greatly overdispersed when compared with the ground truth MARLA sample $Q^*$.
At the other extreme, the selection $\eps = 10^{-7}$ produces greatly underdispersed samples due to slow mixing.
The Stein discrepancy chooses an intermediate value, $\eps = 10^{-4}$.
The same value minimizes the surrogate ground truth Wasserstein measure and produces samples that most closely resemble the $Q^*$ in \figref{julia_compare-hyperparameters-multivariatetpseudohuber-approxwasserstein_hexbin_marginal2d_scatter}.

\subsection{Quantifying a bias-variance trade-off} \label{sec:bias-variance}
Approximate random walk Metropolis-Hastings (ARWMH) \cite{Korattikara2014} with tolerance parameter $\epsilon$ is a
biased MCMC procedure that accelerates posterior inference by approximating the standard MH correction.
Qualitatively, a smaller setting of $\eps$ produces a more faithful approximation of the MH correction and less bias between the chain's stationary distribution and the target distribution of interest.
A larger setting of $\eps$ leads to faster sampling and a more rapid reduction
of Monte Carlo variance, as fewer datapoint likelihoods are computed per sampling step.
We will quantify this bias-variance trade-off as a function of sampling time using the Langevin spanner Stein discrepancy.

In the notation of \exref{huber-regression}, we conduct a Bayesian Huber regression analysis ($c=1$) of 
the log radon levels in $1,190$ Minnesota households~\cite{Gelman2012multilevel} as a function of the log amount of uranium in the county,
an indicator of whether the radon reading was performed in a basement, and an intercept term.
A $\Gsn(0,I)$ prior is placed on the coefficient vector $\beta$.
We run ARWMH with minibatch size $5$ and two settings of the tolerance threshold $\eps$ ($0.1$ and $0.2$) for $10^7$ likelihood
evaluations, discard the sample points from the first $10^5$ evaluations, and thin
the remaining points to sequences of length $1,000$. 
\figref{approxmh_ais} displays the Langevin spanner Stein discrepancy applied to the first
$n$ points in each sequence as a function of the likelihood evaluation
count, which serves as a proxy for sampling time. 
As expected, the higher tolerance sample ($\eps=0.2$) is of
higher Stein quality for a small computational budget but is eventually
overtaken by the $\eps=0.1$ sample with smaller asymptotic bias.
The longest discrepancy computation consumed $0.8$s for the spanner and $20.1$s for a coordinate LP.

To provide external support for the Stein discrepancy quantification, 
we generate a Metropolis-adjusted Langevin chain~\citep{RobertsTw96} of length $10^8$ as a
surrogate $Q^*$ for the target $P$ and display several measures of expectation error between $X\sim Q_n$ and $Z\sim Q^*$ in \figref{approxmh_ais}:
the normalized predictive error $\max_l{
|\Earg{\inner{X - Z}{v_l/\infnorm{v_l}}}|}$ for $v_l$ the $l$-th datapoint covariate vector,
the mean error $\frac{\max_j |\Earg{X_j-\PVAR_j}|}{\max_j
|\Esubarg{Q^*}{\PVAR_j}|}$, and the second moment error $\frac{\max_{j,k}
|\Earg{X_jX_k-\PVAR_j\PVAR_k}|}{\max_{j,k} |\Esubarg{Q^*}{\PVAR_j\PVAR_k}|}$.
We see that the Stein discrepancy provides comparable results without the need for an additional surrogate chain.
\begin{figure}
  \centering
  \includegraphics[width=0.95\textwidth]{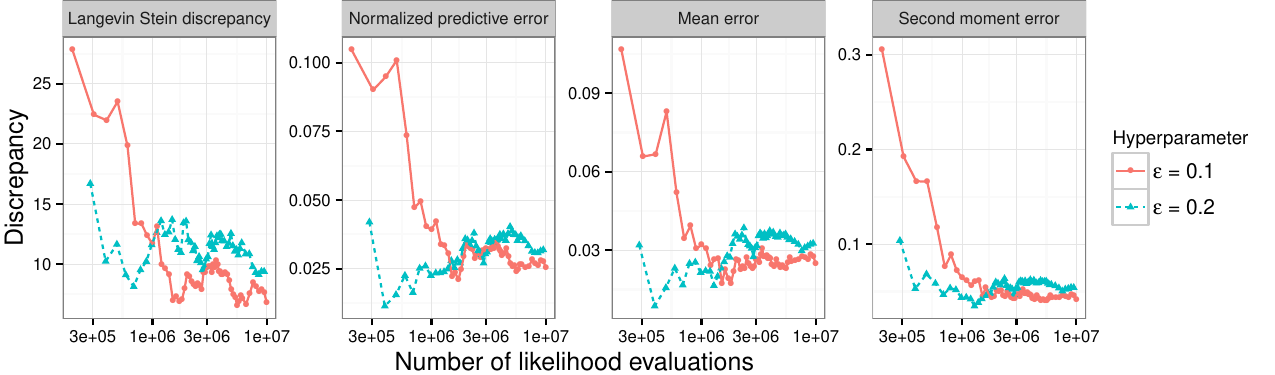}
  \caption{Bias-variance trade-off curves for approximate random walk MH (\secref{bias-variance}).} %
  \label{fig:approxmh_ais}
\end{figure}

\subsection{Comparing quadrature rules} \label{sec:convergence-rates-gmm}
Stein discrepancies can also measure the quality of deterministic sample sequences designed to improve upon Monte Carlo sampling.
For the Gaussian mixture target of \secref{simple-example}, 
\figref{quadrature} compares the median quality of 50 sample sequences 
generated from four quadrature rules recently studied in \cite[Sec. 4.1]{Lacoste2015sequential}: i.i.d.\ sampling from $P$, Quasi-Monte Carlo (QMC) sampling using a deterministic quasirandom number generator, Frank-Wolfe (FW) kernel herding~\cite{Chen2010,Bach2012}, and fully-corrective Frank-Wolfe (FCFW) kernel herding~\cite{Lacoste2015sequential}.
The quality judgments of the Langevin spanner Stein discrepancy \diffref{langevin} closely mimic those of the $L^1$ Wasserstein distance $\wass{}$, which is computable for simple univariate targets~\cite{Vallender74}.
Each Stein discrepancy was computed in under $0.03$s.

Under both diagnostics and as previously observed in other metrics~\cite{Lacoste2015sequential}, the i.i.d.\ samples are typically of lower median quality than their deterministic counterparts.
More suprisingly and in contrast to past work focused on very smooth function classes~\cite{Lacoste2015sequential}, FCFW underperforms FW and QMC in our diagnostics for larger sample sizes.  
Apparently FCFW, which is heavily optimized for smooth function integration, has sacrificed approximation quality for less smooth test functions.
For example, \figref{quadrature} shows that QMC offers a better quadrature estimate than FCFW for $h_{1}(x) = \max\{0, 1
- \min_{j\in\{1,2\}} |x-\mu_j|\}$, a $1$-Lipschitz approximation to the indicator of being within one
standard deviation of a mode.

\begin{figure}[t!]
  \centering
  \begin{subfigure}[t]{0.58\textwidth}
    \centering
    \includegraphics[width=\textwidth]{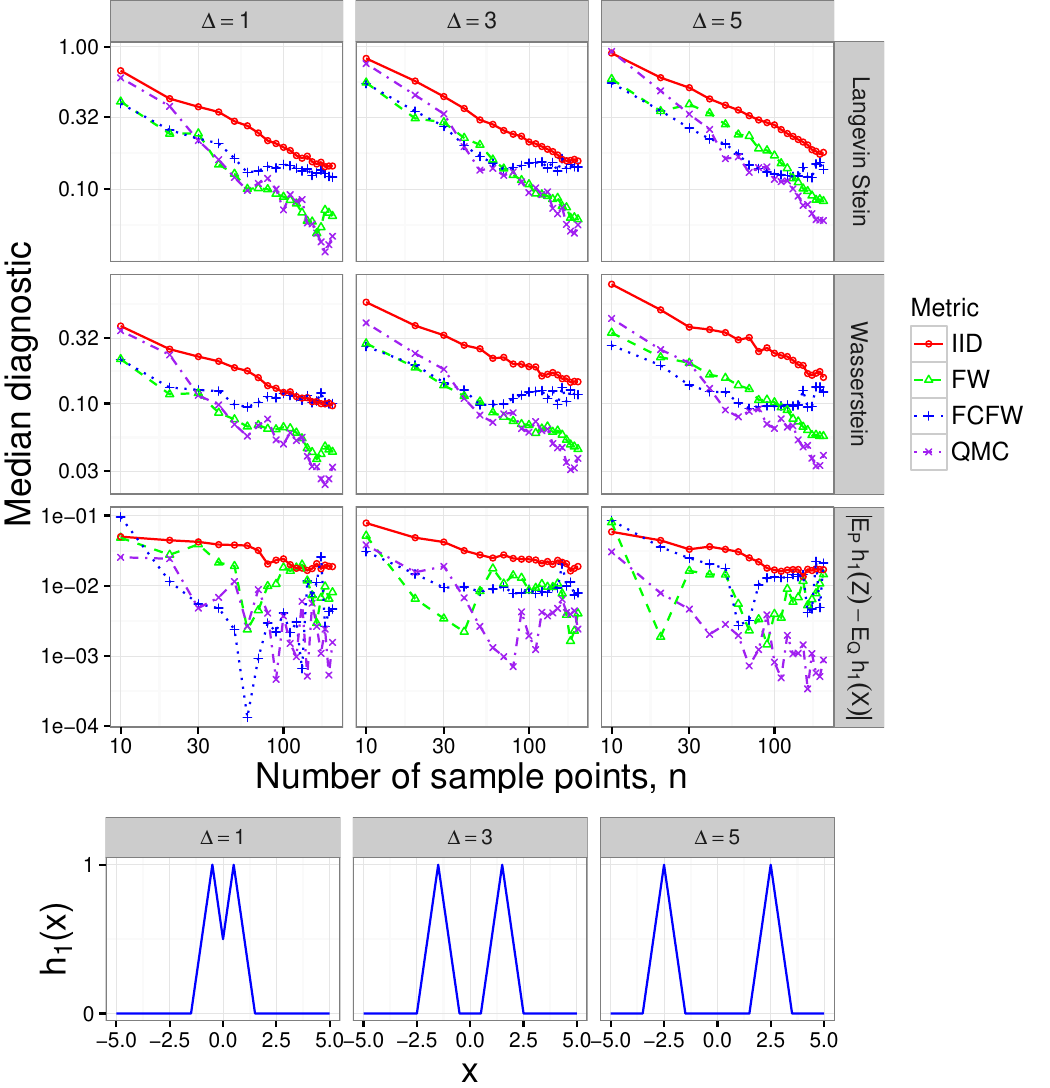}
  \end{subfigure}
  \begin{subfigure}[t]{0.4\textwidth}
    \centering
    \includegraphics[width=\textwidth]{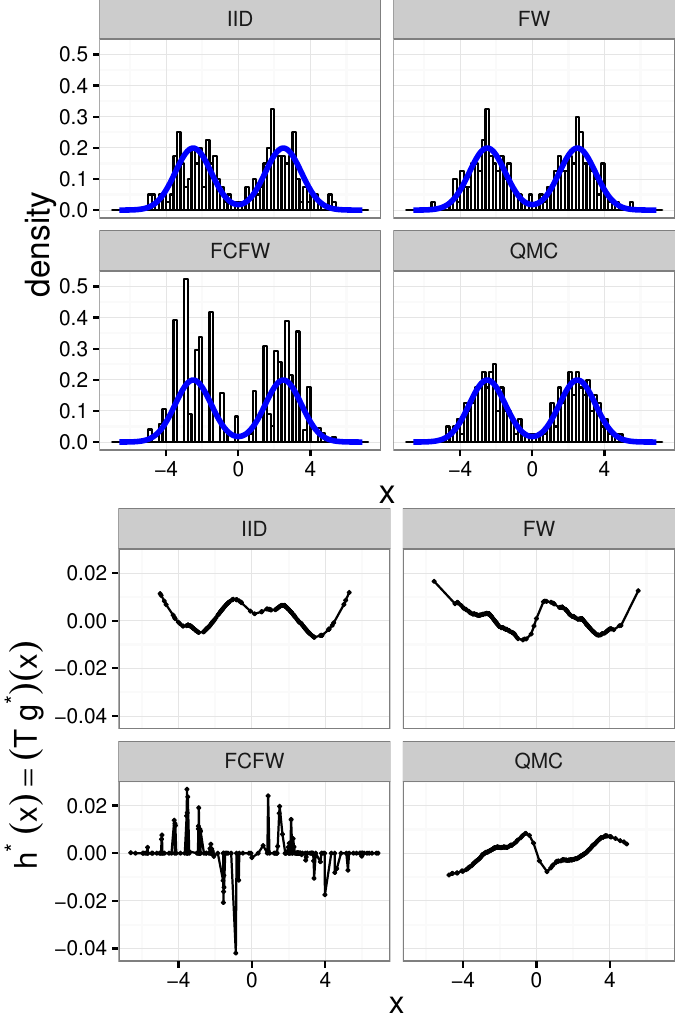}
  \end{subfigure}
  \caption{\tbf{Left:}
    Quadrature rule quality comparison for Gaussian mixture targets $P$ with mode separation $\Delta$ (\secref{convergence-rates-gmm}).
    \textbf{Right:} (Top) Sample histograms with $p$ overlaid ($\Delta=5$, $n=200$). (Bottom) Optimal discriminating test functions $h^* = \diffusion{g^*}$ from Stein program.
  }
  \label{fig:quadrature}
\end{figure}

In addition to providing a sample quality score, the Stein discrepancy optimization problem produces an optimal Stein function $g^*$ and an associated test function $h^* = \diffusion{g^*}$ that is mean zero under $P$ and best distinguishes the sample $Q_n$ from the target $P$.
\figref{quadrature} gives examples of these maximally discriminatve functions $h^*$ for a target mode separation of $\Delta = 5$ and length $200$ sequences from each quadrature rule. 
We also display the associated sample histograms with overlaid target density.
The optimal FCFW function reflects the jagged nature of the FCFW histogram.

\section{Connections and conclusions}
\label{sec:conclusion}
We developed quality measures suitable for comparing the fidelity of arbitrary ``off-target'' sample sequences by generating infinite collections of known target expectations.
\paragraph{Alternative quality measures}
The score statistic of~\citet{FanBrGe06} and the Gibbs sampler convergence criteria of~\citet{ZellnerMi95} account for some sample biases but sacrifice differentiating power by 
exploiting only a finite number of known target expectations.
For example, when $P = \Gsn(0,1)$, the score statistic~\citep{FanBrGe06} cannot differentiate two samples with the same means and variances.
Maximum mean discrepancies (MMDs) over characteristic reproducing kernel Hilbert spaces~\citep{GrettonBoRaScSm06} do detect arbitrary distributional biases but are only computable when the chosen kernel functions can be integrated under the target. 
In practice, one often approximates MMD using a sample from the target, but this requires a separate trustworthy sample from $P$.

While we have focused on the graph and classical Stein sets of~\citep{GorhamMa15}, our diffusion Stein operators can also be paired with the reproducing kernel Hilbert space unit balls advocated in~\citep{OatesGiCh2016,ChwialkowskiStGr2016,LiuLeJo16,GorhamMa17} to form tractable \emph{kernel diffusion Stein discrepancies} or with the random feature functions advocated in~\citep{HugginsMa2018} to form \emph{random feature diffusion Stein discrepancies}.
We have also restricted our attention to Stein operators arising from diffusion generators.
These take the form $\diffarg{g}{x} = \textfrac{1}{p(x)}\inner{\grad}{p(x)m(x)g(x)}$ with $m = a + c$ for $a(x)$ positive semidefinite and $c(x)$ skew-symmetric.  More generally, if the matrix $m$ possesses eigenvalues having a negative real part, then the resulting operator need not correspond to a diffusion process.  Such operators fall into the class of \emph{pseudo-Fokker Planck} operators which have been studied in the context of quantum optics \cite{Risken}.  As noted in \cite{drummond1980,walls1980} it is possible to obtain corresponding stochastic dynamics in an extended state space by introducing complex-valued noise terms;
these operators may merit further study in future work.

\paragraph{Alternative inferential tasks}
While our chief motivation is sample quality measurement, our work is also directly applicable to a variety of inferential tasks that currently rely on the Langevin operator introduced by \citep{GorhamMa15,OatesGiCh2016}, including control variate design~\citep{OatesGiCh2016}, 
goodness-of-fit testing~\citep{ChwialkowskiStGr2016,LiuLeJo16}, variational inference~\citep{Liu2016stein,Ranganath2016,ChenMaGoBrOa2018}, and importance sampling~\citep{LiuLe2017}.
The Stein factor bounds of \thmref{stein-factors} can also be used, in the manner of \citep{Mattingly10,joulin2010,hairer2011spectral}, to characterize the error of numerical discretizations of diffusions.  
These works convert bounds on the solutions of Poisson equations -- Stein factors -- into central limit theorems for $\Esubarg{Q_n}{h(\QVAR)}-\Esubarg{P}{h(\PVAR)}$, confidence intervals for $\Esubarg{P}{h(\PVAR)}$, and mean-squared error bounds for the estimate $\Esubarg{Q_n}{h(\QVAR)}$.
\citet{TehThVo14} and \citet{VollmerZyTe16} extended these approaches to obtain error estimates for approximate discretizations of the Langevin diffusion on $\reals^d$, while, independently of our work, \citet{HugginsZo2017} established error estimates for \Ito diffusion approximations with biased drifts and constant diffusion coefficients.
By \thmref{stein-factors}, their results also hold for \Ito diffusions with non-constant diffusion coefficients.
Following the release of the present paper and with the aim of analyzing discretization error for the overdamped Langevin diffusion, \citet[Thm. 3.1]{FangShXu2018} derived multivariate Stein factor bounds for a class of strongly log-concave distributions. 
Our \thmref{stein-factors} with the choice $\iota = 1/\log(1/\eps)$ provides Stein factors of the same form but applies also to non-log-concave targets and more general diffusions. 

\paragraph{Alternative targets}
Our exposition has focused on the Wasserstein distance $\wass$,
which is only defined for distributions with finite means.
A parallel development could be made for the Dudley metric~\citep{Muller97} to target distributions with undefined mean.
The work of~\citet{Cerrai2001} also suggests that the Lipschitz condition on our
drift and diffusion coefficients can be relaxed.

\appendix
\section{Proof of \propreflow{diffusion-zero}}
\label{sec:diffusion-zero-proof}

Fix any $g\in\steinset$. Since $g$ and $\grad g$ are bounded and $b$, $a$, and $c$ are $P$-integrable, $\Esubarg{P}{\diffarg{g}{Z}}$ is finite.  Define the ball $\ball_{r} = \{x \in \reals^d \, : \, \norm{x}_2 \leq r \}$ with $n_r(z)$ the outward facing unit normal vector for each $z$ on the boundary $\partial \ball_r$.  Since $\pvar\mapsto p(\pvar)(a(\pvar)+c(\pvar))g(\pvar)$ is in $C^1$, %
we may apply the dominated convergence theorem and then the divergence theorem to obtain
\baligns 
\Esubarg{P}{\diffarg{g}{\PVAR}} &= \lim_{r\rightarrow \infty}\textint_{\ball_r} \inner{\grad}{p(\pvar)(a(\pvar)+c(\pvar))g(\pvar)}\, d\pvar\\
&= \lim_{r\rightarrow \infty}\textint_{\partial\ball_r} \inner{n_r(\pvar)}{(a(\pvar) + c(\pvar))g(\pvar)p(\pvar)}\, d\pvar.
 \ealigns
 Let $f(r) = M_0(g)\int_{\partial\ball_r}\opnorm{a(\pvar)+c(\pvar)}\,p(\pvar) \,d\pvar$.  Since $g$ and $n_r$ are bounded, 
 $$\textstyle
  \int_{\partial \ball_r}\inner{n_r(\pvar)}{(a(\pvar)+c(\pvar))g(\pvar)p(\pvar)}\,d\pvar	\leq f(r).
 $$
The coarea formula \cite{ambrosio2000functions} and the integrability of $a$ and $c$ further imply that
$$\textstyle
	\int_{0}^\infty f(r)\,dr = \int_{\reals^d} M_0(g)\opnorm{a(\pvar)+c(\pvar)}\,p(\pvar) \,d\pvar < \infty.
$$
Hence, $\lim\inf_{r\rightarrow \infty} f(r) = 0$, and therefore
$
\Esubarg{P}{\diffarg{g}{\PVAR}} = 0.
$

\section{Proof of \thmreflow{stein-factors}} \label{sec:stein-factors-proof}
Fix any $x\in\reals^d$ and $h \in \twowassset$ with $\Esubarg{P}{h(\PVAR)} = 0$.
Since the drift and diffusion coefficients are Lipschitz, \cite[Thm.~3.4]{Khasminskii11} guarantees that the diffusion 
$(\process{t}{x})_{t\geq0}$ is well-defined.
Using the shorthand $\sr \defeq \int_0^\infty r(t)  \dt$, we will show that the posited function $u_h$ \eqnref{stein-function} exists and solves the \emph{Poisson equation}
\balign\label{eqn:inf-equation}
h = \generator{u_h}
\ealign
with infinitesimal generator $\generator{}$,
that $u_h$ is Lipschitz, 
that $u_h$ has a continuous Hessian,
that $u_h$ has a bounded and \Holder continuous Hessian under additional smoothness assumptions.

\paragraph{Existence of $u_h$ and solving the Poisson equation \eqnref{inf-equation}}
\global\long\def\steinbanach{L}

Consider the set $\steinbanach \defeq (1 + \twonorm{x}^2)C_0(\reals^d) = \lbrace (1 + \twonorm{x}^2)f \, : \, f\in C_0(\reals^d) \rbrace$, where $C_0(\reals^d)$ is the set of continuous functions vanishing at infinity.  Equipped with the norm $\norm{f}_{\steinbanach} = \sup_{x\in \reals^d}|f(x)|/(1 + \twonorm{x}^2)$, the set $\steinbanach$ is a Banach space~\citep{rockner2006kolmogorov}.   As noted in \citep{Doersek2010Semigroup}, the space $\steinbanach$ can also be characterized as the closure of the set of bounded continuous functions, $C_b(\reals^d)$, in the set $\lbrace f:\mathbb{R}^{d}\rightarrow\mathbb{R}\,:\,\norm{f}_\steinbanach < \infty\rbrace$.
To discuss the well-posedness of the Poisson equation \eqnref{inf-equation},
we first show that the transition semigroup of an \Ito diffusion is strongly continuous on $\steinbanach$.

\begin{proposition}
The transition semigroup $\fulltrans{}$ of an \Ito diffusion with Lipschitz drift and diffusion coefficients is strongly continuous on $\steinbanach$.
\end{proposition}
\begin{proof}
Fix any $f\in \steinbanach$ and $x\in\reals^d$.
We first show that $\transarg{t}{f}{x}$ converges pointwise to $f(x)$ as $t\rightarrow 0^{+}$.   
Since the associated \Ito process $\fullprocess{x}$ is almost surely pathwise continuous~\cite[Thm.~3.4]{Khasminskii11} and $f$ is continuous in a neighborhood of $x$, it follows that $f(\process{t}{x})\rightarrow f(x)$ as $t\rightarrow 0^{+}$, almost surely.  Moreover, \cite[Sec. 5, Cor. 1.2]{Friedman1975Stoch} implies that
$$\textstyle
\Earg{\sup_{0\leq t\leq 1}|f(\process{t}{x})|} 
	\leq \norm{f}_{\steinbanach}(1 + \Earg{\sup_{0\leq t \leq 1}\twonorm{\process{t}{x}}^2})
	\leq C\norm{f}_{\steinbanach}(1 + \twonorm{x}^2),
$$
for some $C > 0$ depending only on $M_1(b)$ and $M_1(\sigma)$.  
The dominated convergence theorem now yields the desired pointwise convergence.

To prove the strong continuity of $\fulltrans{}$, it suffices, by \cite[Thm. I.5.8, p. 40]{Engel2000Semigroup}, to verify that $\fulltrans{}$ is weakly continuous, i.e., that $l(\trans{t}{f}) \rightarrow l(f)$, as $t \rightarrow 0^+$, for all elements $l$ of the dual space $\steinbanach^*$.  To this end, fix any $l \in \steinbanach^*$.  By the Riesz-Markov theorem for $\steinbanach$ \cite[Theorem 2.4]{Doersek2010Semigroup}, there exists a finite signed Radon measure $\mu$  such that
\begin{equation}
\label{eqn:dual_norm}\textstyle
	l(f) = \int_{\mathbb{R}^d} f(x)\mu(dx) \text{ and }\int_{\mathbb{R}^d} ( 1 + \twonorm{x}^2)|\mu|(dx) = \lVert l \rVert_{\steinbanach^*},
\end{equation}
for $\lVert \cdot \rVert_{\steinbanach^*}$ the dual norm.  
By Jensen's inequality and \cite[Sec. 5, Cor. 1.2]{Friedman1975Stoch},
\begin{align*}
\forall t,
	\twonorm{\transarg{t}{f}{x}} &\textstyle\leq \Earg{|f(\process{t}{x})|} \leq \lVert f \rVert_{\steinbanach} \mathbb{E}\left[1 + \twonorm{\process{t}{x}}^2\right] \leq C\lVert f \rVert_{\steinbanach} (1 + \twonorm{x}^2).
\end{align*}
Since $1 + \twonorm{x}^2$ is $|\mu|$-integrable by \eqnref{dual_norm}, dominated convergence gives
$$\textstyle
	\lim_{t\rightarrow 0^{+}} l(\trans{t}{f}) = 
	\lim_{t\rightarrow 0^{+}}\int_{\mathbb{R}^d} \transarg{t}{f}{x}\mu(dx) = \int_{\mathbb{R}^d}f(x)\mu(dx) = l(f),
$$
yielding the result.
\end{proof}

Consider the infinitesimal generator $\generator{}$ of the semigroup $\fulltrans{}$ on $\steinbanach$ with 
\[\textstyle
\dom(\generator{})=\big\lbrace f \in \steinbanach \, : \lim_{t\rightarrow 0^{+}} \frac{\trans{t}{f}-f}{t} \mbox{ exists in the } \lVert \cdot \rVert_{\steinbanach} \mbox{ norm} \big\rbrace.
\]
Since $P_t$ is strongly continuous on $\steinbanach$ and $h \in \steinbanach$ with $M_1(h)\leq 1$ and $\Esubarg{P}{h(Z)}=0$, \cite[Prop. 1.5]{EthierKu86} implies that
\begin{equation*}\textstyle
h-\trans{t}{h}=-\generator{}\int_{0}^{t}\trans{s}{h}\,ds=\generator{}u_{h,t}
\qtext{for}
u_{h,t} \defeq -\int_{0}^{t}\trans{s}{h}\,ds.
\end{equation*}
The stationarity of $P$ and the definitions of $\twowass$ and $r$ imply that
\balignst
\norm{\trans{t}{h}}_\steinbanach  &= \norm{P_t h - E_P[h]}_\steinbanach = \sup_{x\in\mathbb{R}^{d}} |\Esubarg{P}{P_t h(x) - P_th(Z)}|/(1+\twonorm{x}^2) \\
	&\leq \sup_{x\in\mathbb{R}^{d}}\textfrac{\E_{P}[\twowass(\delta_{x}\trans{t}{}, \delta_{Z} \trans{t}{})]}{1+\twonorm{x}^{2}} 
	\leq r(t)\sup_{x\in R^{d}}\textfrac{\Esubarg{P}{\twonorm{x-Z} }}{1+\twonorm{x}^{2}},
\ealignst
and hence $\lVert \trans{t}{h}\rVert_{\steinbanach}\rightarrow0$ as $t\rightarrow\infty$,
since $P$ has a finite mean, and $r(t) \to 0$ as $t\to\infty$ as $r$ is integrable and monotonic. 
Arguing similarly, 
\balignst
\lVert u_{h,t}-u_{h,t'}\rVert_{\steinbanach} 
	&\textstyle \leq \norm{\int_{t}^{t'}\E_{P}[\twowass(\delta_{x}\trans{s}{},\delta_Z \trans{s}{})]ds}_{\steinbanach} \\
	\displaystyle&\leq \sup_{x\in R^{d}}\textfrac{\Esubarg{P}{\twonorm{x-Z} }}{1+\twonorm{x}^{2}}\textint_{t}^{t'}r(s)\,ds.
\ealignst
Thus, it follows that $( u_{h,t})_{t>0}$ is a Cauchy sequence in $\steinbanach$ with limit $u_{h}=\int_{0}^{\infty}\trans{s}{h}\,ds\in\steinbanach$.
Thus,  $(h-\trans{t}{h},u_{h,t})\rightarrow(h,u_{h})$  in the graph norm on $\steinbanach\times\steinbanach$, and since $\generator{}$ is closed \cite[Cor. 1.6]{EthierKu86}, $u_{h}\in\dom(\generator{})$ and
$h=\generator{}u_{h}.$

\begin{remark}
The choice of the Banach space is crucial for the argument above.  
As noted in \cite{manca2008kolmogorov} and contrary to the claim in \cite{Barbour90}, 
the semigroup $\fulltrans{}$  fails to be strongly continuous over the Banach space
$\widetilde\steinbanach \defeq (1 + \twonorm{x}^2)C_b(\reals^d)$
when $\fullprocess{x}$ is an Ornstein-Uhlenbeck process, i.e., a Langevin diffusion \diffref{langevin} with a multivariate Gaussian invariant measure.
\end{remark}

\paragraph{Lipschitz continuity of $u_h$}
To demonstrate that $u_h$ is Lipschitz, we choose an arbitrary $v\in \reals^d$, 
and apply the definition of the Wasserstein distance, the assumed decay rate, 
and the integrability of $r$ to obtain
\baligns
&\textstyle\twonorm{u_h(x+v) - u_h(x)}
	\leq\int_0^\infty \twonorm{\Earg{h(\process{t}{x})-h(\process{t}{x+v})}}\,dt  \notag \\
	&\textstyle\leq \int_0^\infty \wass(\delta_x \trans{t},\,\delta_{x+v} \trans{t}) \dt
	\leq \wass(\delta_x,\delta_{x+v})\,\sr
	= \twonorm{v}\,\sr < \infty.
\ealigns

\paragraph{Continuity of $\Hess u_h$}%
Since $u_h \in \dom(\generator{})$ is a continuous solution of the Poisson equation \eqnref{inf-equation}, and since the infinitesimal generator agrees with the characteristic operator of a diffusion when both are defined~\cite[p. 129]{Oksendal2013},
Thm.~5.9 of \cite{Dynkin1965} implies that $u_h \in C^2$.
\paragraph{Boundedness of $\Hess u_h$}
Instantiate the additional preconditions of \eqnref{stein-factor2}, and assume that $M_0(\sigma^{-1}), F_2(\sigma), M_2(b) < \infty$, or else \eqnref{stein-factor2} is vacuous.
\lemref{semigroup-derivative-estimates}, established in \secref{semigroup-derivative-estimates-proof}, shows that the semigroup $\trans{t}{h}$ admits a bounded continuous Hessian, which is integrable in $t$.

\begin{lemma}[Semigroup Hessian estimate]%
\label{lem:semigroup-derivative-estimates}
Suppose that the drift and diffusion coefficients $b$ and $\sigma$ of an \Ito diffusion are Lipschitz with Lipschitz gradients and locally Lipschitz second derivatives.  If the transition semigroup $(\trans{t}{})_{t\geq 0}$ has Wasserstein decay rate $r$,
and $\sigma(x)$ has a right inverse $\sigma^{-1}(x)$ for each $x\in\reals^d$, then,
for all $t > 0$ and any $f \in C^2$ with bounded first and second derivatives,
$\trans{t}{f}$ is twice continuously differentiable with
\balign
\label{eqn:semigroup-gradient-bound}
M_1(\trans{t}{f})
	\leq &\, M_1(f) r(t) \qtext{and}%
\ealign
\balign
M_2(\trans{t}{f}) \label{eqn:semigroup-hessian-bound} 
	\leq &\inf_{t_0 \in (0,t]}\ M_1(f) r(t-t_0) \sqrt{\textfrac{1}{t_0}}e^{t_0\gamma_2} M_0(\sigma^{-1}) \\
	&+\ M_1(f) r(t-t_0) r(0) e^{t_0\gamma_2} M_1(\sigma)M_0(\sigma^{-1}) \notag\\
	&+\ M_1(f) r(t-t_0) \sqrt{t_0}\, r(0) e^{t_0 \gamma_4}  \textfrac{2}{3}\sqrt{\alpha} \notag
\ealign
for $\gamma_\rho \defeq \rho M_1(b)+ \frac{\rho^2-2\rho}{2}M_1(\sigma)^2+\frac{\rho}{2} F_1(\sigma)^2$,
$\alpha \defeq \frac{M_2(b)^2}{2M_1(b)+4M_1(\sigma)^2}+2F_2(\sigma)^2$.
\end{lemma}

The dominated convergence theorem now implies that the Hessian of $u_h$ is obtained
by differentiating twice under the integral sign.  
The advertised bound \eqnref{stein-factor2} on $\Hess u_h$ follows by replacing the infimum on the right-hand side of the semigroup bound \eqnref{semigroup-hessian-bound} with the selection $t_0 = \min(t,1)$, applying the bound $e^{\min(t,1)\gamma_\rho}\leq e^{\gamma_\rho}$ for each $\gamma_\rho$ and $t$, and integrating the result over $t$.

\paragraph{\Holder continuity of $\Hess u_h$}
Finally, instantiate the additional preconditions of \eqnref{stein-factor3-holder},
and fix any $\iota \in (0,1)$. 
The integral representation \eqnref{stein-function} of $u_h$, the dominated convergence theorem, and Jensen's inequality imply
\balignst
M_{1-\iota}(\Hess u_h)
= M_{1-\iota}\left (-\int_0^{\infty} \Hess\trans{t}{h}\dt\right )
\le \int_0^{\infty} M_{1-\iota}(\Hess \trans{t}{h})\dt.
\ealignst
When $t \leq 1$, 
a seminorm interpolation lemma (\lemref{interpol} in the supplement),
a semigroup third derivative estimate (\lemref{ThirdDerivFlow} in the supplement) with $t_0 =\min(t,1)$,
and the semigroup second derivative estimate of \lemref{semigroup-derivative-estimates} with $t_0 =\min(t,1)$
imply
\balignst
M_{1-\iota}(\Hess \trans{t}{h}) 
	\leq M_1(h)2^{\iota}M_0(\Hess \trans{t}{h})^{\iota} M_1(\Hess
\trans{t}{h})^{1-\iota} 
	\leq M_1(h) t^{\iota/2 - 1} / K_1
\ealignst
for some constant $K_1 > 0$ depending only on $M_{1:3}(b),M_{1:3}(\sigma),
M_0(\sigma^{-1}),$ and $r$. Thus
$\int_{0}^{1} M_{1-\iota}(\Hess \trans{t}{h})\dt \leq \frac{2M_1(h)}{K_1 \iota}$.
For $t > 1$, \lemsref{interpol}, \lemssref{ThirdDerivFlow}, and
\lemssref{semigroup-derivative-estimates}
and the integrability of $r$ yield
\balignst
\int_{1}^{\infty}M_{1-\iota}(\Hess \trans{t}{h})\dt \leq
M_1(h)\frac{2}{K_2}\int_{1}^{\infty}r(t-1)\dt = M_1(h)\frac{2}{K_2} s_{r}
\ealignst
for a constant $K_2 > 0$ again depending only on
$M_{1:3}(b),M_{1:3}(\sigma), M_0(\sigma^{-1}),$ and $r$. Combining these bounds and
choosing $K=\min(K_1, K_2)/2$ completes the proof. 
An explicit constant $K$ can be obtained by tracing constants through the proof of
\lemref{ThirdDerivFlow}.
\section{Proof of \lemreflow{semigroup-derivative-estimates}}%
\label{sec:semigroup-derivative-estimates-proof}
Fix any $x\in\reals^d$ and $f : \reals^d \to \reals$ in $C^2$ %
with bounded first and second derivatives, and let $(\process{t}{x})_{t\geq0}$ be an \Ito diffusion solving the stochastic differential equation \eqnref{diffusion}
with starting point $\process{0}{x}=x$, underlying Wiener process $(W_t)_{t\geq 0}$, and transition semigroup $(\trans{t}{})_{t\geq0}$.
Our proof is divided into five pieces establishing, for each $t > 0$, 
the Lipschitz continuity of $\trans{t}{f}$,
the Lipschitz continuity of $\grad \trans{t}{f}$,
the continuity of $\Hess \trans{t}{f}$,
an initial bound on $\Hess \trans{t}{f}$,
and the infimal bound \eqnref{semigroup-hessian-bound} on $\Hess \trans{t}{f}$.

\paragraph{Lipschitz continuity of $\trans{t}{f}$}
The semigroup gradient bound \eqnref{semigroup-gradient-bound}
follows from the Lipschitz continuity of $f$ and the definitions of the Wasserstein decay rate and the Wasserstein distance, as, for any $y\in\reals^d$ and $t\geq 0$,
\baligns
\transarg{t}{f}{x} - \transarg{t}{f}{y}
	&= \Earg{f(\process{t}{x}) - f(\process{t}{y})}
	\leq M_1(f) \twowass(\delta_x \trans{t},\,\delta_y \trans{t}) \\
	&\leq M_1(f) r(t)\, \twowass(\delta_x,\delta_y)
	= M_1(f) r(t) \twonorm{x-y}.
\ealigns
\paragraph{Lipschitz continuity of $\grad \trans{t}{f}$}
Fix any $v,v'\in\reals^d$.
Under our smoothness assumptions on $b$ and $\sigma$, \citep[Theorem V.40]{Protter2005} implies that $(\process{t}{x})_{t\geq0}$ is twice continuously differentiable in $x$.  The first directional \emph{derivative flow} $(V_{t,v})_{t\geq 0}$ solves the \emph{first variation equation},
\balign
\label{eqn:first-variation}
dV_{t,v} = \grad b(\process{t}{x}) V_{t,v} \dt + \grad \sigma(\process{t}{x}) V_{t,v} \dW_t
\qtext{with} V_{0,v} = v,
\ealign
obtained by formally differentiating the equation \eqnref{diffusion} defining $(\process{t}{x})_{t\geq 0}$ with respect to $x$ in the direction $v$.
The second directional derivative flow\\ $(U_{t,v,v'})_{t\geq 0}$ solves 
the \emph{second variation equation},
\balign\notag
dU_{t,v,v'} &= (\grad b(\process{t}{x}) U_{t,v,v'} + \Hess b(\process{t}{x})[V_{t,v'}]V_{t,v})\dt  \\ 
	&+ (\grad \sigma(\process{t}{x}) U_{t,v,v'} + \Hess \sigma(\process{t}{x})[V_{t,v'}]V_{t,v}) \dW_t
\ \text{ with }\ U_{0,v,v'} = 0, \label{eqn:second-variation}
\ealign
obtained by differentiating \eqnref{first-variation} with respect to $x$ in the direction $v'$.

Since $f$ has bounded first and second derivatives, the dominated convergence
theorem implies that, for each $t \geq 0$, $\trans{t}{f}$ is twice differentiable with
\balign
\notag
\inner{\grad \transarg{t}{f}{x}}{v} &= \Earg{\inner{\grad f(\process{t}{x})}{V_{t,v}}} \qtext{and}\\
v'^\top\Hess \transarg{t}{f}{x}v&= \Earg{V_{t,v'}^\top\Hess f(\process{t}{x})V_{t,v} + \inner{\grad f(\process{t}{x})}{U_{t,v,v'}}} 
\label{eqn:semigroup-hessian}
\ealign
obtained by differentiating under the integral sign.
\lemref{derivative-flow-bounds}, proved in \secref{derivative-flow-bounds}, 
justifies the exchanges of derivative and expectation by ensuring that the derivative flows have moments bounded uniformly in $x$.
\begin{lemma}[Derivative flow bounds]
\label{lem:derivative-flow-bounds}
Suppose that $(\process{t}{x})_{t\geq0}$ is an \Ito diffusion with 
starting point $\process{0}{x} = x \in\reals^d$, 
driving Wiener process $(W_t)_{t\geq 0}$,
and Lipschitz drift and diffusion coefficients $b$ and 
$\sigma$ with Lipschitz gradients and locally Lipschitz second derivatives.
If $(V_{t,v})_{t\geq 0}$ and $(U_{t,v,v'})_{t\geq 0}$ respectively solve the stochastic
differential equations \eqnref{first-variation} and \eqnref{second-variation} for $v,v'\in\reals^d$,
then, for any $\rho \geq 2$,
\balign
\label{eqn:derivative-flow-bound}
\Earg{\twonorm{V_{t,v}}^\rho} 
	&\leq \twonorm{v}^\rho\, e^{t \gamma_\rho} \qtext{and}\\
\label{eqn:Hessian-flow-bound}
\Earg{\twonorm{U_{t,v,v'}}^2} 
	&\leq \alpha\twonorm{v}^2\twonorm{v'}^2 t e^{t \gamma_4}
\ealign
for 
$\gamma_\rho \defeq \rho M_1(b)+ \frac{\rho^2-2\rho}{2}M_1(\sigma)^2+\frac{\rho}{2} F_1(\sigma)^2$
and 
$\alpha \defeq \frac{M_2(b)^2}{2M_1(b)+4M_1(\sigma)^2}+2F_2(\sigma)^2$.
\end{lemma}
\noindent
Since $\grad f$ and $\Hess f$ are bounded,
and $(V_{t,v})_{t\geq 0}, (V_{t,v'})_{t \geq 0},$ and $(U_{t,v,v'})_{t\geq 0}$ have second moments bounded uniformly in $x$ by \lemref{derivative-flow-bounds},
the Hessian formula \eqnref{semigroup-hessian} implies that
$\Hess \trans{t}{f}$ is bounded and hence that $\grad \trans{t}{f}$ is Lipschitz continuous for each $t \geq 0$.

\paragraph{Continuity of $\Hess \trans{t}{f}$}
Hereafter we assume that $M_0(\sigma^{-1}) < \infty$, as the semigroup Hessian bound \eqnref{semigroup-hessian-bound} is otherwise vacuous.

The Lipschitz continuity of $f$ and the \Ito diffusion moment bound of \citep[Thm. 3.4, part 4]{Khasminskii11}
together imply that
\[
	\Earg{f(\process{t}{x})^2} \leq \Earg{(|f(x)| + \twonorm{\process{t}{x} - x} M_1(f))^2} < \infty
\]
for all $t\geq 0$.
Since $\sigma^{-1}$ is bounded, and $\grad b$ and $\grad \sigma$ are bounded and Lipschitz, 
\cite[Prop. 3.2]{FournieLaLeLiTo1999} gives the following Bismut-Elworthy-Li-type formula for the directional derivative of $\trans{t}{f}$ for each $t > 0$:
\balign
\textstyle
\inner{\grad{\transarg{t}{f}{x}}}{v} = \frac{1}{t} \Earg{ f(\process{t}{x}) \int_0^t \inner{\sigma^{-1}(\process{s}{x})V_{s,v}}{\dW_s} },\notag
\ealign
By interchanging derivative and integral, the dominated convergence theorem now delivers the Hessian expression
\balign
\label{eqn:BEL2}
&v'^\top\Hess \transarg{t}{f}{x}v 
	= \Earg{J_{1,x} + J_{2,x} + J_{3,x}} \qtext{for}\\ %
&J_{1,x} \textstyle
	\defeq \frac{1}{t} {\inner{\grad {f}({\process{t}{x}})}{V_{t,v'}}
	\int_0^{t}\inner{\sigma^{-1}(\process{s}{x})V_{s,v}}{\dW_s}}, \notag \\	
&J_{2,x} \textstyle
	\defeq \frac{1}{t} 
	{f(\process{t}{x}) \int_0^t \inner{\grad \sigma^{-1}(\process{s}{x})[V_{s,v'}] V_{s,v}}{\dW_s}}, \qtext{and}\notag \\
&J_{3,x} \textstyle
	\defeq \frac{1}{t}
	{f(\process{t}{x}) \int_0^t \inner{\sigma^{-1}(\process{s}{x})U_{s,v,v'}}{\dW_s}}, \notag 
\ealign
for each $t>0$, provided that 
$J_{1,x}, J_{2,x}$, and $J_{3,x}$ are continuous in $x$.
The requisite continuity follows from the 
Lipschitz continuity of $\grad f$ and $f$, 
the boundedness of $\sigma^{-1}$, $\grad \sigma$, and $\Hess \sigma$, 
and the controlled moment growth and \Holder continuity of $\fullprocess{x}$, $(V_{t,v})_{t\geq 0}, (V_{t,v'})_{t \geq 0},$ and $(U_{t,v,v'})_{t\geq 0}$ as functions of $x$ \citep[Theorem V.40]{Protter2005}.  The dominated convergence theorem further implies that $\Hess \trans{t}{f}$ is continuous for each $t > 0$.

\paragraph{Initial bound on $\Hess \trans{t}{f}$}
Now, we fix any $t > 0$ and turn to bounding $\Hess \trans{t}{f}$ in terms of $M_1(f)$, by bounding the expectations of $J_{1,x}, J_{2,x}$, and $J_{3,x}$ of \eqnref{BEL2} in turn.

To control $\Earg{J_{1,x}}$, we apply Cauchy-Schwarz, the \Ito isometry \citep[Eqs. 7.1 and 7.2]{Friedman1975Stoch}, 
the derivative flow bound \eqnref{derivative-flow-bound},
and the fact $e^{s\gamma_2} \leq e^{t\gamma_2}$ for all $s \leq t$
to obtain 
\baligns
\Earg{J_{1,x} }
	&\textstyle\leq \frac{1}{t} \sqrt{\Earg{\inner{\grad f({\process{t}{x}})}{V_{t,v'}}^2}
		\E[{(\int_0^{t}\inner{\sigma^{-1}(\process{s}{x})V_{s,v}}{\dW_s})^2}]} \notag \\
	&\textstyle\leq \frac{1}{t} M_1({f})\sqrt{\Earg{\twonorm{V_{t,v'}}^2}
		\int_0^{t}\Earg{\twonorm{\sigma^{-1}(\process{s}{x})V_{s,v}}^2}\ds} \notag \\	
	&\textstyle\leq \frac{1}{t} M_1({f})M_0(\sigma^{-1})\sqrt{\Earg{\twonorm{V_{t,v'}}^2}
		\int_0^{t}\Earg{\twonorm{V_{s,v}}^2}\ds} \notag \\	
	&\textstyle\leq \frac{1}{t} M_1(f) M_0(\sigma^{-1})\twonorm{v'}\twonorm{v}\sqrt{
		e^{t \gamma_2}\int_0^{t}e^{s \gamma_2} \ds} \\
	&\textstyle\leq \sqrt{\frac{1}{t}} e^{t\gamma_2} M_1(f)  M_0(\sigma^{-1})\twonorm{v'}\twonorm{v},
\ealigns
where we have adopted the definition of $\gamma_\rho$ given in \lemref{derivative-flow-bounds}.

To control $\Earg{J_{2,x}}$, we will first rewrite the unbounded quantity $f(\process{t}{x})$
in terms of more manageable semigroup gradients.
To this end, we note that, since $\trans{t-s}{f} \in C^2$ for all $s \in [0,t]$,
we may apply \Ito's formula  \citep[Thm. 7.1]{Friedman1975Stoch}  to $(s,x) \mapsto \trans{t-s}{f}(x)$ to obtain
the identity
\balign\label{eqn:bel-ito}
\textstyle
f(\process{t}{x}) = \transarg{t}{f}{x} + \int_0^t \inner{\grad \transarg{t-s}{f}{\process{s}{x}}}{\sigma(\process{s}{x}) \dW_s}.
\ealign
Now we may rewrite $\Earg{J_{2,x}}$ as
\baligns
\Earg{J_{2,x}}
	=&\textstyle\ \frac{1}{t}\E\left[\transarg{t}{f}{x} \int_0^t \inner{\grad \sigma^{-1}(\process{s}{x})[V_{s,v'}]V_{s,v}}{\dW_s}\right. \\
	&\textstyle+ \left.\int_0^t \inner{\grad \transarg{t-s}{f}{\process{s}{x}}}{\sigma(\process{s}{x}) \dW_s} \int_0^t \inner{\grad \sigma^{-1}(\process{s}{x})[V_{s,v'}]V_{s,v}}{\dW_s}\right] \\
	=&\textstyle\ \frac{1}{t}\Earg{\int_0^t\inner{\grad \transarg{t-s}{f}{\process{s}{x}}}{\sigma(\process{s}{x})\grad \sigma^{-1}(\process{s}{x})[V_{s,v'}]V_{s,v}} \ds } \\
	=&\textstyle -\frac{1}{t}\Earg{\int_0^t\inner{\grad \transarg{t-s}{f}{\process{s}{x}}}{\grad \sigma(\process{s}{x})[V_{s,v'}]\sigma^{-1}(\process{s}{x})V_{s,v}} \ds },
\ealigns
where we have used Dynkin's formula  \citep[Eq. 7.11]{Friedman1975Stoch}, the \Ito isometry,  and the chain rule, 
\begin{equation}
\grad \sigma^{-1}(x)[v] = -\sigma^{-1}(x)\grad \sigma(x)[v]\sigma^{-1}(x) \label{eq:volDeriv}.
\end{equation}
Finally, we bound $\Earg{J_{2,x}}$ using Cauchy-Schwarz, the semigroup gradient bound \eqnref{semigroup-gradient-bound}, 
the derivative flow bound \eqnref{derivative-flow-bound}, and the fact that $s\mapsto r(t-s)e^{s\gamma_2}$ is increasing:
\baligns
\Earg{J_{2,x}} 
	&\textstyle\leq \frac{1}{t} M_1(\sigma)M_0(\sigma^{-1})
		\int_0^{t}M_1(\trans{t-s}{f})\Earg{\twonorm{V_{s,v'}}\twonorm{V_{s,v}}} \ds \notag \\
	&\textstyle\leq \frac{1}{t} M_1(\sigma)M_0(\sigma^{-1})
		\int_0^{t}M_1(\trans{t-s}{f})\sqrt{\Earg{\twonorm{V_{s,v'}}^2}\Earg{\twonorm{V_{s,v}}^2}} \ds \notag \\
	&\textstyle\leq \frac{1}{t} M_1(\sigma)M_0(\sigma^{-1})
		M_1(f) \twonorm{v'}\twonorm{v} \int_0^{t} r(t-s) e^{s\gamma_2} \ds \notag\\
	&\textstyle\leq r(0) e^{t\gamma_2} M_1(\sigma)M_0(\sigma^{-1})
		M_1(f) \twonorm{v'}\twonorm{v}.
\ealigns

To control $\Earg{J_{3,x}}$, we again appeal to Dynkin's formula and the \Ito isometry to obtain
\baligns
\Earg{J_{3,x}}
	=&\textstyle\ \frac{1}{t} \E\left[\transarg{t}{f}{x} \int_0^t \inner{\sigma^{-1}(\process{s}{x})U_{s,v,v'}}{\dW_s}\right.  \\
	&\textstyle+ \left.\int_0^t \inner{\grad \transarg{t-s}{f}{\process{s}{x}}}{\sigma(\process{s}{x}) \dW_s} \int_0^t \inner{\sigma^{-1}(\process{s}{x})U_{s,v,v'}}{\dW_s}\right] \\
	=&\textstyle\ \Earg{\int_0^{t}\inner{\grad \transarg{t-s}{f}{\process{s}{x}}}{U_{s,v,v'}} \ds},
\ealigns
and we bound this expression using Cauchy-Schwarz, Jensen's inequality, the semigroup gradient bound \eqnref{semigroup-gradient-bound}, the second derivative flow bound \eqnref{Hessian-flow-bound}, and the fact that $s\mapsto r(t-s)e^{s\gamma_4}$ is increasing:
\baligns
\Earg{J_{3,x}}
	&\textstyle\leq \frac{1}{t} \int_0^{t}M_1(\trans{t-s}{f})\Earg{\twonorm{U_{s,v,v'}}} ds
	\leq \frac{1}{t} \int_0^{t}M_1(\trans{t-s}{f})\sqrt{\Earg{\twonorm{U_{s,v,v'}}^2}} ds \\ 
	&\textstyle\leq \frac{1}{t} M_1(f) \sqrt{\alpha}\twonorm{v'}\twonorm{v} \int_0^{t} r(t-s)  \sqrt{s} e^{s \gamma_4} \ds \\
	&\textstyle\leq \frac{2}{3}\sqrt{t}\, r(0) e^{t \gamma_4} M_1(f) \sqrt{\alpha}\twonorm{v'}\twonorm{v},
\ealigns
where $\alpha$ is defined in \lemref{derivative-flow-bounds}.
The advertised result \eqnref{semigroup-hessian-bound} for $t_0 = t$ follows by summing the bounds developed for $\Earg{J_{1,x}},\Earg{J_{2,x}}$, and $\Earg{J_{3,x}}$.

\paragraph{Infimal bound on  $\Hess \trans{t}{f}$}
To obtain the infimum over $t_0 \in (0,t]$ in \eqnref{semigroup-hessian-bound}, we adapt an argument of \cite[Prop.~1.5.1]{Cerrai2001}.
Specifically, fix any $t_0 \in (0,t]$.
Our work thus far shows that $v'^\top\Hess \transarg{t_0}{\tilde f}{x}v  \leq M_1(\tilde f) \zeta(t_0)$ for a real-valued function $\zeta$ and $\tilde f \in C^2$ with bounded first and second derivatives.
Since we now know that $\trans{t-t_0}{f} \in C^2$ with bounded first and second derivatives, the Markov property of the diffusion and the first derivative bound \eqnref{semigroup-gradient-bound} yield
\balign
 &v'^\top\Hess \transarg{t}{f}{x}v
	= v'^\top\Hess \transarg{t_0}{\trans{t-t_0}{f}}{x}v \notag \\
	&\leq M_1(\trans{t-t_0}{f}) \zeta(t_0) \leq M_1(f) r(t-t_0) \zeta(t_0). \notag
\ealign
\subsection{Proof of \lemreflow{derivative-flow-bounds}: Derivative flow bounds}
\label{sec:derivative-flow-bounds}
Fix any $\rho \geq 2$ and $v\in\reals^d$. Since Dynkin's formula and Cauchy-Schwarz give
\baligns
&\textstyle\Earg{\twonorm{V_{t,v}}^\rho}
	= \twonorm{v}^\rho 
	+ \E\left[\int_0^t \rho\inner{V_{s,v} \twonorm{V_{s,v}}^{\rho-2}}{\grad b(\process{s}{x}) V_{s,v}}\right. \\
	&\textstyle+ \frac{\rho}{2}\twonorm{V_{s,v}}^{\rho-4}((\rho-2)\twonorm{V_{s,v}^\top \grad \sigma(\process{s}{x})[V_{s,v}]}^2
	+ \left.\twonorm{V_{s,v}}^2\fronorm{\grad \sigma(\process{s}{x})[V_{s,v}]}^2) \ds\right] \\
	&\textstyle\leq \twonorm{v}^\rho
	+\int_0^t (\rho M_1(b)+ \frac{\rho^2-2\rho}{2}M_1(\sigma)^2+\frac{\rho}{2} F_1(\sigma)^2) \Earg{\twonorm{V_{s,v}}^{\rho}}\ds,
\ealigns
the advertised result \eqnref{derivative-flow-bound} follows from \Gronwall's inequality.

Now fix any $v, v'\in\reals^d$, and define $U_t \defeq U_{t,v,v'}$.
Dynkin's formula and multiple applications of Cauchy-Schwarz and Young's inequality give
\balignst
\Earg{\twonorm{U_t}^2}
    =\ &\E\left[\int_0^t 2\inner{U_s}{\grad b(\process{s}{x})U_s + \Hess b(\process{s}{x})[V_{s,v'}]V_{s,v}}\right. \\
    &+ \left.\fronorm{\grad \sigma(\process{s}{x})[U_s]+\Hess \sigma(\process{s}{x})[V_{s,v'}]V_{s,v}}^2\ds\right]  \\
    \leq\ &\E\left[\int_0^t 2\twonorm{U_s}^2M_1(b) + 2\twonorm{U_s}\twonorm{V_{s,v}}\twonorm{V_{s,v'}}M_2(b)\right. \\
    &+ \left.2\fronorm{\grad \sigma(\process{s}{x})[U_s]}^2+2\fronorm{\Hess \sigma(\process{s}{x})[V_{s,v'}]V_{s,v}}^2 \ds\right] \\
    \leq &\int_0^t (2M_1(b)+2F_1(\sigma)^2+\eps)\Earg{\twonorm{U_s}^2} \\
    &+ (M_2(b)^2/\eps+2F_2(\sigma)^2)\Earg{\twonorm{V_{s,v}}^2\twonorm{V_{s,v'}}^2} \ds
\ealignst
for any $\eps > 0$.
Letting $\gamma_\rho = \rho M_1(b)+ \frac{\rho^2-2\rho}{2}M_1(\sigma)^2+\frac{\rho}{2} F_1(\sigma)^2$,
we see that, by Cauchy-Schwarz and our derivative flow bound \eqnref{derivative-flow-bound},
\balignst
\int_0^t\Earg{\twonorm{V_{s,v}}^2\twonorm{V_{s,v'}}^2} \ds
	&\leq \int_0^t\sqrt{\Earg{\twonorm{V_{s,v}}^4}\Earg{\twonorm{V_{s,v'}}^4}} \ds \\
	&\leq \int_0^t\twonorm{v}^2\twonorm{v'}^2\, e^{s \gamma_4} \ds
	= \twonorm{v}^2\twonorm{v'}^2\frac{e^{t \gamma_4}-1}{\gamma_4}.
\ealignst
Hence, if we choose $\epsilon = \gamma_4 - (2M_1(b)+2F_1(\sigma)^2)$ 
and define $\alpha = M_2(b)^2/\epsilon+2F_2(\sigma)^2$ 
we may write
\balignst
\Earg{\twonorm{U_t}^2}
    \leq \alpha\twonorm{v}^2\twonorm{v'}^2\frac{e^{t \gamma_4}-1}{\gamma_4}
    + \int_0^t \gamma_4\Earg{\twonorm{U_s}^2} \ds.
\ealignst
Gronwall's inequality now yields the result \eqnref{Hessian-flow-bound} via
\balignst
\Earg{\twonorm{U_t}^2}
    &\leq \alpha\twonorm{v}^2\twonorm{v'}^2\left(\frac{e^{t \gamma_4}-1}{\gamma_4}
    + \int_0^t \frac{e^{s \gamma_4}-1}{\gamma_4}\gamma_4 e^{(t-s)\gamma_4 } \ds\right) 
    = \alpha\twonorm{v}^2\twonorm{v'}^2 t e^{t \gamma_4}.
\ealignst
\section{Proof of \thmreflow{constant-lower-bound}}
\label{sec:constant-lower-bound-proof}

We first derive the result for $\norm{\cdot} = \twonorm{\cdot}$.
Without loss of generality, assume $h \in \twowassset$ with $\Esubarg{P}{h(\PVAR)} = 0$.  Our high-level strategy is to relate the Wasserstein distance to the Stein discrepancy via the Stein equation \eqnref{stein-equation} with diffusion Stein operator $\diffusion{}$ \eqnref{diffusion-operator}.
Since the infinitesimal generator $\generator{}$ \eqnref{inf-generator} has the form \eqnref{diff-generator} by \thmref{invariance},
\thmref{stein-factors} implies that there exists a continuously differentiable solution 
$g_h$ to the the Stein equation 
$h(x) = \diffarg{g_h}{x}$
satisfying $M_0(g_h) \leq \sr M_1(h) \leq \sr$.
Since boundedness alone is insufficient to declare that $g_h$ falls into a scaled copy of the classical Stein set $\steinset$, we will develop a smoothed version of the Stein solution with greater regularity.

Since $a$ and $c$ are constant, $b(x)=\half(a+c)\grad \log p(x)$. Fix any $s > 0$ and consider the convolution $g_{h,s}(x) \defeq \Earg{g_{h}(x+sG)}$.  If the smoothing level $s$ is small, the Lipschitz continuity of $h$ implies that that $\diffarg{g_{h,s}}{x}$ provides a close approximation to $h(x)$ for each $x\in\reals^d$:
\balign \label{eqn:constant-smooth-stein-bound}
&h(x)
	\leq \Earg{h(x+sG)} + M_1(h)s\Earg{\twonorm{G}}\\ \notag
	&\leq  \Earg{\textfrac{1}{p(x+sG)}\inner{\grad}{p(x+sG)(a+c)g_h(x+sG)}} + s\Earg{\twonorm{G}} \\ \notag
	&\leq  2\Earg{\inner{b(x+sG)}{g_h(x+sG)}} + \Earg{\inner{a + c}{\grad g_h(x + sG)}} +  s\Earg{\twonorm{G}}\\ \notag
	&\leq  \diffarg{g_{h,s}}{x}+ s \Earg{\twonorm{G}}(1 + 2M_1(b) M_0(g_{h})).
\ealign
Moreover, by our next lemma, proved in \secref{convolution-smoothing}, the smoothed Stein solution admits a bounded Lipschitz gradient 
$\grad g_{h,s}(x) = \Earg{\grad g_{h}(x+sG)}$.
\begin{lemma}[Smoothing by Gaussian convolution] \label{lem:convolution-smoothing}
Let $G\in\reals^d$ be a standard normal random vector, and fix $s > 0$.
If $f : \reals^d \to \reals$ is bounded and measurable, and $f_s(x) \defeq \Earg{f(x+s G)}$, then
\baligns\textstyle
	M_0(f_s)
		\leq M_0(f),\quad
	M_1(f_s)
		\leq \sqrt{\frac{2}{\pi}} \frac{M_0(f)}{s},\qtext{and}
	M_2(f_s)
		\leq \sqrt{2}\frac{M_0(f)}{s^2} .
\ealigns
If, additionally, $f \in C^1$, then $\grad f_s(x) = \Earg{\grad f(x+s G)}$.
\end{lemma}
Indeed, for each non-zero $w\in\reals^d$, we may apply \lemref{convolution-smoothing} to the function $f_{w}(x) \defeq \inner{w}{g_{h}(x)}/\twonorm{w}$ with convolution
$f_{w,s}(x) = \inner{w}{g_{h,s}(x)}/\twonorm{w}$
to obtain the bounds
\balignst
	M_0(g_{h,s}) 
		&= \sup_{w\neq 0} M_0(f_{w,s}) 
		\leq \sup_{w\neq 0} M_0(f_w)
		= M_0(g_h) \leq \sr, \\
    M_1(g_{h,s}) 
		&= \sup_{w\neq 0} M_1(f_{w,s}) 
		\leq \sup_{w\neq 0} \sqrt{\textfrac{2}{\pi}} \textfrac{M_1(f_w)}{s}
		= \sqrt{\textfrac{2}{\pi}} \textfrac{M_1(f_w)}{s} \leq \sqrt{\textfrac{2}{\pi}} \textfrac{\sr}{s}, %
\ealignst
\balignst
	\text{and } M_2(g_{h,s}) 
		= \sup_{w \neq 0} M_2(f_{w,s}) 
		\leq \sup_{w \neq 0} \textfrac{\sqrt{2}M_2(f_w)}{s^2}
		= \textfrac{\sqrt{2}M_2(f_w)}{s^2} \leq \textfrac{\sqrt{2}\,\sr}{s^2}.
\ealignst
Hence, since our choice of $h$ was arbitrary, and 
\[
\textstyle
\kappa_s
	\defeq \maxarg{1, \frac{1}{s} \sqrt{\frac{2}{\pi}}, \frac{\sqrt{2}}{s^2}} 
	= \maxarg{1, \frac{\sqrt{2}}{s^2}} 
	\geq \frac{\max(M_0(g_{h,s}), M_1(g_{h,s}), M_2(g_{h,s}))}{\sr},
\]
we may take expectation under $Q_n$ and supremum over $h$ in \eqnref{constant-smooth-stein-bound} to reach 
\baligns
	\twowass(\mu, \nu) 
		&\leq \inf_{s>0}\diffstein{Q_n}{\steinsetarg{2}} \sr\kappa_s + s \Earg{\twonorm{G}}(1 + 2M_1(b) \sr) \\
		&\leq \maxarg{\diffstein{Q_n}{\steinsetarg{2}} \sr, \eta} + 2\eta
		\leq 3\maxarg{\diffstein{Q_n}{\steinsetarg{2}} \sr, \eta},
\ealigns
where we define $\eta = \sqrt[3]{\diffstein{Q_n}{\steinsetarg{2}}\sqrt{2}\sr\,\Earg{\twonorm{G}}^{2}(1+2M_1(b)\sr)^2}$ and select $s = \sqrt[3]{\diffstein{Q_n}{\steinsetarg{2}} 2\sqrt{2}\sr/(\Earg{\twonorm{G}}(1 + 2M_1(b) \sr))}$ to produce the second inequality.
The generic norm result now follows from 
the assumed norm domination property $\norm{\cdot} \geq \twonorm{\cdot}$, which implies $\steinsetarg{2} \subseteq \steinset$.

\subsection{Proof of \lemreflow{convolution-smoothing}: Smoothing by Gaussian convolution}
\label{sec:convolution-smoothing}

The conclusion $M_0(f_s) \leq M_0(f)$ follows from \Holder's inequality.
Now, fix any $x$ and non-zero $v_1,v_2 \in \reals^d$.
Since $f_s = f\star \phi_s $, where $\phi_{s} \in C^\infty$ is the density of $sG$ and $\star$ is the convolution operator, Leibniz's rule implies that
\baligns
\inner{v_1}{\grad f_s(x)} 
	&\textstyle= \inner{v_1}{(f \star \grad \phi_s)(x)} 
	= \frac{1}{s^2} \int f(x-y) \inner{v_1}{y }\phi_s(y) dy \\
	&\textstyle\leq \frac{M_0(f)}{s^2}\int {|\inner{v_1}{y}|}\phi_s(y)\dy 
	= \sqrt{\frac{2}{\pi}}\frac{M_0(f)}{s}\twonorm{v_1},
\ealigns
as ${\inner{v_1}{G}}/\twonorm{v_1}$ has a standard normal distribution.  Leibniz's rule also gives
\baligns
\Hess f_s(x)[v_1,v_2] 
	&\textstyle= (f \star \Hess \phi_s)(x)[v_1,v_2] \\
	&\textstyle\leq \frac{M_0(f)}{s^2}\int_{\reals^d} \left|\inner{v_1}{zz^\top v_2}/s^2 - \inner{v_1}{v_2}\right| \phi_s(z)\,dz\\
					  &\textstyle\leq \frac{M_0(f)}{s^2}\sqrt{\int_{\reals^d} \left|\inner{v_1}{zz^\top v_2}/s^2 - \inner{v_1}{v_2}\right| ^2 \phi_s(z)\,dz}\\
					  &\textstyle= \frac{M_0(f)}{s^2}\sqrt{\inner{v_1}{v_2}^2 + \twonorm{v_1}^2\twonorm{v_2}^2} \leq \frac{\sqrt{2}M_0(f)}{s^2}\twonorm{v_1}\twonorm{v_2},
\ealigns
where the last equality follows by Isserlis' theorem.  
Finally, when $f \in C^1$, Leibniz's rule gives $\grad f_s = \grad f \star \phi_s$.
\section{Proof of \thmreflow{nonconstant-lower-bound}}
\label{sec:nonconstant-lower-bound-proof}
We will derive each inequality for $\norm{\cdot} = \twonorm{\cdot}$;
the generic norm results will then follow from 
the property $\norm{\cdot} \geq \twonorm{\cdot}$, which implies $\steinsetarg{2} \subseteq \steinset$.

Fix any $h \in \hset = \{ h : \reals^d\to \reals \mid h \in C^3, M_1(h) \leq 1, M_2(h) < \infty, M_3(h) < \infty \}$ with $\Esubarg{P}{h(\PVAR)} = 0$.  
We assume that $M_1(b),$ $M_2(b)$, $M_1(\sigma)$, $F_2(\sigma)$, $M_1^*(m),$ and $M_0(\sigma^{-1})$ are all finite, or else the results are vacuous. 
Our high-level strategy is to relate the Wasserstein distance to the Stein discrepancy via the Stein equation \eqnref{stein-equation} with diffusion Stein operator $\diffusion{}$ \eqnref{diffusion-operator}.
By \thmref{stein-factors}, we know that there exists a Lipschitz solution 
$g_h$ to the the Stein equation 
$h(x) = \diffarg{g_h}{x}$
satisfying $M_0(g_h) \leq \sr M_1(h) \leq \sr$
and $M_1(g_h) \leq \beta M_1(h) \leq \beta$, for $\beta \defeq \beta_1+\beta_2$, where $\beta_1$ and $\beta_2$ are defined in \thmref{stein-factors}.
Since a Lipschitz gradient is also needed to declare that $g_h$ falls into a scaled copy of the classical Stein set $\steinset$, we will develop a smoothed version of the Stein solution with greater regularity.

For this purpose, fix any $s > 0$ and consider the convolution $g_{h,s}(x) \defeq \Earg{g_{h}(x+sG)}$.  
If the smoothing level $s$ is small, the Lipschitz continuity of $m$ and $h$ implies that $\diffarg{g_{h,s}}{x}$ closely approximates $h(x)$ for each $x\in\reals^d$:
\balign\label{eqn:nonconstant-smooth-stein-bound}
&h(x) 
	\leq \Earg{h(x+sG)} + M_1(h)s\Earg{\twonorm{G}}\\ \notag
	\leq\  &2\Earg{\inner{b(x+sG)}{g_h(x+sG)} + \inner{m(x+sG)}{\grad g_h(x+sG)}}+ s\Earg{\twonorm{G}} \\ \notag
	\leq\ &\diffarg{g_{h,s}}{x} + s \zeta.
\ealign

\subsection{Proof of the first inequality}\label{sec:first-inequality-proof}
Moreover, by an argument mirroring that of \thmref{constant-lower-bound}, \lemref{convolution-smoothing} shows that $g_{h,s}$ admits a Lipschitz gradient $\grad g_{h,s}(x) = \Earg{\grad g_{h}(x+sG)}$ and satisfies the derivative bounds
\balign\label{eqn:smooth-stein-factors}
	M_0(g_{h,s}) 
		&\leq M_0(g_h) \leq \sr, \\ \notag
	M_1(g_{h,s}) 
		&= M_0(\grad g_{h,s}) \leq M_0(\grad g_h) \leq \beta,\qtext{and} \\ \notag
	M_2(g_{h,s}) 
		&= M_1(\grad g_{h,s}) 
		\leq \sqrt{\textfrac{2}{\pi}}\textfrac{M_0(\grad g_h)}{s}
		\leq \sqrt{\textfrac{2}{\pi}}\textfrac{\beta}{s}.
\ealign
Let $\eta \defeq s^*\zeta$ for 
$
s^* = \sqrt{\diffstein{Q_n}{\steinsetarg{2}}\sqrt{{2}{/\pi}}\beta/\zeta}.
$
Since $\hset$ is dense in $\twowassset$,  
we may take expectation under $Q_n$ and supremum over $h$ in \eqnref{nonconstant-smooth-stein-bound} to reach 
\baligns%
	\twowass(\mu, \nu) 
		&\leq \inf_{s>0}\diffstein{Q_n}{\steinsetarg{2}} \textstyle\maxarg{\sr, \beta, \sqrt{\textfrac{2}{\pi}}\frac{\beta}{s}} 
		+ s \zeta \\
		&\leq \max(\diffstein{Q_n}{\steinsetarg{2}} \max(\sr,\beta), \eta) + \eta \\
		&\leq 2\max(\diffstein{Q_n}{\steinsetarg{2}}  \max(\sr,\beta), \eta).
\ealigns
\subsection{Proof of the second inequality}
Assume now that $\grad^3 b$ and $\grad^3 \sigma$ are bounded and locally Lipschitz.
Fix any $\iota \in (0,1)$. 
\lemref{convolution-smoothing} and an auxiliary smoothing lemma (\lemref{standardsmoothing} in the supplement) imply that
$M_{2}(g_{h,s}) = M_1(\grad g_{h,s}) \leq \sqrt{d}\frac{M_{1-\iota}(\grad g_{h})}{s^{\iota}}$. 
This improved dependence on $s$ will allow us to establish a near-linear relationship between the Stein discrepancy and the Wasserstein distance.
By %
\thmref{stein-factors},
$M_{1-\iota}(\grad g_h) \leq \frac{1}{K}(\frac{1}{\iota}+s_{r})$ for $K$ depending only on $M_{1:3}(\sigma),M_{1:3}(b),$ $M_0(\sigma^{-1}),$ and $r$. 
Hence, $M_{2}(g_{h,s}) \leq C_\iota/s^{\iota}$ for $C_{\iota} \defeq \frac{\sqrt{d}}{K}(\frac{1}{\iota}+s_{r})$.
Following the derivation in \secref{first-inequality-proof} and choosing $s^*=\left(\frac{\iota C_{\iota}\diffstein{Q_n}{\steinsetarg{2}}}{\zeta}\right){}^{\frac{1}{\iota+1}}$ and $\eta \defeq \frac{\zeta}{\iota} s^*$, we obtain
\balignt \label{eqn:iota-lower-bound}
\twowass(P,Q_n)
	\leq&\inf_{s>0}\diffstein{Q_n}{\steinsetarg{2}}\max(s_{r},\beta,C_{\iota}s^{-\iota})+s\zeta \\ \nonumber
	\leq&\max(\diffstein{Q_n}{\steinsetarg{2}}\max\left(s_{r},\beta\right),\eta) +\eta\iota\\\nonumber
	 \leq& 2\max(\diffstein{Q_n}{\steinsetarg{2}}\max\left(s_{r},\beta\right),\eta).
\ealignt

Now consider the case in which $\diffstein{Q_n}{\steinset} < e^{-1}$ and the choice $\iota =
1/\log(\nicefrac{1}{\diffstein{Q_n}{\steinset}}) \in (0,1)$. Since $x^{1/(\log x - 1)} \le e$ for all $x\in (0,e^{-1})$, 
\baligns
\textfrac{1}{\iota}\diffstein{Q_n}{\steinset}^{\frac{1}{1+\iota}}
  &= \log (\nicefrac{1}{\diffstein{Q_n}{\steinset}})
  \diffstein{Q_n}{\steinset}^{1 + \nicefrac{1}{(\log \diffstein{Q_n}{\steinset} - 1)}} \\
  &\le e \log (\nicefrac{1}{\diffstein{Q_n}{\steinset}}) \diffstein{Q_n}{\steinset}.
\ealigns
Introduce the shorthand $c_0 = \frac{\sqrt{d}}{K\zeta}$.
Since $\nicefrac{1}{1+\iota}\in(\nicefrac{1}{2}, 1)$, we have $c_0^{\frac{1}{1+\iota}} \le \max(\sqrt{c_0},
c_0)$. Similarly, $1 + s_r\iota > 1$,
so $(1 + \iota s_r)^{\frac{1}{1+\iota}} \le 1 + \iota s_r$.
Therefore,
\baligns
&\textfrac{\zeta}{\iota} \diffstein{Q_n}{\steinset}^{\textfrac{1}{1+\iota}}(\textfrac{1+
  \iota s_r}{K \zeta/\sqrt{d}})^{\textfrac{1}{1+\iota}} \\
&\le e\zeta \diffstein{Q_n}{\steinset} \log(\nicefrac{1}{\diffstein{Q_n}{\steinset}})
\max \left(\textfrac{d^{1/4}}{\sqrt{K\zeta}}, \textfrac{\sqrt{d}}{K\zeta}\right)
\left(1 + \textfrac{s_r}{\log(\nicefrac{1}{\diffstein{Q_n}{\steinset}})}\right) \\
&= e\, \diffstein{Q_n}{\steinset} \max\left(\textfrac{d^{1/4}\sqrt{\zeta}}{\sqrt{K}}, \textfrac{\sqrt{d}}{K}\right)
(s_r + \log(\nicefrac{1}{\diffstein{Q_n}{\steinset}})).
\ealigns

Next, fix any $\iota\in(0,1)$ and consider the case in which $\diffstein{Q_n}{\steinset} \ge e^{-1}$ so that $\diffstein{Q_n}{\steinset}^{\textfrac{1}{1+\iota}} \le
\diffstein{Q_n}{\steinset} e^{\textfrac{\iota}{\iota+1}}$. Because $\textfrac{1}{\iota}e^{\textfrac{\iota}{\iota+1}}\le \half
e^{1/2} < e$ and $(1+\iota s_r)^{\nicefrac{1}{1+\iota}} \le 1 + s_r$,
we conclude that
\baligns
\textfrac{\zeta}{\iota} \diffstein{Q_n}{\steinset}^{\textfrac{1}{1+\iota}}(\textfrac{1+
  \iota s_r}{K \zeta/\sqrt{d}})^{\textfrac{1}{1+\iota}}
\le e\, \diffstein{Q_n}{\steinset}
\max\left(\textfrac{d^{1/4}\sqrt{\zeta}}{\sqrt{K}},
\textfrac{\sqrt{d}}{K}\right) (s_r + 1).
\ealigns
The result follows from estimates of these two cases and the bound \eqref{eqn:iota-lower-bound}.

\section{Proof of \propreflow{discrepancy-upper-bound}} \label{sec:discrepancy-upper-bound-proof}
Fix any $g \in \steinset$.
Since $\Esubarg{P}{\diffarg{g}{\PVAR}} = 0$ by \propref{diffusion-zero}, we may write
\balign\label{eqn:stein-upper-bound-proof-expression}
|\Esubarg{Q_n}{\diffarg{g}{X}}|
	&= |\Esubarg{Q_n}{\diffarg{g}{X}} - \Esubarg{P}{\diffarg{g}{\PVAR}}| \notag \\
	&= |2\Earg{\inner{b(X)-b(\PVAR)}{g(X)}+ \inner{b(\PVAR)}{g(X)-g(\PVAR)}} \notag \\
	&\quad + \Earg{\inner{m(X)-m(\PVAR)}{\grad g(X)}+\inner{m(\PVAR)}{\grad
            g(X)-\grad g(\PVAR)}}|.
\ealign
for any coupling of $X$ and $Z$.
We obtain the first advertised inequality by repeatedly applying the Fenchel-Young inequality for dual norms,
invoking the boundedness and Lipschitz constraints on $g$ and $\grad g$, and taking a supremum over $g \in \steinset$.
The second inequality follows from the firstby invoking Jensen's inequality, the fact $\min(x,y)\le
x^ty^{1-t}$ for all $x,y\geq 0$, \Holder's inequality, and finally the
definition of $\lswass{s}$.

We prove the final claim by bounding the first advertised inequality in a second manner.
Let $(X,Z)$ be coupled so that $c\defeq \min(\lswass{1}(Q_n,P), 2) = \min(\Earg{\norm{X-\PVAR}}, 2)$, $A = 2\norm{b(\PVAR)} + \norm{m(\PVAR)}$, and $B = \min(\norm{X-\PVAR}, 2)$. The Fenchel-Young inequality ($xy \leq e^x - y + y\log y$ for $y\geq 0, x\in\reals$), the concavity of $x\mapsto \min(x,2)$, and Jensen's inequality now yield the result as
\balignst
&\Earg{AB}
	= \Earg{(A-\log(\mu_0/c))B} + \Earg{B}\log(\mu_0/c) \\
    &\leq \Earg{e^{A-\log(\mu_0/c)} - B + B \log(B)} + \Earg{B}\log(\mu_0/c) \\
    &= c - \Earg{B\log(e/B)} + \Earg{B}\log(\mu_0/c)
    \leq c + c\log(\mu_0/c)
    = c\log(e\mu_0/c).
\ealignst

\section{Proof of \thmreflow{concave-decay}}
\label{sec:concave-decay-proof}
Fix any $x,y\in\reals^d$, and define two \Ito diffusions solving
$d\process{t}{x} = b(\process{t}{x})\dt + \sigma(\process{t}{x})\,dW_t$ with $\process{0}{x}=x$ and
$d\process{t}{y} = b(\process{t}{y})\dt + \sigma(\process{t}{y})\,dW_t$ with $\process{0}{y}=y$,
for $(W_t)_{t\geq0}$ a shared Wiener process.
Applying Dynkin's formula to the function $f(t,x) = e^{kt}\norm{x}_G^2$ for the difference process $\process{t}{x}-\process{t}{y}$ yields
\balignst
&\textstyle\Earg{f(t,\process{t}{x}-\process{t}{y})}
  = \norm{x-y}_G^2 + \E[{\int_0^t ke^{ks} \norm{\process{s}{x} -
      \process{s}{y}}_G^2 \ds}]\\
&\textstyle+ \E[{\int_0^te^{ks}\left(\norm{\sigma(\process{s}{x})
  - \sigma(\process{s}{y})}_G^2 + 2\inner{b(\process{s}{x}) -
  b(\process{s}{y})}{G(\process{s}{x} - \process{s}{y})}\right)\ds}]
\ealignst
By the uniform dissipativity assumption, the right-hand side is at most 
$\norm{x-y}_G^2 = \wassarg{\norm{\cdot}_G}(\delta_x,\,\delta_y)^2$.
For the transition semigroup $\fulltrans{}$,
\[
\Earg{f(t,\process{t}{x}-\process{t}{y})} =
e^{kt}\Earg{\norm{\process{t}{x}-\process{t}{y}}_G^2} 
\geq e^{kt}\wassarg{\norm{\cdot}_G}(\delta_x \trans{t}{},\,\delta_y \trans{t}{})^2,
\]
by Cauchy-Schwarz.
The result now follows from the fact that
$\mineig{G_1}\le \norm{z}_G^2/\twonorm{z}^2\le \maxeig{G_1}$ for all
$z\neq 0$.

\section{Proof of \thmreflow{nonconstant-decay}}\label{sec:nonconstant-decay-proof}
As in the proof of \citep[Thm. 2.6]{Wang2016}, we fix two arbitrary starting points $x,y\in\reals^d$ and define a pair of coupled \Ito diffusions $\fullprocess{x}$ and $\fullprocess{y}$, each with associated marginal semigroup $(\trans{t}{})_{t\geq 0}$.
Specifically, we set $\process{0}{x}= x$ and $\process{0}{y}=y$ and let $\fullprocess{x}$ and $\fullprocess{y}$ solve the equations 
\balignst
d\process{t}{x} 
	&= b(\process{t}{x}) \dt + \sigma_0(\process{t}{x}) \dW_t' + \lambda_0 \dW_t'' \\
d\process{t}{y} 
	&= \textstyle b(\process{t}{y}) \dt + \sigma_0(\process{t}{y}) \dW_t' + \lambda_0 \big(I - 2 \frac{\process{t}{x} - \process{t}{y}}{\twonorm{\process{t}{x} - \process{t}{y}}} \frac{\process{t}{x} - \process{t}{y}^\top}{\twonorm{\process{t}{x} - \process{t}{y}}}\big) \dW_t'',
\ealignst
where $(W_t')_{t\geq 0}$ is an $m$-dimensional Wiener process and $(W_t'')_{t\geq 0}$ is an independent $d$-dimensional Wiener process.

Following the argument of \citet[Sec. 4]{Eberle2015}, we define the difference process $Y_t = \process{t}{x} - \process{t}{y}$, its norm $r_t = \twonorm{Y_t}$,
and the one-dimensional Wiener process $W_t = \int_0^t \inner{Y_s/r_s}{\dW_s''}$, and apply the generalized \Ito formula \cite[Thm. 22.5]{Kallenberg2005Foundations} %
to obtain the  stochastic differential equations
\baligns
d\twonorm{Y_t}^2 
	&= (2\inner{Y_t}{b(\process{t}{x}) - b(\process{t}{y})} + \norm{\sigma_0(\process{t}{x}) - \sigma_0(\process{t}{y})}_F^2 + 4\lambda_0^2) \dt \\
	&+ 2 \inner{Y_t}{( \sigma_0(\process{t}{x}) - \sigma_0(\process{t}{y}) ) \dW_t'} + 4 \lambda_0 \twonorm{Y_t} \dW_t\qtext{and}\\
df(r_t) 
	&= f'(r_t)/(r_t) \inner{Y_t}{( \sigma_0(\process{t}{x}) - \sigma_0(\process{t}{y}) ) \dW_t'} + 2 \lambda_0 f'(r_t) \dW_t \\
	+ &(f''(r_t)(2\lambda_0^2 + \texthalf \twonorm{ (\sigma_0(\process{t}{x}) - \sigma_0(\process{t}{y}))^\top Y_t}^2/r_t^2)  - \textfrac{1}{2\alpha} f'(r_t) \kappa(r_t) r_t)\dt 
\ealigns
for any concave increasing $f : [0,\infty) \mapsto [0,\infty)$ with absolutely continuous derivative, $f(0) = 0$, and $f'(0) = 1$.  
Since the drift term in the latter equation is bounded above by
$
\beta_t \defeq (2/\alpha) (  f''(r_t) - (1/4) f'(r_t) \kappa(r_t) r_t ),
$
the argument of \citep[p. 15]{Eberle2015} shows that the results of \citep[Thm. 1 and Cor. 2]{Eberle2015} hold for our choice of $\alpha$ and $\kappa$.

\section*{Acknowledgments}
We thank Simon Lacoste-Julien for sharing his quadrature code, Martin Hairer for discussing interpolation inequalities, 
Andreas Eberle for reading an earlier version of this manuscript, and Murat Erdogdu for identifying an important typographical error
in an earlier version of this manuscript.
This material is based upon work supported by the grant EPSRC EP/N000188/1, the National Science Foundation DMS RTG Grant No. 1501767, the National Science Foundation Graduate Research Fellowship under Grant No. DGE-114747, the Frederick E. Terman Fellowship, and the Lloyd's Register Foundation programme on Data Centric engineering at the Alan Turing Institute, UK.
\bibliography{stein}

\begin{thebibliography}{95}
\providecommand{\natexlab}[1]{#1}
\providecommand{\url}[1]{\texttt{#1}}
\expandafter\ifx\csname urlstyle\endcsname\relax
  \providecommand{\doi}[1]{doi: #1}\else
  \providecommand{\doi}{doi: \begingroup \urlstyle{rm}\Url}\fi

\bibitem[Ambrosio et~al.(2000)Ambrosio, Fusco, and
  Pallara]{ambrosio2000functions}
L.~Ambrosio, N.~Fusco, and D.~Pallara.
\newblock \emph{Functions of bounded variation and free discontinuity
  problems}.
\newblock Oxford University Press, 2000.

\bibitem[Bach et~al.(2012)Bach, Lacoste-Julien, and Obozinski]{Bach2012}
F.~Bach, S.~Lacoste-Julien, and G.~Obozinski.
\newblock On the equivalence between herding and conditional gradient
  algorithms.
\newblock In \emph{Proc. 29th ICML}, ICML'12, 2012.

\bibitem[Barbour(1988)]{Barbour88}
A.~D. Barbour.
\newblock Stein's method and {P}oisson process convergence.
\newblock \emph{J. Appl. Probab.}, \penalty0 (Special Vol. 25A):\penalty0
  175--184, 1988.
\newblock ISSN 0021-9002.
\newblock A celebration of applied probability.

\bibitem[Barbour(1990)]{Barbour90}
A.~D. Barbour.
\newblock Stein's method for diffusion approximations.
\newblock \emph{Probab. Theory Related Fields}, 84\penalty0 (3):\penalty0
  297--322, 1990.
\newblock ISSN 0178-8051.

\bibitem[Bouts et~al.(2014)Bouts, ten Brink, and Buchin]{BoutsteBu14}
Q.~W. Bouts, A.~P. ten Brink, and K.~Buchin.
\newblock A framework for {C}omputing the {G}reedy {S}panner.
\newblock In \emph{Proc. of 30th SOCG}, pages 11:11--11:19, New York, NY, 2014.
  ACM.

\bibitem[Brooks et~al.(2011)Brooks, Gelman, Jones, and Meng]{BrooksGeJoMe11}
S.~Brooks, A.~Gelman, G.~Jones, and X.-L. Meng.
\newblock \emph{Handbook of {M}arkov chain {M}onte {C}arlo}.
\newblock CRC press, 2011.

\bibitem[Cattiaux and Guillin(2014)]{CattiauxGu14}
P.~Cattiaux and A.~Guillin.
\newblock Semi log-concave {M}arkov diffusions.
\newblock In \emph{S\'eminaire de {P}robabilit\'es {XLVI}}, volume 2123 of
  \emph{Lecture Notes in Math.}, pages 231--292. Springer, Cham, 2014.

\bibitem[Cerrai(2001)]{Cerrai2001}
S.~Cerrai.
\newblock \emph{Second order {PDE}'s in finite and infinite dimension: a
  probabilistic approach}, volume 1762.
\newblock Springer Science \& Business Media, 2001.

\bibitem[Chatterjee and Meckes(2008)]{ChatterjeeMe08}
S.~Chatterjee and E.~Meckes.
\newblock Multivariate normal approximation using exchangeable pairs.
\newblock \emph{ALEA Lat. Am. J. Probab. Math. Stat.}, 4:\penalty0 257--283,
  2008.
\newblock ISSN 1980-0436.

\bibitem[Chatterjee and Shao(2011)]{ChatterjeeSh11}
S.~Chatterjee and Q.~Shao.
\newblock Nonnormal approximation by {S}tein's method of exchangeable pairs
  with application to the {C}urie-{W}eiss model.
\newblock \emph{Ann. Appl. Probab.}, 21\penalty0 (2):\penalty0 464--483, 2011.
\newblock ISSN 1050-5164.

\bibitem[Chen et~al.(2011)Chen, Goldstein, and Shao]{ChenGoSh11}
L.~Chen, L.~Goldstein, and Q.~Shao.
\newblock \emph{Normal approximation by {S}tein's method}.
\newblock Probability and its Applications. Springer, Heidelberg, 2011.
\newblock ISBN 978-3-642-15006-7.

\bibitem[Chen et~al.(2018)Chen, Mackey, Gorham, Briol, and
  Oates]{ChenMaGoBrOa2018}
W.~Y. Chen, L.~Mackey, J.~Gorham, F.-X. Briol, and C.~Oates.
\newblock Stein points.
\newblock In \emph{Proc. 35th ICML}, ICML'18, 2018.

\bibitem[Chen et~al.(2010)Chen, Welling, and Smola]{Chen2010}
Y.~Chen, M.~Welling, and A.~Smola.
\newblock Super-samples from kernel herding.
\newblock In \emph{UAI}, 2010.

\bibitem[Chew(1986)]{Chew86}
P.~Chew.
\newblock There is a {P}lanar {G}raph {A}lmost {A}s {G}ood {A}s the {C}omplete
  {G}raph.
\newblock In \emph{Proc. 2nd SOCG}, pages 169--177, New York, NY, 1986. ACM.

\bibitem[Chwialkowski et~al.(2016)Chwialkowski, Strathmann, and
  Gretton]{ChwialkowskiStGr2016}
K.~Chwialkowski, H.~Strathmann, and A.~Gretton.
\newblock A kernel test of goodness of fit.
\newblock In \emph{Proc. 33rd ICML}, ICML, 2016.

\bibitem[Conca and Vanninathan(2007)]{conca2007periodic}
C.~Conca and M.~Vanninathan.
\newblock Periodic homogenization problems in incompressible fluid equations.
\newblock \emph{Handbook of Mathematical Fluid Dynamics}, 4:\penalty0 649--698,
  2007.

\bibitem[{Doersek} and {Teichmann}(2010)]{Doersek2010Semigroup}
P.~{Doersek} and J.~{Teichmann}.
\newblock {A Semigroup Point Of View On Splitting Schemes For Stochastic
  (Partial) Differential Equations}.
\newblock \emph{ArXiv e-prints}, Nov. 2010.

\bibitem[Drummond and Gardiner(1980)]{drummond1980}
P.~D. Drummond and C.~W. Gardiner.
\newblock {Generalised P-representations in quantum optics}.
\newblock \emph{J. Phys. A}, 13\penalty0 (7):\penalty0 2353, 1980.

\bibitem[Drummond and Walls(1980)]{walls1980}
P.~D. Drummond and D.~F. Walls.
\newblock {Quantum theory of optical bistability. I. Nonlinear polarisability
  model}.
\newblock \emph{J. Phys. A}, 13\penalty0 (2):\penalty0 725, 1980.

\bibitem[Duncan et~al.(2016)Duncan, Lelievre, and
  Pavliotis]{duncan2016variance}
A.~B. Duncan, T.~Lelievre, and G.~Pavliotis.
\newblock Variance reduction using nonreversible {L}angevin samplers.
\newblock \emph{J. Stat. Phys.}, 163\penalty0 (3):\penalty0 457--491, 2016.

\bibitem[Dynkin(1965)]{Dynkin1965}
E.~Dynkin.
\newblock \emph{Markov {P}rocesses : Volume 1}.
\newblock Springer Berlin Heidelberg, Berlin, Heidelberg, 1965.
\newblock ISBN 978-3-662-00033-5.

\bibitem[Eberle(2015)]{Eberle2015}
A.~Eberle.
\newblock Reflection couplings and contraction rates for diffusions.
\newblock \emph{Probab. Theory Related Fields}, pages 1--36, 2015.

\bibitem[Engel and Nagel(2000)]{Engel2000Semigroup}
K.~Engel and R.~Nagel.
\newblock \emph{One-parameter semigroups for linear evolution equations},
  volume 194 of \emph{Graduate Texts in Mathematics}.
\newblock Springer-Verlag, New York, 2000.
\newblock ISBN 0-387-98463-1.

\bibitem[Ethier and Kurtz(1986)]{EthierKu86}
S.~N. Ethier and T.~G. Kurtz.
\newblock \emph{Markov processes}.
\newblock Wiley Series in Probability and Mathematical Statistics: Probability
  and Mathematical Statistics. John Wiley \& Sons, Inc., New York, 1986.
\newblock ISBN 0-471-08186-8.

\bibitem[Fan et~al.(2006)Fan, Brooks, and Gelman]{FanBrGe06}
Y.~Fan, S.~P. Brooks, and A.~Gelman.
\newblock Output assessment for {M}onte {C}arlo simulations via the score
  statistic.
\newblock \emph{J. Comp. Graph. Stat.}, 15\penalty0 (1), 2006.

\bibitem[Fang et~al.(2018)Fang, Shao, and Xu]{FangShXu2018}
X.~Fang, Q.-M. Shao, and L.~Xu.
\newblock Multivariate approximations in wasserstein distance by stein’s
  method and bismut’s formula.
\newblock \emph{Probability Theory and Related Fields}, pages 1--35, 2018.

\bibitem[Fourni{\'e} et~al.(1999)Fourni{\'e}, Lasry, Lebuchoux, Lions, and
  Touzi]{FournieLaLeLiTo1999}
E.~Fourni{\'e}, J.~Lasry, J.~Lebuchoux, P.~Lions, and N.~Touzi.
\newblock Applications of {M}alliavin calculus to {M}onte {C}arlo methods in
  finance.
\newblock \emph{Fin. Stochastics}, 3\penalty0 (4):\penalty0 391--412, 1999.
\newblock ISSN 0949-2984.

\bibitem[Friedman(1975)]{Friedman1975Stoch}
A.~Friedman.
\newblock \emph{Stochastic differential equations and applications. {V}ol. 1}.
\newblock Academic Press [Harcourt Brace Jovanovich, Publishers], New
  York-London, 1975.
\newblock Probability and Mathematical Statistics, Vol. 28.

\bibitem[{Gan} et~al.(2017){Gan}, {R{\"o}llin}, and {Ross}]{GanRoRo2016}
H.~L. {Gan}, A.~{R{\"o}llin}, and N.~{Ross}.
\newblock Dirichlet approximation of equilibrium distributions in {C}annings
  models with mutation.
\newblock \emph{Advances in Applied Probability}, 49\penalty0 (3):\penalty0
  927--959, 2017.

\bibitem[Gaunt(2016)]{Gaunt2016}
R.~E. Gaunt.
\newblock Rates of convergence in normal approximation under moment conditions
  via new bounds on solutions of the stein equation.
\newblock \emph{J. Theoret. Probab.}, 29\penalty0 (1):\penalty0 231--247, 2016.

\bibitem[Gelman(2012)]{Gelman2012multilevel}
A.~Gelman.
\newblock Multilevel (hierarchical) modeling: what it can and cannot do.
\newblock \emph{Technometrics}, 2012.

\bibitem[Gelman et~al.(2014)Gelman, Carlin, Stern, Dunson, Vehtari, and
  Rubin]{GelmanCaStDuVeRu2014}
A.~Gelman, J.~Carlin, H.~Stern, D.~Dunson, A.~Vehtari, and D.~Rubin.
\newblock \emph{Bayesian data analysis}.
\newblock Texts in Statistical Science Series. CRC Press, Boca Raton, FL, third
  edition, 2014.
\newblock ISBN 978-1-4398-4095-5.

\bibitem[Girolami and Calderhead(2011)]{girolami2011riemann}
M.~Girolami and B.~Calderhead.
\newblock Riemann {M}anifold {L}angevin and {H}amiltonian {M}onte {C}arlo
  methods.
\newblock \emph{J. R. Stat. Soc. Ser. B}, 73\penalty0 (2):\penalty0 123--214,
  2011.

\bibitem[Glaeser(1958)]{Glaeser58}
G.~Glaeser.
\newblock \'{E}tude de quelques alg\`ebres tayloriennes.
\newblock \emph{J. Analyse Math.}, 6:\penalty0 1--124; erratum, insert to 6
  (1958), no. 2, 1958.

\bibitem[Gorham and Mackey(2015)]{GorhamMa15}
J.~Gorham and L.~Mackey.
\newblock Measuring sample quality with {S}tein's method.
\newblock In \emph{Adv. NIPS 28}, pages 226--234. 2015.

\bibitem[Gorham and Mackey(2017)]{GorhamMa17}
J.~Gorham and L.~Mackey.
\newblock Measuring sample quality with kernels.
\newblock In \emph{Proc. of 34st ICML}, ICML'17, 2017.

\bibitem[G{\"o}tze(1991)]{Gotze91}
F.~G{\"o}tze.
\newblock On the rate of convergence in the multivariate {CLT}.
\newblock \emph{Ann. Probab.}, 19\penalty0 (2):\penalty0 724--739, 1991.

\bibitem[Gretton et~al.(2006)Gretton, Borgwardt, Rasch, Sch{\"o}lkopf, and
  Smola]{GrettonBoRaScSm06}
A.~Gretton, K.~Borgwardt, M.~Rasch, B.~Sch{\"o}lkopf, and A.~Smola.
\newblock A kernel method for the two-sample-problem.
\newblock In \emph{Adv. NIPS 19}, pages 513--520, 2006.

\bibitem[Gudmundsson et~al.(2007)Gudmundsson, Klein, Knauer, and
  Smid]{gudmundsson2007small}
J.~Gudmundsson, O.~Klein, C.~Knauer, and M.~Smid.
\newblock {Small Manhattan Networks and Algorithmic Applications for the Earth
  Movers Distance}.
\newblock In \emph{Proc. 23rd EuroCG}, pages 174--177, 2007.

\bibitem[Hairer et~al.(2014)Hairer, Stuart, and Vollmer]{hairer2011spectral}
M.~Hairer, A.~Stuart, and S.~Vollmer.
\newblock Spectral gaps for a {M}etropolis--{H}astings algorithm in infinite
  dimensions.
\newblock \emph{Ann. Appl. Probab.}, 2014.

\bibitem[Har-Peled and Mendel(2006)]{Har-PeledMe06}
S.~Har-Peled and M.~Mendel.
\newblock Fast construction of nets in low-dimensional metrics and their
  applications.
\newblock \emph{SIAM J. Comput.}, 35\penalty0 (5):\penalty0 1148--1184, 2006.

\bibitem[Hartley and Zisserman(2004)]{HartleyZi2004}
R.~Hartley and A.~Zisserman.
\newblock \emph{Multiple View Geometry in Computer Vision}.
\newblock Cambridge University Press, ISBN: 0521540518, second edition, 2004.

\bibitem[Horowitz(1987)]{horowitz1987second}
A.~Horowitz.
\newblock The second order {L}angevin equation and numerical simulations.
\newblock \emph{Nucl. Phys. B}, 280:\penalty0 510--522, 1987.

\bibitem[Huber and Ronchetti(2009)]{HuberRo2009}
P.~Huber and E.~Ronchetti.
\newblock \emph{Robust statistics}.
\newblock Wiley Series in Probability and Statistics. John Wiley \& Sons, Inc.,
  Hoboken, NJ, second edition, 2009.
\newblock ISBN 978-0-470-12990-6.

\bibitem[Huggins and Zou(2017)]{HugginsZo2017}
J.~Huggins and J.~Zou.
\newblock {Quantifying the accuracy of approximate diffusions and Markov
  chains}.
\newblock In \emph{Proc. 20th AISTATS}, pages 382--391, 2017.

\bibitem[Huggins and Mackey(2018)]{HugginsMa2018}
J.~H. Huggins and L.~Mackey.
\newblock Random feature {S}tein discrepancies.
\newblock In \emph{Adv. NIPS 31}, 2018.

\bibitem[Hwang et~al.(1993)Hwang, Hwang-Ma, and Sheu]{hwang1993accelerating}
C.~Hwang, S.~Hwang-Ma, and S.~Sheu.
\newblock Accelerating {G}aussian diffusions.
\newblock \emph{Ann. Appl. Probab.}, pages 897--913, 1993.

\bibitem[Joulin and Ollivier(2010)]{joulin2010}
A.~Joulin and Y.~Ollivier.
\newblock Curvature, concentration and error estimates for {M}arkov chain
  {M}onte {C}arlo.
\newblock \emph{Ann. Probab.}, 38\penalty0 (6):\penalty0 2418--2442, 11 2010.

\bibitem[Kallenberg(2002)]{Kallenberg2005Foundations}
O.~Kallenberg.
\newblock \emph{Foundations of modern probability}.
\newblock Probability and its Applications. Springer-Verlag, New York, second
  edition, 2002.
\newblock ISBN 0-387-95313-2.

\bibitem[Kent(1978)]{kent1978time}
J.~Kent.
\newblock Time-reversible diffusions.
\newblock \emph{Ann. Appl. Probab.}, pages 819--835, 1978.

\bibitem[Khasminskii(2012)]{Khasminskii11}
R.~Khasminskii.
\newblock \emph{Stochastic stability of differential equations}, volume~66 of
  \emph{Stochastic Modelling and Applied Probability}.
\newblock Springer, Heidelberg, second edition, 2012.
\newblock ISBN 978-3-642-23279-4.
\newblock With contributions by G. N. Milstein and M. B. Nevelson.

\bibitem[Korattikara et~al.(2014)Korattikara, Chen, and
  Welling]{Korattikara2014}
A.~Korattikara, Y.~Chen, and M.~Welling.
\newblock Austerity in {MCMC} land: Cutting the {M}etropolis-{H}astings budget.
\newblock In \emph{Proc. of 31st ICML}, ICML'14, 2014.

\bibitem[Lacoste-Julien et~al.(2015)Lacoste-Julien, Lindsten, and
  Bach]{Lacoste2015sequential}
S.~Lacoste-Julien, F.~Lindsten, and F.~Bach.
\newblock Sequential kernel herding: Frank-{W}olfe optimization for particle
  filtering.
\newblock In \emph{AISTATS}, 2015.

\bibitem[Landim et~al.(1998)Landim, Olla, and Yau]{landim1998convection}
C.~Landim, S.~Olla, and H.~Yau.
\newblock Convection--diffusion equation with space--time ergodic random flow.
\newblock \emph{Probability theory and related fields}, 112\penalty0
  (2):\penalty0 203--220, 1998.

\bibitem[Ley et~al.(2017)Ley, Reinert, and Swan]{LeyReSw2017}
C.~Ley, G.~Reinert, and Y.~Swan.
\newblock Stein's method for comparison of univariate distributions.
\newblock \emph{Probab. Surveys}, 14:\penalty0 1--52, 2017.

\bibitem[Liu(1996)]{Liu1996}
C.~Liu.
\newblock Bayesian robust multivariate linear regression with incomplete data.
\newblock \emph{JASA}, 91\penalty0 (435):\penalty0 1219--1227, 1996.
\newblock ISSN 01621459.

\bibitem[Liu and Lee(2017)]{LiuLe2017}
Q.~Liu and J.~Lee.
\newblock {Black-box Importance Sampling}.
\newblock In \emph{Proc. 20th AISTATS}, pages 952--961, 2017.

\bibitem[Liu and Wang(2016)]{Liu2016stein}
Q.~Liu and D.~Wang.
\newblock Stein variational gradient descent: A general purpose bayesian
  inference algorithm.
\newblock In \emph{Adv. NIPS 29}, pages 2378--2386, 2016.

\bibitem[Liu et~al.(2016)Liu, Lee, and Jordan]{LiuLeJo16}
Q.~Liu, J.~Lee, and M.~Jordan.
\newblock A kernelized {S}tein discrepancy for goodness-of-fit tests.
\newblock In \emph{Proc. of 33rd ICML}, volume~48 of \emph{ICML}, pages
  276--284, 2016.

\bibitem[Lubin and Dunning(2015)]{LubinDu15}
M.~Lubin and I.~Dunning.
\newblock Computing in operations research using {J}ulia.
\newblock \emph{INFORMS Journal on Computing}, 27\penalty0 (2):\penalty0
  238--248, 2015.

\bibitem[Lunardi(2007)]{lunardiintroduction}
A.~Lunardi.
\newblock An introduction to interpolation theory.
\newblock 2007.

\bibitem[Ma et~al.(2015)Ma, Chen, and Fox]{ma2015complete}
Y.~Ma, T.~Chen, and E.~Fox.
\newblock A complete recipe for stochastic gradient {MCMC}.
\newblock In \emph{Adv. NIPS 28}, pages 2899--2907, 2015.

\bibitem[Mackey and Gorham(2016)]{MackeyGo16}
L.~Mackey and J.~Gorham.
\newblock Multivariate {S}tein factors for a class of strongly log-concave
  distributions.
\newblock \emph{Electron. Commun. Probab.}, 21:\penalty0 14 pp., 2016.

\bibitem[Manca(2008)]{manca2008kolmogorov}
L.~Manca.
\newblock \emph{Kolmogorov operators in spaces of continuous functions and
  equations for measures}.
\newblock PhD thesis, Scuola Normale Superiore di Pisa, 2008.

\bibitem[Mattingly et~al.(2010)Mattingly, Stuart, and Tretyakov]{Mattingly10}
J.~Mattingly, A.~Stuart, and M.~Tretyakov.
\newblock Convergence of numerical time-averaging and stationary measures via
  {P}oisson equations.
\newblock \emph{SIAM J. Numer. Anal.}, 48\penalty0 (2):\penalty0 552--577,
  2010.

\bibitem[Meckes(2009)]{Meckes09}
E.~Meckes.
\newblock On {S}tein's method for multivariate normal approximation.
\newblock In \emph{High dimensional probability {V}: the {L}uminy volume},
  volume~5 of \emph{Inst. Math. Stat. Collect.}, pages 153--178. Inst. Math.
  Statist., Beachwood, OH, 2009.

\bibitem[M{\"u}ller(1997)]{Muller97}
A.~M{\"u}ller.
\newblock Integral probability metrics and their generating classes of
  functions.
\newblock \emph{Ann. Appl. Probab.}, 29\penalty0 (2):\penalty0 pp. 429--443,
  1997.

\bibitem[Nourdin et~al.(2010)Nourdin, Peccati, and
  R{\'e}veillac]{NourdinPeRe10}
I.~Nourdin, G.~Peccati, and A.~R{\'e}veillac.
\newblock Multivariate normal approximation using {S}tein's method and
  {M}alliavin calculus.
\newblock \emph{Ann. Inst. Henri Poincar\'e Probab. Stat.}, 46\penalty0
  (1):\penalty0 45--58, 2010.
\newblock ISSN 0246-0203.

\bibitem[Oates et~al.(2016)Oates, Girolami, and Chopin]{OatesGiCh2016}
C.~J. Oates, M.~Girolami, and N.~Chopin.
\newblock Control functionals for {M}onte {C}arlo integration.
\newblock \emph{Journal of the Royal Statistical Society: Series B (Statistical
  Methodology)}, pages n/a--n/a, 2016.
\newblock ISSN 1467-9868.

\bibitem[Oksendal(2013)]{Oksendal2013}
B.~Oksendal.
\newblock \emph{Stochastic differential equations: an introduction with
  applications}.
\newblock Springer Science \& Business Media, 6 edition, 2013.

\bibitem[Optimization(2015)]{Gurobi15}
G.~Optimization.
\newblock Gurobi optimizer reference manual, 2015.
\newblock URL \url{http://www.gurobi.com}.

\bibitem[Pardoux and Veretennikov(2001)]{PardouxVe01}
E.~Pardoux and A.~Veretennikov.
\newblock On the {P}oisson equation and diffusion approximation. i.
\newblock \emph{Ann. Probab.}, pages 1061--1085, 2001.

\bibitem[Patterson and Teh(2013)]{PattersonTe13}
S.~Patterson and Y.~Teh.
\newblock Stochastic gradient {R}iemannian langevin dynamics on the probability
  simplex.
\newblock In \emph{Adv. NIPS 26}, pages 3102--3110, 2013.

\bibitem[Pavliotis(2014)]{Pavliotis14}
G.~A. Pavliotis.
\newblock \emph{Stochastic processes and applications}, volume~60 of
  \emph{Texts in Applied Mathematics}.
\newblock Springer, New York, 2014.
\newblock ISBN 978-1-4939-1322-0; 978-1-4939-1323-7.
\newblock Diffusion processes, the {F}okker-{P}lanck and {L}angevin equations.

\bibitem[Peleg and Sch{\"a}ffer(1989)]{PelegSc89}
D.~Peleg and A.~Sch{\"a}ffer.
\newblock Graph spanners.
\newblock \emph{J. Graph Theory}, 13\penalty0 (1):\penalty0 99--116, 1989.

\bibitem[Protter(2005)]{Protter2005}
P.~Protter.
\newblock \emph{Stochastic integration and differential equations}, volume~21
  of \emph{Stochastic Modelling and Applied Probability}.
\newblock Springer-Verlag, Berlin, 2005.
\newblock ISBN 3-540-00313-4.
\newblock Second edition. Version 2.1, Corrected third printing.

\bibitem[Rai{\checkv c}(2004)]{raivc2004multivariate}
M.~Rai{\checkv c}.
\newblock A multivariate {CLT} for decomposable random vectors with finite
  second moments.
\newblock \emph{Journal of Theoretical Probability}, 17\penalty0 (3):\penalty0
  573--603, 2004.

\bibitem[Ranganath et~al.(2016)Ranganath, Tran, Altosaar, and
  Blei]{Ranganath2016}
R.~Ranganath, D.~Tran, J.~Altosaar, and D.~Blei.
\newblock Operator variational inference.
\newblock In \emph{Advances in Neural Information Processing Systems}, pages
  496--504, 2016.

\bibitem[Reinert and R{\"o}llin(2009)]{ReinertRo09}
G.~Reinert and A.~R{\"o}llin.
\newblock Multivariate normal approximation with {S}tein's method of
  exchangeable pairs under a general linearity condition.
\newblock \emph{Ann. Probab.}, 37\penalty0 (6):\penalty0 2150--2173, 2009.
\newblock ISSN 0091-1798.

\bibitem[Rey-Bellet and Spiliopoulos(2014)]{rey2014irreversible}
L.~Rey-Bellet and K.~Spiliopoulos.
\newblock Irreversible {Langevin} samplers and variance reduction: a large
  deviation approach.
\newblock \emph{arXiv preprint arXiv:1404.0105}, 2014.

\bibitem[Risken(1996)]{Risken}
H.~Risken.
\newblock \emph{{The Fokker-Planck Equation: Methods of Solutions and
  Applications}}.
\newblock Springer Series in Synergetics. Springer, 2nd ed. 1989. 3rd printing
  edition, Sept. 1996.
\newblock ISBN 354061530X.

\bibitem[Roberts and Tweedie(1996)]{RobertsTw96}
G.~Roberts and R.~Tweedie.
\newblock Exponential convergence of {L}angevin distributions and their
  discrete approximations.
\newblock \emph{Bernoulli}, 2\penalty0 (4):\penalty0 341--363, 1996.
\newblock ISSN 1350-7265.

\bibitem[Roberts and Stramer(2002)]{roberts2002langevin}
G.~O. Roberts and O.~Stramer.
\newblock Langevin diffusions and {M}etropolis-{H}astings algorithms.
\newblock 4\penalty0 (4):\penalty0 337--357, 2002.

\bibitem[R{\"o}ckner et~al.(2006)R{\"o}ckner, Sobol,
  et~al.]{rockner2006kolmogorov}
M.~R{\"o}ckner, Z.~Sobol, et~al.
\newblock Kolmogorov equations in infinite dimensions: well-posedness and
  regularity of solutions, with applications to stochastic generalized burgers
  equations.
\newblock \emph{Ann. Probab.}, 34\penalty0 (2):\penalty0 663--727, 2006.

\bibitem[Shvartsman(2008)]{Shvartsman08}
P.~Shvartsman.
\newblock The {W}hitney extension problem and {L}ipschitz selections of
  set-valued mappings in jet-spaces.
\newblock \emph{Trans. Amer. Math. Soc.}, 360\penalty0 (10):\penalty0
  5529--5550, 2008.

\bibitem[Soffritti and Galimberti(2011)]{Soffritti2011multivariate}
G.~Soffritti and G.~Galimberti.
\newblock Multivariate linear regression with non-normal errors: a solution
  based on mixture models.
\newblock \emph{Statistics and Computing}, 21\penalty0 (4):\penalty0 523--536,
  2011.

\bibitem[Stein(1972)]{Stein72}
C.~Stein.
\newblock A bound for the error in the normal approximation to the distribution
  of a sum of dependent random variables.
\newblock In \emph{Proc. 6th {B}erkeley {S}ymposium on {M}athematical
  {S}tatistics and {P}robability ({U}niv. {C}alifornia, {B}erkeley, {C}alif.,
  1970/1971), {V}ol. {II}: {P}robability theory}, pages 583--602. Univ.
  California Press, Berkeley, Calif., 1972.

\bibitem[Stein et~al.(2004)Stein, Diaconis, Holmes, and Reinert]{SteinDiHoRe04}
C.~Stein, P.~Diaconis, S.~Holmes, and G.~Reinert.
\newblock Use of exchangeable pairs in the analysis of simulations.
\newblock In \emph{Stein's method: expository lectures and applications},
  volume~46 of \emph{IMS Lecture Notes Monogr. Ser.}, pages 1--26. Inst. Math.
  Statist., Beachwood, OH, 2004.

\bibitem[Stuart et~al.(2004)Stuart, Voss, Wilberg,
  et~al.]{stuart2004conditional}
A.~Stuart, J.~Voss, P.~Wilberg, et~al.
\newblock Conditional path sampling of {SDE}s and the {Langevin} {MCMC} method.
\newblock \emph{Commun. Math. Sci.}, 2\penalty0 (4):\penalty0 685--697, 2004.

\bibitem[Teh et~al.(2014)Teh, Thi{\'e}ry, and Vollmer]{TehThVo14}
Y.~Teh, A.~Thi{\'e}ry, and S.~Vollmer.
\newblock Consistency and fluctuations for stochastic gradient {L}angevin
  dynamics.
\newblock \emph{arXiv:1409.0578}, 2014.

\bibitem[Vallender(1974)]{Vallender74}
S.~Vallender.
\newblock Calculation of the {W}asserstein distance between probability
  distributions on the line.
\newblock \emph{Theory Probab. Appl.}, 18\penalty0 (4):\penalty0 784--786,
  1974.

\bibitem[Vollmer et~al.(2016)Vollmer, Zygalakis, and Teh]{VollmerZyTe16}
S.~Vollmer, K.~Zygalakis, and Y.~Teh.
\newblock Exploration of the (non-)asymptotic bias and variance of stochastic
  gradient langevin dynamics.
\newblock \emph{J. Mach. Learn. Res.}, 17\penalty0 (159):\penalty0 1--48, 2016.

\bibitem[{Wang}(2016)]{Wang2016}
F.~{Wang}.
\newblock {Exponential Contraction in {W}asserstein Distances for {D}iffusion
  {S}emigroups with {N}egative {C}urvature}.
\newblock \emph{ArXiv e-prints}, Mar. 2016.

\bibitem[Zellner(1976)]{Zellner1976}
A.~Zellner.
\newblock Bayesian and {N}on-{B}ayesian analysis of the regression model with
  multivariate student-t error terms.
\newblock \emph{JASA}, 71\penalty0 (354):\penalty0 400--405, 1976.

\bibitem[Zellner and Min(1995)]{ZellnerMi95}
A.~Zellner and C.~Min.
\newblock Gibbs sampler convergence criteria.
\newblock \emph{JASA}, 90\penalty0 (431):\penalty0 921--927, 1995.

\end{thebibliography}
\bibliographystyle{abbrvnatnourl}
\setcounter{section}{8}
\renewcommand\appendixname{Supplementary Appendix}
\begin{adjustwidth}{-2cm}{-2cm}
\newpage

\section{Smoothing and interpolation}\label{sec:smoothing}
We present in this section two essentially standard results on smoothing by convolution and seminorm interpolation
\citep[see, e.g.,][Ex. 1.1.8]{lunardiintroduction} which support the proof of \thmref{nonconstant-lower-bound}. 
Throughout, we let $G\in\reals^d$
be a standard normal vector and $\phi\in C^\infty$ be its probability density.
For any $s > 0$ and function $f : \reals^d \to \reals$ we define
\balignst
f_{s}(x)\defeq\Earg{f(x+sG)}=s^{-d}\int f(y)\phi\left(\frac{x-y}{s}\right)dy.
\ealignst

The first result bounds the Lipschitz constant of $f_s$ in terms of the \Holder continuity of $f$.
\begin{lemma}[Smoothing by convolution II]
\label{lem:standardsmoothing}
Fix $\iota \in (0,1)$ and consider any $f:\reals^d\to \reals$ with
$M_{1-\iota}(f) < \infty$.  For all $s > 0$,
\balignst
M_1(f_s) \le \Earg{\twonorm{G}^{2-2\iota}}^{1/2} M_{1-\iota}(f)s^{-\iota}.
\ealignst
\end{lemma}
\begin{proof}
Fix any $\twonorm{v}\le 1$ and $x\in\reals^d$.
Leibniz's rule implies that
\balignst
\inner{\grad f_s(x)}{v} = s^{-d-1} \int f(y) \inner{\grad \phi\left(\frac{x-y}{s}\right )}{v}\dy.
\ealignst
Because $s^{-d}\int \grad \inner{\phi(\frac{x-y}{s})}{v} \dy =0$ for any $v\in\reals^d$, we also have
\balignst
|\inner{\grad f_s(x)}{v}|
  = |s^{-d-1} \int f(y) \inner{\grad \phi\left(\frac{x-y}{s}\right)}{v}\dy|
  &= |s^{-d-1} \int [f(y) - f(x)] \inner{\grad \phi\left(\frac{x-y}{s}\right)}{v}\dy| \\
  &= |s^{-d-1} \int [f(x-z) - f(x)] \inner{\grad \phi\left(\frac{z}{s}\right)}{v}\dz| \\
  &\le s^{-d-1} \int M_{1-\iota}(f)\twonorm{z}^{1-\iota} |\inner{\grad \phi\left(\frac{z}{s}\right)}{v}|\dz \\
  &= M_{1-\iota}(f) s^{-\iota} \int \twonorm{\omega}^{1-\iota} |\inner{\grad \phi(\omega)}{v}|\,d\omega,
\ealignst
where we have used substitutions $z \defeq x - y$ and $\omega \defeq z / s$. Finally,
as $\grad \phi(\omega) = -\omega \phi(\omega)$ for all $\omega\in\reals^d$,
we can use the spherical symmetry of the standard normal and Cauchy-Schwarz
to yield
\balignst
\int \twonorm{\omega}^{1-\iota} |\inner{\grad \phi(\omega)}{v}|\,d\omega
  &= \Earg{\twonorm{G}^{1-\iota} |\inner{G}{v}|}
  \le \Earg{\twonorm{G}^{2-2\iota}}^{1/2} \Earg{|\inner{G}{v}|^2}^{1/2} \\
  &= \Earg{\twonorm{G}^{2-2\iota}}^{1/2} \Earg{G_1^2}^{1/2}
  = \Earg{\twonorm{G}^{2-2\iota}}^{1/2},
\ealignst
concluding the lemma.
\end{proof}

The second result provides interpolation bounds for the \Holder seminorm $M_k$ where $k \not\in \mathbb{N}$.

\begin{lemma}[Seminorm interpolation]\label{lem:interpol}Let $k>0$ and $f\in C^{\lceil k\rceil}(\mathbb{R}^{d})$.
Then we have that 
\[
M_{k}(f)\leq2^{1-\left\{ k\right\} }\left(M_{\lceil k\rceil-1}(f)\right)^{1-\left\{ k\right\} }\left(M_{\lceil k\rceil}(f)\right)^{\left\{ k\right\} }.
\]
\end{lemma}
\begin{proof}
For $m \in \mathbb{N}$, let $V_{m}=\lbrace(v_{1},\ldots,v_{m})\,:\,\twonorm{v_{i}}\leq 1 \text{ for each } i \in\{1,\dots,m\} \rbrace$. Using the fundamental theorem of calculus we obtain 
\balignst
\sup_{V_{\left\lceil k\right\rceil -1}}\Big|{\nabla}^{\left\lceil k\right\rceil -1}f(x)&[v_{1},v_{2},\ldots,v_{\lceil k\rceil-1}] - {\nabla}^{\left\lceil k\right\rceil -1}f(y)[v_{1},v_{2},\ldots,v_{\lceil k\rceil-1}]\Big| \\  &=\sup_{V_{\left\lceil k\right\rceil -1}}\Big|\int_{0}^{1} {\nabla}^{\left\lceil k\right\rceil }f(x+s(y-x))[v_{1},v_{2},\ldots,v_{\lceil k\rceil-1},y-x]ds\Big|\\
 & \le\sup_{V_{\left\lceil k\right\rceil -1}}\Big|\sup_{z}{\nabla}^{\left\lceil k\right\rceil }f(z)[v_{1},v_{2},\ldots,v_{\lceil k\rceil-\text{1}},y-x]\Big|\\
 & \leq\sup_{z}\opnorm{{\nabla}^{\left\lceil k\right\rceil }f(z)}\twonorm{x-y}.
\ealignst
An application of the triangle inequality gives rise to
\balignst
\sup_{V_{\left\lceil k\right\rceil -1}}\left|{\nabla}^{\left\lceil k\right\rceil -1}f(x)[v_{1},v_{2},\ldots,v_{\lceil k\rceil-1}]- {\nabla}^{\left\lceil k\right\rceil -1}f(y)[v_{1},v_{2},\ldots,v_{\lceil k\rceil-1}]\right|\le2\sup_{z}\opnorm{{\nabla}^{\lceil k\rceil-1}f(z)}.
\ealignst
There we obtain
\balignst
M_{k}(f) & =\sup_{x,y\in\mathbb{R}^{d};x\neq y}\frac{\opnorm{{\nabla}^{\lceil k\rceil-1}f(x)-{\nabla}^{\lceil k\rceil-1}f(y)}}{\twonorm{x-y}^{\lbrace k\rbrace}}\\
 & \leq\sup_{x,y\in\mathbb{R}^{d};x\neq y}\frac{2^{1-\left\{ k\right\} }\left(\sup_{z}\opnorm{{\nabla}^{\left\lceil k\right\rceil }f(z)}\right)^{\lbrace k\rbrace}\left(\sup_{z}\opnorm{{\nabla}^{\lceil k\rceil-1}f(z)}\right)^{1-\left\{ k\right\} }\twonorm{x-y}^{\lbrace k\rbrace}}{\twonorm{x-y}^{\lbrace k\rbrace}}\\
 & \leq2^{1-\left\{ k\right\} }\left(\sup_{z}\opnorm{{\nabla}^{\left\lceil k\right\rceil }f(z)}\right)^{\lbrace k\rbrace}\left(\sup_{z}\opnorm{{\nabla}^{\lceil k\rceil-1}f(z)}\right)^{1-\left\{ k\right\} }\\
 & \leq2^{1-\lbrace k\rbrace}\left(M_{\lceil k\rceil}(f)\right)^{\lbrace k\rbrace}\left(M_{\lceil k\rceil-1}(f)\right)^{1-\lbrace k\rbrace}
\ealignst
thus proving the statement.
\end{proof}

\section{Semigroup third derivative estimate}
\label{app:thirdderivflow}
\begin{lemma}[Semigroup third derivative estimate]
\label{lem:ThirdDerivFlow}
Suppose that the drift and diffusion coefficients $b$ and $\sigma$ of an \Ito diffusion have bounded, locally Lipschitz first, second, and third derivatives.  If the transition semigroup $(\trans{t}{})_{t\geq 0}$ has Wasserstein decay rate $r$,
$\sigma(x)$ has a right inverse $\sigma^{-1}(x)$ for each $x\in\reals^d$, and $M_0(\sigma^{-1}) < \infty$, then,
for all $t > 0$ and any $f \in C^3$ with bounded second and third derivatives,
\balignt
M_3(\trans{t}{f}) \leq \inf_{t_0\in (0,t]} M_1(f)r(t-t_0)\frac{c}{t_0}e^{Ct_0}\label{eq:lemThirdDerivFlow}
\ealignt
for constants $c, C$ depending only on $M_{1:3}(\sigma), M_{1:3}(b), M_0(\sigma^{-1}), $ and $r$.
\end{lemma}

\begin{proof}
Our proof closely follows that of \lemref{semigroup-derivative-estimates} in \secref{semigroup-derivative-estimates-proof}, and we will only highlight the important differences.
Throughout, $c$ and $C$ will represent arbitrary constants depending only on $M_{1:3}(\sigma), M_{1:3}(b), M_0(\sigma^{-1}), $ and $r$ that may change from expression to expression.

Fix any $v, v', v''$ with unit Euclidean norms in $\reals^d$ and, without loss of generality, fix any $f : \reals^d \to \reals$ in $C^3$ with bounded  first, second, and third derivatives.
Let $(\process{t}{x})_{t\geq0}$ be an \Ito diffusion solving the stochastic differential equation \eqnref{diffusion}
with starting point $\process{0}{x}=x$, underlying Wiener process $(W_t)_{t\geq 0}$, and transition semigroup $(\trans{t}{})_{t\geq0}$.
Under our smoothness assumptions on $b$ and $\sigma$, \citep[Theorem V.40]{Protter2005} implies that $(\process{t}{x})_{t\geq0}$ is thrice continuously differentiable in $x$ with third directional derivative flow $(Y_{t,v,v',v''})_{t\geq 0}$ solving the \emph{third variation equation},
\balignt
dY_{t,v,v^{\prime},v^{\prime\prime}} &= \grad b(\process{t}{x})Y_{t,v,v^{\prime},v^{\prime\prime}}\dt+\Hess b(\process{t}{x})[U_{t,v,v'}]V_{t,v^{\prime\prime}}\dt\label{eq:3rdDerivFlow}\\
 & +\nabla^{3}b(\process{t}{x})[V_{t,v},V_{t,v'},V_{t,v^{\prime\prime}}]\dt+\Hess b(\process{t}{x})[U_{t,v',v''}]V_{t,v}\dt\nonumber \\
 & +\Hess b(\process{t}{x})[U_{t,v,v''}]V_{t,v'}\dt\nonumber \\
 & +\grad\sigma(\process{t}{x})Y_{t,v,v^{\prime},v^{\prime\prime}}\dW_{t}+\Hess\sigma(\process{t}{x})[U_{t,v,v'}]V_{t,v^{\prime\prime}}\dW_{t}\nonumber \\
 & +\nabla^{3}\sigma(\process{t}{x})[V_{t,v},V_{t,v'},V_{t,v^{\prime\prime}}]\dW_{t}+\Hess\sigma(\process{t}{x})[U_{t,v',v''}]V_{t,v}\dW_{t}\nonumber \\
 & +\Hess\sigma(\process{t}{x})[U_{t,v,v''}]V_{t,v'}\dW_{t} \qtext{with} Y_{0,v,v',v''} = 0, \nonumber 
\ealignt
obtained by differentiating \eqnref{second-variation} with respect to $x$ in the direction $v''$.

In a manner analogous to the derivation of \eqref{eqn:BEL2} in proof of \lemref{semigroup-derivative-estimates}, 
we can derive an expression for the third derivative of the semi-group,
\balignt
 & \nabla^{3}\transarg{t}{f}{x}[v,v',v'']=\Earg{\sum_{i,j}J_{i,j,x}}\qtext{for}\label{eq:third-deriv-semi-group}\\
 & J_{1,1,x}{\textstyle \defeq\frac{1}{t}{\inner{\Hess{f}(\process{t}{x})V_{t,v''}}{V_{t,v'}}\int_{0}^{t}\inner{\sigma^{-1}(\process{s}{x})V_{s,v}}{\dW_{s}}},\notag}\nonumber \\
 & J_{1,2,x}{\textstyle \defeq\frac{1}{t}{\inner{\grad{f}(\process{t}{x})}{U_{t,v',v''}}\int_{0}^{t}\inner{\sigma^{-1}(\process{s}{x})V_{s,v}}{\dW_{s}}},\notag}\nonumber \\
 & J_{1,3,x}{\textstyle \defeq\frac{1}{t}{\inner{\grad{f}(\process{t}{x})}{V_{t,v'}}\int_{0}^{t}\inner{\grad\sigma^{-1}(\process{s}{x})[V_{s,v''}]V_{s,v}}{\dW_{s}}},\notag}\nonumber \\
 & J_{2,1x}{\textstyle \defeq\frac{1}{t}\inner{\grad f\left({\process{t}{x}}\right)}{V_{t,v''}}\int_{0}^{t}\inner{\grad\sigma^{-1}(\process{s}{x})[V_{s,v'}]V_{s,v}}{\dW_{s}},\notag}\nonumber \\
 & J_{2,2,x}{\textstyle \defeq\frac{1}{t}{f(\process{t}{x})\int_{0}^{t}\inner{\Hess\sigma^{-1}(\process{s}{x})[V_{s,v''}][V_{s,v'}]V_{s,v}}{\dW_{s}}},\notag}\nonumber \\
 & J_{2,3,x}{\textstyle \defeq\frac{1}{t}{f(\process{t}{x})\int_{0}^{t}\inner{\grad\sigma^{-1}(\process{s}{x})[U_{s,v',v''}]V_{s,v}}{\dW_{s}}},\notag}\nonumber \\
 & J_{2,4,x}{\textstyle \defeq\frac{1}{t}{f(\process{t}{x})\int_{0}^{t}\inner{\grad\sigma^{-1}(\process{s}{x})[V_{s,v'}]U_{s,v,v''}}{\dW_{s}}},\notag}\nonumber \\
 & J_{3,1,x}{\textstyle \defeq\frac{1}{t}{\inner{\grad{f}({\process{t}{x}})}{V_{t,v''}}\int_{0}^{t}\inner{\sigma^{-1}(\process{s}{x})U_{s,v,v'}}{\dW_{s}}},\notag}\nonumber \\
 & J_{3,2,x}{\textstyle \defeq\frac{1}{t}{f(\process{t}{x})\int_{0}^{t}\inner{\nabla\sigma^{-1}(\process{s}{x})[V_{s,v''}]U_{s,v,v'}}{\dW_{s}}},\notag}\nonumber \\
 & J_{3,3,x}{\textstyle \defeq\frac{1}{t}{f(\process{t}{x})\int_{0}^{t}\inner{\sigma^{-1}(\process{s}{x})Y_{s,v,v',v''}}{\dW_{s}}},\notag}\nonumber \\
& J_{3,4,x}{\textstyle \defeq\frac{1}{t}{\inner{\grad{f}({\process{t}{x}})}{V_{t,v'}}\int_{0}^{t}\inner{\sigma^{-1}(\process{s}{x})U_{s,v,v''}}{\dW_{s}}}.\notag}\nonumber 
\ealignt
We will bound each term $J_{i,j,x}$ in \eqref{eq:third-deriv-semi-group} in turn.  

\subsection{The $J_{1,\cdot,x}$ terms}

We will provide a step-by-step calculation for the first term. 
By Cauchy-Schwarz,
\balignst
\Earg{J_{1,1,x}} & {\textstyle =\frac{1}{t}\Earg{\inner{\Hess{f}(\process{t}{x})V_{t,v''}}{V_{t,v'}}\int_{0}^{t}\inner{\sigma^{-1}(\process{s}{x})V_{s,v}}{\dW_{s}}}\notag}\\
 & \;\leq\frac{1}{t}\sqrt{\Earg{\inner{\Hess{f}(\process{t}{x})V_{t,v''}}{V_{t,v'}}^{2}}\E[{ (\int_{0}^{t}\inner{\sigma^{-1}(\process{s}{x})V_{s,v}}{\dW_{s}})^{2}}]}.
\ealignst
We use the derivative flow bounds of Lemma \ref{lem:derivative-flow-bounds} to realize
\balignst
\sqrt{\Earg{\twonorm{V_{t,v'}}^{2}\twonorm{V_{t,v''}}^{2}}} & \leq\sqrt[4]{\E \twonorm{V_{t,v'}}^{4}\E \twonorm{V_{t,v''}}^{4}}
 \leq\twonorm{v'}\twonorm{v''}e^{\frac{1}{2}t\gamma_{4}}.
\ealignst
Cauchy-Schwarz, the \Ito isometry \citep[Eqs. 7.1 and 7.2]{Friedman1975Stoch}, and \lemref{derivative-flow-bounds} now yield
\balignst
\Earg{J_{1,1,x}} 
 & \overset{\text{}}{\leq}\frac{1}{t}M_{2}(f)\sqrt{\Earg{\twonorm{V_{t,v'}}^{2}\twonorm{V_{t,v''}}^{2}}\Earg{ \int_{0}^{t}\twonorm{\sigma^{-1}(\process{s}{x})V_{s,v}}^2\ds}}\\
 & \leq\frac{1}{t}M_{2}(f)\twonorm{v'}\twonorm{v''}\twonorm{v}e^{\frac{1}{2}t\gamma_{4}}M_{0}(\sigma^{-1})\left(\frac{1}{\gamma_{2}}(e^{\gamma_{2}t}-1)\right)^{\frac{1}{2}} 
 \leq M_2(f)\frac{c}{\sqrt{t}} e^{Ct}.
\ealignst
Similar reasoning yields
\balignst
\Earg{J_{1,2,x}} & {\textstyle =\E \frac{1}{t}{\inner{\grad{f}({\process{t}{x}})}{U_{t,v',v''}}\int_{0}^{t}\inner{\sigma^{-1}(\process{s}{x})V_{s,v}}{\dW_{s}}}\notag}\\
 & \leq\frac{1}{t}M_{1}(f)\sqrt{t}e^{\frac{1}{4}t\gamma_{4}}M_{0}(\sigma^{-1})\left(\frac{1}{\gamma_{2}}(e^{\gamma_{2}t}-1)\right)^{\frac{1}{2}}
 \leq M_{1}(f)c e^{Ct}
\ealignst
and using equation \eqref{eq:volDeriv}
\balignst
\Earg{J_{1,3,x}} & \leq\frac{1}{t}M_{1}(f)e^{\frac{1}{2}\gamma_{2}t}\twonorm{v'}\sqrt{\E \int_{0}^{t}M_{1}(\sigma^{-1})^{2}\twonorm{V_{s,v}}^2\twonorm{V_{s,v''}}^2\ds}\\
 & \leq\frac{1}{t}M_{1}(f)e^{\frac{1}{2}\gamma_{2}t}\twonorm{v'}M_{1}(\sigma^{-1})\twonorm{v}\twonorm{v''}\left(\int_{0}^{t}e^{\gamma_{4}s}ds\right)^{\frac{1}{2}}\\
 & \leq t^{-\frac{1}{2}}M_{1}(f)e^{\gamma_{2}t/2}\twonorm{v'}M_{0}(\sigma^{-1})^2M_1(\sigma)\twonorm{v}\twonorm{v''}e^{\gamma_{4}t/2}
 \leq M_{1}(f)\frac{c}{\sqrt{t}} e^{Ct}.
\ealignst

\subsection{The $J_{2,\cdot,x}$ terms}
The bound $\Earg{J_{2,1,x}} \leq M_{1}(f)\frac{c}{\sqrt{t}} e^{Ct}$ follows exactly as it did for $J_{1,3,x}$.
To tackle the remaining $J_{2,\cdot,x}$ terms, we will rewrite the unbounded quantity $f(\process{t}{x})$ using \eqnref{bel-ito}.
We obtain the bound
\balignst
\E J_{2,2,x} & {\textstyle =\E \frac{1}{t}{f(\process{t}{x})\int_{0}^{t}\inner{\Hess\sigma^{-1}(\process{s}{x})[V_{s,v''}][V_{s,v'}]V_{s,v}}{\dW_{s}}}}\\
 & =\E \frac{1}{t}{\int_{0}^{t}\inner{\grad\transarg{t-s}{f}{\process{s}{x}}}{\sigma(\process{s}{x})\dW_{s}}.\int_{0}^{t}\inner{\Hess\sigma^{-1}(\process{s}{x})[V_{s,v''}][V_{s,v'}]V_{s,v}}{\dW_{s}}}\\
 & =\frac{1}{t}\E \int_{0}^{t}\inner{\grad\transarg{t-s}{f}{\process{s}{x}}}{\sigma(\process{s}{x})\Hess\sigma^{-1}(\process{s}{x})[V_{s,v''}][V_{s,v'}]V_{s,v}}ds\\
 & \leq\frac{1}{t}M_{1}(f)r(0)\left(2M_{1}(\sigma)^{2}M_{0}(\sigma^{-1})^{2}+M_{2}(\sigma)M_{0}(\sigma^{-1})\right)\int_{0}^{t}\Earg{\twonorm{V_{s,v''}}\twonorm{V_{s,v'}}\twonorm{V_{s,v'}}}ds.\\
 & \leq M_{1}(f) r(0)\left(2M_{1}(\sigma)^{2}M_{0}(\sigma^{-1})^{2}+M_{2}(\sigma)M_{0}(\sigma^{-1})\right)e^{\gamma_{3}t}\twonorm{v''}\twonorm{v'}\twonorm{v}
 \leq M_{1}(f)c e^{Ct},
\ealignst
where we used the chain rule expression
\begin{equation*}
\begin{aligned}
\textstyle
\nabla^{2}\sigma^{-1}(x)[v][v']  =-\sigma(x)^{-1}\Big(-&\nabla\sigma(x)[v]\sigma(x)^{-1}\nabla\sigma[v'](x) 
 -\nabla\sigma(x)[v']\sigma(x)^{-1}\nabla\sigma(x)[v] \\ 
\textstyle
&\quad +\nabla^{2}\sigma(x)[v][v']\Big)\sigma(x)^{-1}
\end{aligned}
\end{equation*}
 to rewrite $\sigma(\process{s}{x})\Hess\sigma^{-1}(\process{s}{x})$. The next term satisfies
\balignst
\Earg{J_{2,3,x}} & =\E \frac{1}{t}{\int_{0}^{t}\inner{\grad\transarg{t-s}{f}{\process{s}{x}}}{\sigma(\process{s}{x})\dW_{s}}\int_{0}^{t}\inner{\grad\sigma^{-1}(\process{s}{x})[U_{s,v',v''}]V_{s,v}}{\dW_{s}}}\\
 & =\E \frac{1}{t}{\int_{0}^{t}\inner{\grad\transarg{t-s}{f}{\process{s}{x}}}{\sigma(\process{s}{x})\grad\sigma^{-1}(\process{s}{x})[U_{s,v',v''}]V_{s,v}}ds}\\
 & \leq\frac{1}{t}M_{1}(f)M_{0}(\sigma^{-1})M_{1}(\sigma)r(x)\int_{0}^{t}\E \norm{U_{s,v',v''}}\norm{V_{s,v}}ds\\
 & \leq M_{1}(f)M_{0}(\sigma^{-1})M_{1}(\sigma)r(x)\frac{1}{t}\int_{0}^{t}\norm{v}\norm{v'}\norm{v''}\left(\alpha se^{\gamma_{4}s}\right)^{\frac{1}{2}}e^{s\gamma_{2}/2}ds
 \leq M_{1}(f)c e^{Ct}.
\ealignst
The term $\Earg{J_{2,4,x}}$ can be bounded in the same way by swapping the roles of $v$ and $v'$.

\subsection{The $J_{3,\cdot}$ terms}
Cauchy-Schwarz and the \Ito isometry \citep[Eqs. 7.1 and 7.2]{Friedman1975Stoch} yield
\balignst
\Earg{J_{3,1,x}} 
	&\leq \frac{1}{t}M_1(f)M_0(\sigma^{-1})\sqrt{\Earg{\twonorm{V_{t,v''}}^2} \Earg{\int_{0}^{t} \twonorm{U_{s,v,v'}}^2\ds}} \\
	&\leq \frac{1}{t}M_1(f)M_0(\sigma^{-1})e^{\half t\gamma_2}\left(\frac{1}{\gamma_{2}}(e^{\gamma_{2}t}-1)\right)^{\frac{1}{2}} 
	\leq M_{1}(f)\frac{c}{\sqrt{t}} e^{Ct}.
\ealignst
The bound for $\Earg{J_{3,4,x}}$ is identical, and
the bound $\Earg{J_{3,2,x}} \leq M_{1}(f) c e^{Ct}$ follows exactly as it did for $J_{2,3,x}$.
Now we consider the last term 
\balignst
J_{3,3,x}{\textstyle =\frac{1}{t}{f(\process{t}{x})\int_{0}^{t}\inner{\sigma^{-1}(\process{s}{x})Y_{s,v,v',v''}}{\dW_{s}}},\notag}
\ealignst
Using (\ref{eqn:bel-ito}), \lemref{semigroup-derivative-estimates}, and the non-increasing property of $r$, we find that
\balignst
\Earg{J_{3,3,x}} & =\frac{1}{t}\E \int_0^t\inner{\grad\transarg{t-s}{f}{\process{s}{x}}}{Y_{s,v,v',v''}}ds
 \leq\frac{1}{t}\int_0^t M_{1}(\trans{t-s}{f})\E \twonorm{Y_{s,v,v',v''}}ds\\
 & \leq\frac{1}{t}\int_0^t M_{1}(\trans{t-s}{f})\left(\E \twonorm{Y_{s,v,v',v''}}^{2}\right)^{\frac{1}{2}}ds
 \leq M_{1}(f)\frac{1}{t}\int_0^t r(t-s) \left(\E \twonorm{Y_{s,v,v',v''}}^{2}\right)^{\frac{1}{2}}ds \\
 & \leq M_{1}(f) r(0) \frac{1}{t}\int_0^t \left(\E \twonorm{Y_{s,v,v',v''}}^{2}\right)^{\frac{1}{2}}ds.
\ealignst
This final expression is bounded by $M_{1}(f) c e^{Ct}$ provided that 
$
\E \twonorm{Y_{s,v,v',v''}}^{2}\leq ce^{Cs}.
$
We will establish such a bound for the third directional derivative flow in
Section~\ref{subsec:Third-derivative-flow}. 

\subsection{Semigroup third derivative bound}
By combining the bounds for each $J_{i,j,x}$ term,
adapting the argument of \cite[Prop.~1.5.1]{Cerrai2001}, 
and invoking the semigroup gradient bound and
Hessian bound $M_2(\trans{s}{f}) \leq M_1(f) r(s-s_0)\frac{c'}{\sqrt{s_0}}e^{C' s_0}$ of Lemma \ref{lem:semigroup-derivative-estimates},
we obtain, for any $t_0 \in (0, t]$ and $s_0 = t_0/2$
\balign
 & \left\Vert
\nabla^{3}\trans{t}{f}[v,v',v'']\right\Vert_{\text{op}}=\left\Vert
\nabla^{3}\trans{t_{0}/2}{\left({\trans{t-t_{0}/2}{f}}\right)}[v,v',v'']\right\Vert_{op}\label{eq:3rdCerrai}\\
  & \leq 
  ( M_{1}(\trans{t-t_{0}/2}{f}) + M_2(\trans{t-t_{0}/2}{f}) )\frac{c}{\sqrt{t_0/2}}e^{C t_0/2} \nonumber \\
 &\leq M_{1}(f)(r(t-t_{0}/2) +  r(t-t_0/2-s_0)\frac{c'}{\sqrt{s_0}}e^{C' s_0})\frac{c}{\sqrt{t_0/2}}e^{C t_0/2}.\nonumber\\
 &\leq M_{1}(f)(r(t-t_{0}/2) +  r(t-t_0)\frac{c'}{\sqrt{t_0/2}}e^{C' t_0/2})\frac{c}{\sqrt{t_0/2}}e^{C t_0/2}.\nonumber
\ealign

\subsection{Third derivative flow bound\label{subsec:Third-derivative-flow}}
Introduce the shorthand $(Y_{t})_{t \geq 0}$ for $(Y_{t,v,v',v''})_{t \geq 0}$ solving the third variation equation \eqref{eq:3rdDerivFlow}.
Dynkin's formula gives $\E \twonorm{Y_{t}}^{2} = \int_{0}^{t}T_{1} + T_{2} \ds$ for
\balignst
T_1 &\defeq \E2\Big\langle Y_{s}, \grad b(\process{s}{x})Y_{s}+\Hess b(\process{s}{x})[U_{s,v,v'}]V_{s,v^{\prime\prime}}+\nabla^{3}b(\process{s}{x})[V_{s,v},V_{s,v'},V_{s,v^{\prime\prime}}] \\
&\quad +\Hess b(\process{s}{x})[U_{s,v',v''}]V_{s,v}+\Hess b(\process{t}{x})[U_{t,v,v''}]V_{t,v'}\Big\rangle\\
T_{2}& \defeq \E\Big\lVert\grad\sigma(\process{s}{x})[Y_{s}]+\Hess\sigma(\process{s}{x})[U_{s,v,v'}]V_{s,v^{\prime\prime}}+\nabla^{3}\sigma(\process{s}{x})[V_{s,v},V_{s,v'},V_{s,v^{\prime\prime}}] \\
 &\quad +\Hess\sigma(\process{s}{x})[U_{s,v',v''}]V_{s,v}+\Hess\sigma(\process{s}{x})[U_{s,v,v''}]V_{s,v'}\Big\rVert_{F}^{2}.
\ealignst
We have by Cauchy-Schwarz and Young's inequality \balignst
\frac{T_{1}}{2} & \leq\E \bigl(\twonorm{Y_{s}}^{2}M_{1}(b)+M_{2}(b)\twonorm{Y_{s}}\underbrace{\twonorm{U_{s,v,v'}}\twonorm{V_{s,v''}}}_{+2\text{ permutations}}+M_{3}(b)\twonorm{Y_{s}}\twonorm{V_{s,v}}\twonorm{V_{s,v'}}\twonorm{V_{s,v''}}\bigr)\\
 & \leq\E \bigl(\twonorm{Y_{s}}^{2}\left(M_{1}(b)+M_{2}^{2}(b)+M_{3}^{2}(b)\right)+\underbrace{M_{2}^{2}(b){\twonorm{U_{s,v,v'}}^{2}\twonorm{V_{s,v''}}^{2}}}_{+2\text{ permutations}} \\ 
 &\quad +M_{3}(b)^{2}{\left(\twonorm{V_{s,v}}\twonorm{V_{s,v'}}\twonorm{V_{s,v''}}\right)^{2}}\bigr)
\ealignst
and
\balignst\frac{T_{2}}{4} &\leq \E\twonorm{Y_{2}}^{2}\left(M_{1}(\sigma)^{2}+\norm{\nabla^{2}\sigma}_{F_{3}}^{2}+\norm{\nabla^{3}\sigma}_{F_{3}}^{2}\right) \\ &\quad +\norm{\nabla^{3}\sigma}_{F_{3}}^{2}\E\left(\norm{V_{s}}\norm{V_{s}'}\norm{V_{s}''}\right)^{2}+\norm{\nabla^{2}\sigma}_{F_{3}}^{2}\underbrace{\Earg{\twonorm{U_{s,v,v'}}^{2}\twonorm{V_{s,v''}}^{2}}}_{+2\text{ permutations}}.
\ealignst
Provided that we establish a bound of $\E \twonorm{U_{s,v,v'}}^4 \leq c t e^{Ct}$,
we have that overall
\balignst
\E \twonorm{Y_{t}}^{2}\leq\int_{0}^{t}c\E \twonorm{Y_{s}}^{2}ds+ce^{Ct}.
\ealignst
We can conclude using Gronwall's inequality that 
\balignt
\E \twonorm{Y_{t}}^{2}\leq ce^{Ct}.\label{eq:third-deriv-flow}
\ealignt
It remains to establish bounds on $\E \twonorm{U_{t,v,v'}}^{\rho}$ for $\rho > 2$.
Recall that the second derivative flow solves \eqnref{second-variation}.
Applying Ito's formula to $f(U_{t,v,v'})=\lVert U_{t,v,v'}\rVert_{2}^{\rho},$
taking expectations, and introducing the shorthand $U_{t}=U_{t,v,v'}$, we obtain
\balignst
 &  \Earg{\twonorm{U_{t}}^{\rho}}=\twonorm{U_{0}}^{\rho}+\E\Big[\int_{0}^{t}\rho\inner{U_{s}\twonorm{U_{s,}}^{\rho-2}}{\grad b(\process{s}{x})U_{s}+\Hess b(\process{s}{x})[V_{s,v'}]V_{s,v})} \\
 &\quad + \frac{\rho}{2}\twonorm{U_{s}}^{\rho-4}((\rho-2)\twonorm{U_{s}^{\top}\grad\sigma(\process{s}{x})[U_{s,v}]+U_{s}^{\top}\Hess\sigma(\process{s}{x})[V_{s,v'}]V_{s,v}}^{2} \\ &\quad+ \twonorm{U_{s,v}}^{2}\fronorm{\grad\sigma(\process{s}{x})[U_{s,}]+\Hess\sigma(\process{s}{x})[V_{s,v'}]V_{s,v}}^{2})\ds \Big]\\
\leq & \twonorm{U_{0}}^{\rho}+\int_{0}^{t}\rho M_{1}(b)\twonorm{U_{s}}^{\rho}+\rho M_{2}(b)\twonorm{U_{s}}^{\rho-1}\twonorm{V_{s}}\twonorm{V'_{s}}\, \\
&\quad +\frac{\rho^{2}-\rho}{2}\left(M_{1}(\sigma)^{2}\twonorm{U_{s}}^{\rho}+M_{2}(\sigma)^{2}\twonorm{U_{s}}^{\rho-2}\twonorm{V_{s}}\twonorm{V'_{s}}\right) \\
&\quad +\frac{\rho}{2}\left(F_{1}(\sigma)^{2}\twonorm{U_{s}}^{\rho}+F_{2}(\sigma)^{2}\twonorm{U_{s}}^{\rho-2}\twonorm{V_{s}}\twonorm{V'_{s}}\right)ds\\
\leq & \twonorm{U_{0}}^{\rho}+\int_{0}^{t}\Earg{\twonorm{U_{s}}^{\rho}}(\rho M_{1}(b)+(\rho-1)M_{2}(b)+M_{1}(\sigma)^{2}\frac{\rho^{2}-\rho}{2}+M_{2}(\sigma)^{2}\frac{(\rho-1)^{2}}{2}+F_{2}(\sigma)\frac{\rho-1}{2})ds\\
 & +\int_0^t\left(M_{2}(b)+\frac{\rho-1}{2}M_{2}(\sigma)^{2}+\frac{1}{2}\right)\Earg{\left(\twonorm{V_{s}}\twonorm{V'_{s}}\right)^{\rho}}ds\\
\leq & \twonorm{U_{0}}^{\rho}+\int_{0}^{t}\Earg{\twonorm{U_{s}}^{\rho}}(\rho M_{1}(b)+(\rho-1)M_{2}(b)+M_{1}(\sigma)^{2}\frac{\rho^{2}-\rho}{2}+M_{2}(\sigma)^{2}\frac{(\rho-1)^{2}}{2}+F_{2}(\sigma)\frac{\rho-1}{2})ds\\
 & +\int_{0}^{t}\left(M_{2}(b)+\frac{\rho-1}{2}M_{2}(\sigma)^{2}+\frac{1}{2}\right)\left(\twonorm{v}\twonorm{v'}\right)^{\rho}e^{\gamma_{2\rho}s}ds 
\ealignst
where we use that, by Young's inequality,
\balignst
\twonorm{U_{s}}^{\rho-1}\twonorm{V_{s}}\twonorm{V'_{s}}\leq\frac{\rho-1}{\rho}\twonorm{U_{s}}^{\rho}+\frac{1}{\rho}\twonorm{V_{s}}^{\rho}\twonorm{V_{s}'}^{\rho},
\ealignst
and similarly
\balignst
\twonorm{U_{s}}^{\rho-2}\twonorm{V_{s}}\twonorm{V_{s}'}\leq\frac{\rho-2}{\rho}\twonorm{U_{s}}^{\rho}+\frac{2}{\rho}\twonorm{V_{s}}^{\rho/2}\twonorm{V_{s}'}^{\rho/2}.
\ealignst
Following the arguments of Section \ref{sec:derivative-flow-bounds}, \Gronwall's inequality gives
\balignst
\Earg{\twonorm{U_{t}}^{\rho}}\leq\left(M_{2}(b)+\frac{\rho-1}{2}M_{2}(\sigma)^{2}+\frac{1}{2}\right)\left(\twonorm{v}\twonorm{v'}\right)^{\rho}e^{\gamma_{2\rho}t}t\exp{\gamma_{\rho}t}.
\ealignst
\end{proof}

\end{adjustwidth}

\end{document}